\documentclass{article}
\PassOptionsToPackage{numbers, compress}{natbib}
\usepackage[main,final]{neurips_2026}
\usepackage[utf8]{inputenc}
\usepackage[T1]{fontenc}    
\usepackage{hyperref}      
\usepackage{url}          
\usepackage{booktabs}      
\usepackage{amsfonts}     
\usepackage{nicefrac}    
\usepackage{microtype}    
\usepackage{xcolor}        
\usepackage{graphicx}
\usepackage{amsmath,amssymb,amsthm,mathtools}
\usepackage{bbm}
\usepackage{enumitem}
\usepackage{algorithm}
\usepackage{algorithmic}
\usepackage{cleveref}
\usepackage{titletoc}
\newcommand\DoToC{%
  \startcontents
  \printcontents{}{1}{\textbf{Table of Contents}\vskip3pt\hrule\vskip5pt}
  \vskip3pt\hrule\vskip5pt
}
\usepackage{subcaption}
\theoremstyle{plain}
\newtheorem{theorem}{Theorem}[section]
\newtheorem{proposition}[theorem]{Proposition}
\newtheorem{corollary}[theorem]{Corollary}
\newtheorem{lemma}[theorem]{Lemma}

\theoremstyle{remark}
\newtheorem{remark}[theorem]{Remark}

\newcommand{\istar}{i^\star}

\newcommand{\Ber}{\mathrm{Ber}}
\newcommand{\KL}[2]{d\!\left(#1,#2\right)}
\newcommand{\Neff}{N_{\mathrm{eff}}}
\newcommand{\one}{\mathbf{1}}

\newcommand{\argmax}{\operatorname*{arg\,max}}
\newcommand{\argmin}{\operatorname*{arg\,min}}

\newcommand{\floor}[1]{\left\lfloor #1 \right\rfloor}
\newcommand{\ceil}[1]{\left\lceil #1 \right\rceil}
\newcommand{\eps}{\varepsilon}
\newcommand{\cF}{\mathcal{F}}
\newcommand{\cG}{\mathcal{G}}
\newcommand{\cK}{\mathcal{K}}

\newcommand{\cA}{\mathcal{A}}
\newcommand{\cE}{\mathcal{E}}

\newcommand{\Tstar}{T^\star}
\newcommand{\Gdisc}{\Gamma^{\mathrm{disc}}}
\title{Sharp Non-Asymptotic Analysis of the Penalized Challenger in $\beta$-EB-TCI for Bernoulli Bandits}
\author{%
  Nam Nguyen$^{1,2}$ \qquad Tuan Dam$^{1}$\thanks{Corresponding author.} \\
  $^{1}$Hanoi University of Science and Technology, Hanoi, Vietnam\\
  $^{2}$Center for AI Research, VinUniversity, Hanoi, Vietnam\\
  \texttt{26nam.n@vinuni.edu.vn}, \texttt{tuandq@soict.hust.edu.vn}
}
\hypersetup{
  pdftitle={Sharp Non-Asymptotic Analysis of the Penalized Challenger in beta-EB-TCI for Bernoulli Bandits},
  pdfauthor={Nam Nguyen and Tuan Quang Dam},
  colorlinks=true,
  linkcolor=blue,
  citecolor=blue,
  urlcolor=blue
}
\begin{document}
\numberwithin{equation}{section} %
\maketitle
\begin{abstract}
Top-two algorithms are simple and effective for fixed-confidence best-arm
identification, but their sharp non-asymptotic behavior is still not well
understood. We study this question for Bernoulli bandits through
$\beta$-EB-TCI, the empirical-best top-two rule of \citet{jourdan2022top},
whose challenger minimizes a Bernoulli transportation cost plus a logarithmic
penalty in its pull count. We first analyze the stabilized regime, in which the
empirical leader has become the true best arm and its sampling fraction stays
close to $\beta$. In this regime, every suboptimal arm is sampled linearly
often, and, with high probability, the stopping time is at most $T_{\beta}^{\star}(\mu)\log(1/\delta)$
plus stabilization and lower-order concentration terms, with the exact
fixed-$\beta$ leading coefficient. The main remaining difficulty is therefore to control when the
regime is reached. For the original rule, we obtain a non-asymptotic
high-probability stopping bound on every Bernoulli instance with a unique best
arm. Under a finite-mean sufficient-exploration condition, the bound yields the
sharp expected sample complexity; combined with the exploration result of
\citet{jourdan2022top}, this covers pairwise-distinct means under their
tie-breaking convention. A mild forced-exploration guard, which uses only
$O(\sqrt{Kt})$ pulls up to time $t$, gives a self-contained sharp expected
bound for any number of arms with a unique best arm. We also identify a limitation of
proof strategies that try to handle equal suboptimal means through a single
index-comparison argument.
\end{abstract}
\section{Introduction}
\label{sec:introduction}

Stochastic bandits model sequential decision-making under uncertainty. At each
round, a learner chooses an arm, observes a reward, and uses the accumulated
data to decide what to sample next \citep{lattimore2020bandit}. We study
fixed-confidence best-arm identification for a Bernoulli bandit with a unique
best arm \citep{even2006action,kaufmann2016complexity}. A strategy stops at a
random time $\tau_{\delta}$ and recommends an arm, which must be the best one
with probability at least $1-\delta$. The goal is to minimize the expected
number of samples while preserving both termination and the target error
probability. We use \emph{finite-confidence} for guarantees that hold at each
fixed $\delta$, rather than only in the asymptotic limit $\delta\to0$.

A central benchmark is the instance-dependent characteristic time $\Tstar(\mu)$.
Every $\delta$-correct strategy needs at least about $\Tstar(\mu)\log(1/\delta)$
samples in expectation as $\delta\to0$, and $1/\Tstar(\mu)$ is the largest rate
at which an allocation of samples can accumulate evidence against the hardest
competitor \citep{kaufmann2016complexity,garivier2016track}. Track-and-Stop
attains this benchmark by tracking an estimated optimal allocation, which
requires repeatedly solving a global optimization problem.

Top-two algorithms provide a simpler alternative
\citep{russo2016simple,shang2020fixed,jourdan2022top}. They maintain a current
leader and compare it with one challenger at a time; in the randomized rules
considered here, the leader is sampled with probability $\beta\in(0,1)$ and the
challenger otherwise. This structure is easy to implement and has strong
asymptotic guarantees. Once the leader remains the best arm, its sampling
fraction converges to $\beta$, so the natural benchmark becomes the fixed-$\beta$
characteristic time $\Tstar_{\beta}(\mu)\ge\Tstar(\mu)$: the best achievable when
the remaining fraction $1-\beta$ is split optimally among the competitors
\citep{russo2016simple,jourdan2022top} (Section~\ref{subsec:characteristic-time}).

In particular, \citet{jourdan2022top} prove asymptotic $\beta$-optimality,
$\limsup_{\delta\to0}\mathbb{E}_{\mu}[\tau_{\delta}]/\log(1/\delta)\le\Tstar_{\beta}(\mu)$,
for several top-two rules on bounded rewards, including Bernoulli $\beta$-EB-TCI
with uniform tie-breaking when all means are distinct. Non-asymptotic results
with the exact leading coefficient are more limited: the limit does not control
the remainder at a fixed $\delta$, and high-probability stopping bounds alone do
not control the expectation \citep{balagopalan2025fixing}. The closest existing
analysis is the Gaussian Top-Two UCB (TTUCB) result of \citet{jourdan2023ucbtop},
which treats a different top-two rule and a different reward model. This leaves
the following question. \emph{Does Bernoulli $\beta$-EB-TCI stop after
$\Tstar_{\beta}(\mu)\log(1/\delta)$ samples up to explicit lower-order terms at
each fixed $\delta$?}

We analyze the exact Bernoulli $\beta$-EB-TCI rule, where EB-TCI stands for
Empirical-Best leader with Transportation-Cost-Improved challenger. Under the
unguarded rule, the empirical-best arm is sampled with probability $\beta$ at each
round. Otherwise, the algorithm samples the challenger against which the
Bernoulli generalized likelihood-ratio (GLR) evidence is weakest, after adding
the logarithm of its pull count. This term favors less-sampled challengers when
the evidence values are comparable. It is part of the algorithm, not an
auxiliary proof term.

A finite-confidence analysis of this rule must address three coupled questions.
First, how closely do the empirical GLR statistics follow their population
counterparts under adaptive sampling? Second, once the true best arm is the
persistent leader, does the penalized challenger rule sample every suboptimal
arm often enough? Third, how long does it take until the leader stays correct
and the best-arm sampling fraction stays near $\beta$? Since the error guarantee
of \citet{jourdan2022top} concerns runs that stop, termination must also be
proved. We answer the first two questions in the \emph{stabilized regime}, using
a time-uniform comparison of empirical and population GLR statistics
(Theorem~\ref{thm:symmetric-glr-control}) and a discrete benchmark
(Proposition~\ref{prop:oracle}) showing that the penalty adds a common $\log t$ to the continuous rate
$t/\Tstar_{\alpha}(\mu)$, up to a bounded remainder. We then bound the time needed to reach
this regime.

For Bernoulli means $\mu\in(0,1)^K$ with a unique best arm and fixed $\beta$,
our expected-time bound reads informally, for every $\delta\in(0,1/2]$,
\[
  \mathbb{E}_{\mu}[\tau_{\delta}]
  \le
  \Tstar_{\beta}(\mu)\log(1/\delta)
  +O\!\left(\sqrt{\log(1/\delta)\,\log\log(e/\delta)}\right),
\]
under finite-mean sufficient exploration for the unguarded rule: every arm's
count is at least a positive constant times $\sqrt{t/K}$ after a random time
with finite mean. An exploration guard, which pulls a least-sampled arm
whenever some count falls below $\sqrt{t/K}$ and uses $O(\sqrt{Kt})$ pulls by
time $t$, gives the same bound for every such instance.

\paragraph{Contributions.}
\begin{enumerate}[leftmargin=1.35em]
\item \textbf{Stopping and count growth after stabilization.}
High-probability stopping bounds have the exact fixed-$\beta$ leading
coefficient, as do expected bounds when stabilization admits a summable
upper-quantile bound (Theorem~\ref{thm:conditional}); each suboptimal count
grows linearly after a transient (Proposition~\ref{prop:linear}).
Throughout, \emph{sharp} refers only to this exact leading coefficient; we do
not optimize lower-order constants or prove a matching non-asymptotic lower bound.

\item \textbf{Whole-run guarantees.}
We prove a high-probability stopping bound for every unique-best instance
(Corollary~\ref{cor:unguarded-hp}) and the expected bounds above
(Theorems~\ref{thm:unguarded-suffexp} and~\ref{thm:guarded}). Combined with
\citet{jourdan2022top}, the unguarded expected bound covers pairwise-distinct
means under their tie-breaking convention, which does not include our
deterministic default (Corollary~\ref{cor:distinctmeans}).

\item \textbf{A limitation of one-step comparisons.}
With equal-mean suboptimal arms and logarithmically many pulls of a higher-mean
arm, a population index comparison required to hold for all sufficiently large
leader counts requires a number of pulls of an
equal-mean arm that is polynomial in the inverse confidence level
(Proposition~\ref{prop:barrier}).
\end{enumerate}

Section~\ref{sec:setting-algorithm}
defines the problem and algorithm; Sections~\ref{sec:oracle-glr}--\ref{sec:unconditional-closures}
develop the tools, stopping bounds, and entry guarantees;
Section~\ref{sec:experiments} reports experiments.

\section{Related Work}
\label{sec:related-work}

We review fixed-confidence best-arm identification and the concentration tools
used in our analysis.

\paragraph{Fixed-confidence identification and allocation tracking.}
Fixed-confidence best-arm identification has progressed from elimination
methods to information-theoretic, instance-dependent analyses
\citep{mannor2004sample,even2006action,audibert2010best,karnin2013almost}.
The information-theoretic lower bounds of \citet{kaufmann2016complexity} and the
characteristic-time formulation and Track-and-Stop rule of \citet{garivier2016track}
turn these bounds into an adaptive rule that tracks a globally optimal allocation. Methods using lower and
upper confidence bounds (LUCB), including KL-based refinements, form another family
\citep{kalyanakrishnan2012pac,kaufmann2013information,jamieson2014lilucb}.
We fix the best-arm sampling fraction at $\beta$ and use local
leader--challenger updates rather than tracking the globally optimal allocation.

\paragraph{Top-two methods.}
Top-two algorithms repeatedly compare a current leader with a selected
challenger. They were introduced in Bayesian form by \citet{russo2016simple},
studied in fixed-confidence settings by \citet{shang2020fixed}, and developed
into a frequentist bounded-reward framework by \citet{jourdan2022top}.
In that framework, $\beta$-EB-TCI uses the empirical leader. We specialize
the rule to Bernoulli rewards and analyze the challenger index with its exact
logarithmic count penalty, rather than an unpenalized rule, a Gaussian
approximation, or a different top-two construction. The closest existing
non-asymptotic analysis is the Gaussian TTUCB result of
\citet{jourdan2023ucbtop}.

\paragraph{Time-uniform concentration and stopping-time tails.}
Our concentration arguments use mixture martingales and time-uniform
confidence bounds, or confidence sequences
\citep{kaufmann2021mixture,howard2020time,howard2021confidence}.
We convert armwise KL confidence sequences into time-uniform control of pairwise
Bernoulli GLR statistics on the effective-sample-size scale of the two arms.
Correctness and high-probability stopping bounds alone do not imply a sharp
expected stopping-time bound
\citep{balagopalan2025fixing}. We use the GLR stopping-rule error
guarantee of \citet{jourdan2022top} and prove termination and finite-time
challenger-count bounds here. Appendix~\ref{app:prior-results} records the
role of each result used from prior work.

\section{Problem Setting and Algorithm}
\label{sec:setting-algorithm}

We define the bandit model, evidence statistics, allocation benchmarks, and algorithm.

\subsection{Problem Formulation}

Consider $K\ge2$ Bernoulli arms, indexed by $[K]:=\{1,\dots,K\}$, with means
$\mu=(\mu_1,\dots,\mu_K)\in(0,1)^K$.
We assume a unique best arm,
$
  \istar=\argmax_{i\in[K]}\mu_i.
$
The algorithm first pulls each arm once. After $t\ge K$ pulls, $N_{t,i}$ and
$\widehat{\mu}_{t,i}$ denote arm $i$'s pull count and empirical mean;
$I_{t+1}$ denotes the arm sampled at round $t+1$. Unless stated otherwise,
all finite-set $\argmax$ and $\argmin$ operations use fixed deterministic
tie-breaking (smallest-index in experiments), with the exception stated in
Corollary~\ref{cor:distinctmeans}.

A strategy samples adaptively, stops at a stopping time $\tau_{\delta}$ for its
observation and randomization history, and recommends $\widehat{i}_{\tau_{\delta}}$.
Write $\mathbb{P}_{\mu}$ for the law of this history on instance $\mu$.
For $\delta\in(0,1)$, the strategy is $\delta$-correct if, for every instance
$\mu$ with a unique best arm,
\[
  \mathbb{P}_{\mu}(\tau_{\delta}<\infty)=1
  \quad\text{and}\quad
  \mathbb{P}_{\mu}(\widehat{i}_{\tau_{\delta}}\neq \istar)\le\delta.
\]
The goal is to minimize the expected number of samples subject to these conditions.
\emph{Finite-confidence} guarantees hold at each fixed $\delta$, not only as $\delta\to0$.

\subsection{Evidence Statistics}

The Bernoulli divergence, measuring the cost of explaining mean $p$ by $q$, is
defined for $p\in[0,1]$ and $q\in(0,1)$ by
\[
  \KL{p}{q}
  :=
  p\log\frac{p}{q}+(1-p)\log\frac{1-p}{1-q},
\]
with $0\log0=0$ and $1\log1=0$. Here $\Ber(p)$ denotes the Bernoulli law
with mean $p$. For interior arguments,
$\KL{p}{q}=\operatorname{KL}(\Ber(p),\Ber(q))$, the Kullback--Leibler
divergence between these laws. All logarithms are natural.

Evidence that distinct arms $i,j\in[K]$ have different means is measured by
the symmetric two-sample GLR statistic
$
  V_t(i,j)
  :=
  \inf_{x\in(0,1)}
  \left\{
    N_{t,i}\KL{\widehat{\mu}_{t,i}}{x}
    +
    N_{t,j}\KL{\widehat{\mu}_{t,j}}{x}
  \right\}.
$
The infimum is unchanged over $[0,1]$ with the standard Bernoulli-KL boundary
convention. Replacing empirical means by true means isolates allocation
information in the population counterpart
$
  V_t^{\mu}(i,j)
  :=
  \inf_{x\in(0,1)}
  \left\{
    N_{t,i}\KL{\mu_i}{x}
    +
    N_{t,j}\KL{\mu_j}{x}
  \right\}.
$

The one-sided transportation cost of \citet{jourdan2022top} counts evidence
that $i$ beats $j$ only under that empirical ordering:
\[
  W_t(i,j)
  :=
  \begin{cases}
  V_t(i,j), & \widehat{\mu}_{t,i}\ge\widehat{\mu}_{t,j},\\
  0, & \widehat{\mu}_{t,i}<\widehat{\mu}_{t,j}.
  \end{cases}
\]
If $i$ is the empirical leader and $j\neq i$, then
$
  W_t(i,j)=V_t(i,j).
$
The symmetric GLR is minimized at the sample-size-weighted pooled mean,
with the boundary convention above:
$
  x_t(i,j)
  =
  \frac{N_{t,i}\widehat{\mu}_{t,i}+N_{t,j}\widehat{\mu}_{t,j}}
       {N_{t,i}+N_{t,j}}.
$
For $V_t^{\mu}(i,j)$, the pooled mean uses true rather than empirical means.
Appendix~\ref{app:bernoulli-specialization} derives these identities.

\subsection{Characteristic Times}
\label{subsec:characteristic-time}

Prior work gives the leading lower-bound term $\Tstar(\mu)\log(1/\delta)$
for the expected number of samples used by $\delta$-correct strategies as $\delta\to0$
\citep{kaufmann2016complexity,garivier2016track}. The fixed-$\beta$ benchmark
of \citet{jourdan2022top} prescribes the best arm's sampling fraction.

Allocate fraction $\alpha$ of the samples to the best arm
and split $1-\alpha$ among competitors. If arm $j$ receives fraction $\omega$,
its population GLR grows at rate $I_j^{\alpha}(\omega)$ per sample, defined
below. With $K=2$ and arm $1$ best, all remaining samples go to arm $2$,
giving $1/\Tstar_{\alpha}(\mu)=I_2^{\alpha}(1-\alpha)$.

For $\alpha\in(0,1)$, $j\in[K]\setminus\{\istar\}$, and scalar $\omega>0$,
define the pairwise information rate by
\[
  I_j^{\alpha}(\omega)
  :=\inf_{x\in(0,1)}
  \left\{\alpha \KL{\mu_{\istar}}{x}+\omega \KL{\mu_j}{x}\right\}.
\]
For the allocation vector $w=(w_j)_{j\neq\istar}$, the feasible set and
fixed-fraction characteristic time are
\[
\begin{aligned}
  \Delta_{1-\alpha}
  &:=\left\{w\in(0,\infty)^{K-1}:\sum_{j\neq \istar}w_j=1-\alpha\right\},\\
  \Tstar_{\alpha}(\mu)^{-1}
  &:=\max_{w\in\Delta_{1-\alpha}}
  \min_{j\neq \istar} I_j^{\alpha}(w_j).
\end{aligned}
\]
The global characteristic time is
$\Tstar(\mu):=\min_{\alpha\in(0,1)}\Tstar_{\alpha}(\mu)$.

For more than two arms, the smallest pairwise rate limits progress. The best
split equalizes these rates (Appendix~\ref{app:boundary-inverses}), giving the
leading sample requirement $\Tstar_{\beta}(\mu)\log(1/\delta)$ at best-arm
fraction $\beta$. Since $\Tstar_{\beta}(\mu)\ge\Tstar(\mu)$, fixing $\beta$
can increase this requirement. We keep $\beta$ fixed throughout; online
estimation or adaptive choice of $\beta$ is outside our scope.

\subsection{The \texorpdfstring{$\beta$}{beta}-EB-TCI Algorithm}

We now specify the leader--challenger comparisons and stopping rule.
After $t$ pulls, the empirical leader is
$
B_{t+1}\in\argmax_{i\in[K]}\widehat{\mu}_{t,i}.
$
Set $\widehat{i}_t:=B_{t+1}$ as the recommendation. On an unguarded round
before stopping, the challenger minimizes evidence plus a logarithmic count
penalty, favoring less-sampled arms when evidence is comparable:
$
  C_{t+1}\in\argmin_{j\in[K]\setminus\{B_{t+1}\}}
  \left\{W_t(B_{t+1},j)+\log N_{t,j}\right\}.
$
On an unguarded round, pull the leader with probability $\beta\in(0,1)$ and
the challenger with probability $1-\beta$; after stabilization, $\beta$
controls the true best arm's sampling fraction.

The optional \emph{exploration guard} instead pulls a least-sampled arm when
$\min_{i\in[K]}N_{t,i}<\sqrt{t/K}$. The leader is still defined; the challenger
is used only on unguarded rounds. Write $G_t^{\mathrm{tot}}$ for the number of
guarded pulls by time $t$; the \emph{unguarded algorithm} omits the guard.
The GLR stopping threshold of \citet{jourdan2022top} is
\begin{equation}
\label{eq:threshold}
  c(t,\delta)=\log(1/\delta)+2\log(1+t/2)+2+\log(K-1).
\end{equation}
The stopping time and terminal recommendation are
\begin{equation}
\label{eq:stopping-rule}
  \tau_{\delta}
  :=
  \inf\left\{t\ge K:
  \min_{j\neq B_{t+1}}W_t(B_{t+1},j)>c(t,\delta)\right\},
  \qquad
  \widehat{i}_{\tau_{\delta}}:=B_{\tau_{\delta}+1}.
\end{equation}
For bounded rewards and any adaptive sampling rule, \citet{jourdan2022top} give
\[
  \mathbb{P}_{\mu}\!\left(
    \tau_{\delta}<\infty,
    \widehat{i}_{\tau_{\delta}}\neq \istar
  \right)
  \le \delta.
\]
This bounds the probability of a wrong recommendation on runs that stop.
The almost-sure termination results in
Sections~\ref{sec:stable-entry}--\ref{sec:unconditional-closures} then give
$\delta$-correctness in the covered regimes (Appendix~\ref{app:correctness}).
Algorithm~\ref{alg:ebtci} summarizes the rule.

\begin{algorithm}[H]
\caption{$\beta$-EB-TCI with GLR stopping and optional exploration guard.}
\label{alg:ebtci}
\begin{algorithmic}[1]
\STATE Pull each arm once.
\FOR{$t=K,K+1,\dots$}
  \STATE $B_{t+1}\in\argmax_{i\in[K]}\widehat{\mu}_{t,i}$ \hfill \emph{fixed deterministic tie-breaking}
  \STATE $\widehat{i}_t\gets B_{t+1}$
  \IF{$\min_{j\neq B_{t+1}}W_t(B_{t+1},j)>c(t,\delta)$}
    \STATE stop and recommend $\widehat{i}_t$
  \ENDIF
  \IF{the guard is enabled and $\min_{i\in[K]}N_{t,i}<\sqrt{t/K}$}
    \STATE $I_{t+1}\in\argmin_{i\in[K]}N_{t,i}$ \hfill \emph{guarded pull}
  \ELSE
    \STATE $C_{t+1}\in\argmin_{j\in[K]\setminus\{B_{t+1}\}}\left\{W_t(B_{t+1},j)+\log N_{t,j}\right\}$
    \STATE Draw $U_{t+1}\sim\Ber(\beta)$ independently of the past.
    \STATE $I_{t+1}\gets B_{t+1}$ if $U_{t+1}=1$, and $I_{t+1}\gets C_{t+1}$ otherwise.
  \ENDIF
\ENDFOR
\end{algorithmic}
\end{algorithm}

\section{Population Allocation Benchmark and Uniform GLR Control}
\label{sec:oracle-glr}

We develop a population allocation benchmark and a uniform empirical GLR
comparison to control challenger evidence once the empirical leader is correct.
Throughout, bound constants $C_{\cdots}$ are positive and finite unless stated
otherwise; their subscripts and arguments specify dependence, with $K$ included in $\mu$.
Values may differ between results.

\subsection{A Discrete Penalized Allocation Benchmark}
\label{subsec:discrete-oracle}

The algorithm's integer counts and logarithmic count penalty require a discrete
counterpart to $\Tstar_{\alpha}(\mu)$ from Section~\ref{subsec:characteristic-time}.
Let integers $n\ge1$ and $s\ge K-1$ be the best-arm pull count and total
suboptimal-arm pull count. For each suboptimal arm $j$ and integer $m\ge1$, define
$
  \Phi_j(n,m)
  :=
  \inf_{x\in(0,1)}
  \left\{n \KL{\mu_{\istar}}{x}+m \KL{\mu_j}{x}\right\}
  +\log m.
$
For count vectors $m=(m_j)_{j\neq\istar}$, define the discrete penalized
allocation benchmark
\[
  \Gdisc_{n,s}(\mu)
  :=
  \max_{m\in\{1,2,\dots\}^{K-1}:\sum_{j\neq \istar}m_j=s}
  \min_{j\neq \istar}\Phi_j(n,m_j).
\]
This benchmark gives the largest minimum population index at the prescribed
counts, including $\beta$-EB-TCI's logarithmic count penalty. It depends on true
means unavailable to the learner.

\begin{proposition}[Approximation of the discrete penalized benchmark]
\label{prop:oracle}
Fix a compact interval $U\Subset(0,1)$. There exists a constant
$C_{\mu}(U)<\infty$ such that for all integers $n\ge1$ and $s\ge K-1$, if
$t=n+s$, $\alpha=n/t$, and $\alpha\in U$, then
\[
  \left\lvert
    \Gdisc_{n,s}(\mu)
    -
    \left(
      \frac{t}{\Tstar_{\alpha}(\mu)}+\log t
    \right)
  \right\rvert
  \le C_{\mu}(U).
\]
\end{proposition}

Writing $w_j=m_j/t$ separates the continuous contribution
$t/\Tstar_{\alpha}(\mu)$ from the common penalty $\log t$.
Rounding and arm-dependent penalty terms have bounded cost uniformly for $\alpha\in U$
(Appendix~\ref{app:oracle-proof}).

\subsection{Uniform Comparison of Empirical and Population Pairwise GLR Statistics}
\label{subsec:glr-control-main}

We next control the error from replacing true means by adaptive empirical
means. For $t\ge K$ and $\eta\in(0,1)$, define the time-uniform confidence term
$
  H_t(\eta):=\log(K/\eta)+\log\log(et)
$
and, for distinct arms $i,j\in[K]$, the two-arm effective sample size
$
  N_{\mathrm{eff},t}(i,j):=\frac{N_{t,i}N_{t,j}}{N_{t,i}+N_{t,j}}.
$
It satisfies
$\tfrac12\min\{N_{t,i},N_{t,j}\}\le N_{\mathrm{eff},t}(i,j)\le
\min\{N_{t,i},N_{t,j}\}$ and sets the fluctuation scale in the next theorem.

\begin{theorem}[Time-uniform pairwise Bernoulli GLR approximation]
\label{thm:symmetric-glr-control}
There exists an instance-dependent constant $C_{\mu}<\infty$ such that for every
$\eta\in(0,1)$ there is an event $\cG_{\eta}$ with
$\mathbb{P}_{\mu}(\cG_{\eta})\ge1-\eta$ such that, on
$\cG_{\eta}$, simultaneously for all $t\ge K$ and all pairs
$i,j\in[K]$ with $i\neq j$,
\[
  \left\lvert V_t(i,j)-V_t^{\mu}(i,j)\right\rvert
  \le
  C_{\mu}
  \left(
    \sqrt{N_{\mathrm{eff},t}(i,j)H_t(\eta)}
    +
    H_t(\eta)
  \right).
\]
\end{theorem}

For distinct means, $V_t^{\mu}(i,j)$ grows linearly in
$N_{\mathrm{eff},t}(i,j)$, so the error is lower order when
$H_t(\eta)=o(N_{\mathrm{eff},t}(i,j))$. For equal means,
$V_t^{\mu}(i,j)=0$ and the theorem bounds absolute fluctuations
(Appendix~\ref{app:proof-glr-control}).

\section{Sharp Guarantees in the Stabilized Regime}
\label{sec:stable-entry}

Starting from stabilization, we lower-bound evidence against every challenger
(Proposition~\ref{prop:envelope}), invert the stopping threshold
(Theorem~\ref{thm:conditional}), and deduce linear count growth
(Proposition~\ref{prop:linear}). Section~\ref{sec:unconditional-closures}
then bounds the time needed to reach this regime.
For analysis, continue sampling beyond $\tau_{\delta}$, which remains the first
threshold crossing.
For integers $t\ge K$, let
$\alpha_t:=N_{t,\istar}/t$. Fix a compact interval $U\Subset(0,1)$ with
$\beta$ in its interior. A \emph{$U$-stabilization time} is a random time
$t_U\in\{K,K+1,\dots\}\cup\{\infty\}$ such that
$
  B_{t+1}=\istar \text{ and } \alpha_t\in U, \qquad \forall t\ge t_U.
$
Throughout this section, \emph{stabilization} refers to this joint condition.

For an almost surely finite nonnegative random time $T$, an
\emph{upper-quantile bound} is a function $u_T:(0,1]\to[0,\infty)$ satisfying
\[
  \mathbb{P}_{\mu}\left(T\le u_T(\eta)\right)\ge1-\eta
  \qquad\text{for every }\eta\in(0,1].
\]
We call the bound \emph{summable} if
$\sum_{k\ge0}2^{-k}u_T(2^{-k})<\infty$, which implies $\mathbb{E}_{\mu}[T]<\infty$.
Constants in expected bounds may also depend on the assumed upper-quantile
bound or sufficient-exploration parameters.

For $t\ge t_U$ and $j\neq \istar$, define the population penalized index
$
  Q_t(j):=\Phi_j(N_{t,\istar},N_{t,j})
  =V_t^{\mu}(\istar,j)+\log N_{t,j}.
$
Since $B_{t+1}=\istar$,
$\widehat{\mu}_{t,\istar}\ge\widehat{\mu}_{t,j}$ for every $j\neq \istar$
and $W_t(\istar,j)=V_t(\istar,j)$. For $\eta\in(0,1)$, use the looser confidence term
$
  \widetilde{H}_t(\eta):=\log(K/\eta)+\log(et).
$

\begin{proposition}[Uniform lower bound on evidence against all challengers]
\label{prop:envelope}
Fix a compact interval $U\Subset(0,1)$ containing $\beta$ in its interior. There
exist constants $C_{\mu}(U),\overline{C}_{\mu}(U)<\infty$ such that for every
$\eta\in(0,1)$ there is an event $\cE_{\mathrm{env}}(\eta)$ with
$\mathbb{P}_{\mu}(\cE_{\mathrm{env}}(\eta))\ge1-2\eta$ on which the
following holds simultaneously for every $U$-stabilization time $t_U$ and every
$s\ge t_U$:
\[
  \min_{j\neq \istar}Q_s(j)
  \ge
  \Gdisc_{N_{s,\istar},\,s-N_{s,\istar}}(\mu)
  -
  C_{\mu}(U)
  \left(t_U+\sqrt{s\widetilde{H}_s(\eta)}+\widetilde{H}_s(\eta)\right)
  -
  \overline{C}_{\mu}(U).
\]
Consequently, on the same event,
\[
  \min_{j\neq \istar}W_s(\istar,j)
  \ge
  \frac{s}{\Tstar_{\alpha_s}(\mu)}
  -
  C_{\mu}(U)
  \left(t_U+\sqrt{s\widetilde{H}_s(\eta)}+\widetilde{H}_s(\eta)\right)
  -
  \overline{C}_{\mu}(U).
\]
\end{proposition}

The evidence against every challenger tracks the fixed-fraction information
rate, up to stabilization and concentration errors.
Appendix~\ref{app:proof-envelope} gives the proof.

\begin{theorem}[Non-asymptotic stopping bound in the stabilized regime]
\label{thm:conditional}
Fix $\beta\in(0,1)$ and a compact interval $U\Subset(0,1)$ containing $\beta$
in its interior. Let $t_U$ be an almost surely finite $U$-stabilization time.
Then there exist constants $C_{\mu,k}(U)<\infty$, $k\in\{0,1,2,3\}$, such that, for
every $\delta\in(0,1/2]$ and every $\eta\in(0,1/4]$, with probability at least
$1-3\eta$,
\begin{align*}
  \tau_{\delta}
  &\le
  (1+C_{\mu,0}(U))t_U
  +\Tstar_{\beta}(\mu)\log(1/\delta)
  +C_{\mu,1}(U)
  \sqrt{\log(1/\delta)
  \left(\log\frac{K}{\eta}+\log\log\frac{e}{\delta}\right)} \\
  &\quad
  +C_{\mu,2}(U)
  \left(\log\frac{K}{\eta}+\log\log\frac{e}{\delta}\right)
  +C_{\mu,3}(U).
\end{align*}
Moreover, if $t_U$ admits a summable upper-quantile bound, then there exist
constants $C'_{\mu,k}(U)<\infty$, $k\in\{1,2,3\}$, such that, for all
$\delta\in(0,1/2]$,
\[
\begin{aligned}
  \mathbb{E}_{\mu}[\tau_{\delta}]
  &\le
  \Tstar_{\beta}(\mu)\log(1/\delta)
  +
  C'_{\mu,1}(U)
  \sqrt{\log(1/\delta)\log\log(e/\delta)} \\
  &\quad +C'_{\mu,2}(U)\log\log(e/\delta)
  +
  C'_{\mu,3}(U).
\end{aligned}
\]
\end{theorem}

The bound has the exact fixed-$\beta$ leading coefficient, with corrections for
stabilization and concentration errors and bounded remainders
(Appendix~\ref{app:proof-conditional}).
A high-probability bound on $t_U$ yields a high-probability stopping bound;
a summable upper-quantile bound yields an expected stopping-time bound.

\begin{proposition}[Linear growth of every suboptimal-arm pull count in the stabilized regime]
\label{prop:linear}
Fix a compact interval $U\Subset(0,1)$ containing $\beta$ in its interior and
define
$
  \lambda_{\mu,U}
  :=
  \inf_{\alpha\in U}\Tstar_{\alpha}(\mu)^{-1}>0.
$
For every $\eta\in(0,1)$, on $\cE_{\mathrm{env}}(\eta)$ from
Proposition~\ref{prop:envelope}, simultaneously for every $U$-stabilization
time $t_U$, every $s\ge t_U$, and every $j\neq \istar$,
\[
  N_{s,j}
  \ge
  \frac{
    \lambda_{\mu,U}s
    -C_{\mu}(U)
    \left(t_U+\sqrt{s\widetilde{H}_s(\eta)}+\widetilde{H}_s(\eta)\right)
    -\overline{C}_{\mu}(U)-\log s
  }{\KL{\mu_j}{\mu_{\istar}}}.
\]
\end{proposition}

Each suboptimal-arm count grows linearly after a transient controlled by $t_U$
and the concentration error; Appendix~\ref{app:challenger-growth} gives the proof.

\section{Reaching the Stabilized Regime}
\label{sec:unconditional-closures}

We bound entry into the regime of Section~\ref{sec:stable-entry} using
armwise confidence bounds and sufficient sampling.
Here $t_U$ denotes the earliest $U$-stabilization time, or $\infty$ if none exists.
To quantify empirical separation, let $c_0$ be the positive finite universal
constant from Lemma~\ref{lem:armwisecs} in Appendix~\ref{app:armwise-cs}.
For $\eta\in(0,1)$, positive integers $n$, and $j\in[K]\setminus\{\istar\}$, define
\[
  r(n,\eta)
  :=
  \sqrt{\frac{c_0}{2n}
  \left(\log\frac{8K^2}{\eta}+\log\log(en)\right)},
\]
\[
  n_j(\eta)
  :=
  \min\left\{n\ge1:r(n,\eta)\le\frac{\mu_{\istar}-\mu_j}{4}\right\},
  \qquad
  n_{\star}(\eta):=\max_{j\neq \istar}n_j(\eta).
\]
The count $n_j(\eta)$ makes the confidence radius at most one quarter of the
gap between arms $\istar$ and $j$. The first time all count requirements hold is
$
  t_{\mathrm{sep}}(\eta)
  :=
  \inf\left\{
    t\ge K:
    N_{t,\istar}\ge n_{\star}(\eta)
    \text{ and }
    N_{t,j}\ge n_j(\eta)\ \forall j\neq \istar
  \right\}.
$

Let $\cK^{\mathrm{sep}}_{\eta}$ denote the simultaneous armwise event
\[
  \lvert\widehat{\mu}_{t,i}-\mu_i\rvert
  \le r(N_{t,i},\eta),
  \qquad \forall t\ge K,\ i\in[K].
\]
Then $\mathbb{P}_{\mu}(\cK^{\mathrm{sep}}_{\eta})\ge1-\eta$
(Appendix~\ref{app:armwise-cs}).

\begin{corollary}[Time until all arm counts support empirical separation]
\label{cor:separation}
Fix a compact interval $U\Subset(0,1)$ containing $\beta$ in its interior. There
exists $C_{\mu,\beta}(U)<\infty$ such that for every $\eta\in(0,1/4]$, on the
event of Proposition~\ref{prop:envelope},
$
  t_{\mathrm{sep}}(\eta)
  \le
  C_{\mu,\beta}(U)
  \left(t_U+\log(K/\eta)+\log\log(e/\eta)+1\right).
$
\end{corollary}

Linear count growth controls separation by $t_U$ and logarithmic confidence
terms (Appendix~\ref{app:challenger-growth}). We next show why these counts
ensure permanent leader correctness, then bound their accumulation before stabilization.

\begin{proposition}[Sufficient sampling makes the leader permanently correct]
\label{prop:permanent}
For every $\eta\in(0,1)$, on $\cK^{\mathrm{sep}}_{\eta}$, the empirical
leader is permanently correct after $t_{\mathrm{sep}}(\eta)$:
$
  B_{t+1}=\istar,
  \forall t\ge t_{\mathrm{sep}}(\eta).
$
\end{proposition}

The confidence event excludes later incorrect leaders
(Appendix~\ref{app:permanent-proof}).

\begin{proposition}[High-probability stabilization bound for the unguarded algorithm]
\label{prop:poly}
The unguarded $\beta$-EB-TCI algorithm admits
finite exponents $a_{\mu}>0$, $b_{\mu}\ge0$, and a constant
$C_{\mu,\beta,1}<\infty$ such that for every $\eta\in(0,1/4]$,
$
  \mathbb{P}_{\mu}
  \left(
    t_{\mathrm{sep}}(\eta/6)
    \le
    C_{\mu,\beta,1}\eta^{-a_{\mu}}(\log(e/\eta))^{b_{\mu}}
  \right)
  \ge1-\eta/2.
$
For every compact interval $U\Subset(0,1)$ containing $\beta$ in its interior,
there is a constant $C_{\mu,\beta,2}(U)<\infty$ such that
$
  \mathbb{P}_{\mu}
  \left(
    t_U\le C_{\mu,\beta,2}(U)\eta^{-a_{\mu}}(\log(e/\eta))^{b_{\mu}}
  \right)
  \ge1-\eta.
$
\end{proposition}

The earliest $t_U$ is defined on each sample path without reference to $\eta$.
The bounds for all $\eta$ therefore concern the same random time.
Letting $\eta\downarrow0$ gives $t_U<\infty$ almost surely
(Appendix~\ref{app:poly-localization}).

\begin{corollary}[Non-asymptotic high-probability bound for the unguarded algorithm]
\label{cor:unguarded-hp}
Fix a compact interval $U\Subset(0,1)$ containing $\beta$ in its interior.
Under the standing unique-best-arm assumption, the unguarded algorithm admits
constants $C_{\mu,\beta,k}(U)<\infty$, $k\in\{1,2,3\}$, such that for every
$\delta\in(0,1/2]$ and $\eta\in(0,1/4]$,
\begin{align*}
  \mathbb{P}_{\mu}\Bigg(
  \tau_{\delta}
  &\le
  \Tstar_{\beta}(\mu)\log(1/\delta)
  +C_{\mu,\beta,1}(U)
  \sqrt{\log(1/\delta)
  \left(\log(K/\eta)+\log\log(e/\delta)\right)} \\
  &\quad
  +C_{\mu,\beta,2}(U)
  \left(\log(K/\eta)+\log\log(e/\delta)\right)
  +C_{\mu,\beta,3}(U)\eta^{-a_{\mu}}(\log(e/\eta))^{b_{\mu}}
  \Bigg)
  \ge1-4\eta.
\end{align*}
\end{corollary}

Appendix~\ref{app:poly-localization} derives the corollary from
Theorem~\ref{thm:conditional}. The polynomial-in-$1/\eta$ bound need not be
summable, so it gives no expected bound by itself. The next theorem supplies
a condition that does.

\begin{theorem}[Expected stopping time under finite-mean sufficient exploration]
\label{thm:unguarded-suffexp}
Consider the unguarded algorithm and fix a compact interval
$U\Subset(0,1)$ containing $\beta$ in its interior. Assume \emph{finite-mean
sufficient exploration}: there exist a positive finite constant $c_{\mathrm{se}}$ and an
almost surely finite random time $T_{\mathrm{se}}$ with
$\mathbb{E}_{\mu}[T_{\mathrm{se}}]<\infty$ such that
\[
  \min_{i\in[K]}N_{t,i}\ge c_{\mathrm{se}}\sqrt{t/K}
  \qquad\forall t\ge T_{\mathrm{se}}.
\]
Then there exists a $U$-stabilization time $t_U$ that admits
a summable upper-quantile bound. Consequently, there are constants
$C_{\mu,k}(U)<\infty$, $k\in\{1,2,3\}$, such that, for all
$\delta\in(0,1/2]$,
\[
\begin{aligned}
  \mathbb{E}_{\mu}[\tau_{\delta}]
  &\le
  \Tstar_{\beta}(\mu)\log(1/\delta)
  +C_{\mu,1}(U)\sqrt{\log(1/\delta)\log\log(e/\delta)} \\
  &\quad +C_{\mu,2}(U)\log\log(e/\delta)
  +C_{\mu,3}(U).
\end{aligned}
\]
\end{theorem}

Finite-mean sufficient exploration makes stabilization integrable
(Appendix~\ref{app:suff-exp}).

\begin{corollary}[Expected bound for pairwise-distinct means]
\label{cor:distinctmeans}
Assume pairwise-distinct arm means and let the unguarded Bernoulli
$\beta$-EB-TCI rule break ties as in \citet{jourdan2022top}
(e.g., uniformly at random). Then there exist instance-dependent constants
$C_{\mu,\beta,1},C_{\mu,\beta,2},C_{\mu,\beta,3}<\infty$ such that, for all
$\delta\in(0,1/2]$,
\[
  \mathbb{E}_{\mu}[\tau_{\delta}]
  \le
  \Tstar_{\beta}(\mu)\log(1/\delta)
  +C_{\mu,\beta,1}\sqrt{\log(1/\delta)\log\log(e/\delta)}
  +C_{\mu,\beta,2}\log\log(e/\delta)
  +C_{\mu,\beta,3}.
\]
\end{corollary}

The cited exploration result supplies the premise of
Theorem~\ref{thm:unguarded-suffexp} (Appendix~\ref{app:distinct-means}).

\paragraph{A simple exploration guard.}
Pulling a least-sampled arm when its count is below $\sqrt{t/K}$ controls
expected stabilization time under a unique best arm.

\begin{proposition}[Deterministic exploration rate of the guarded algorithm]
\label{prop:guard}
Assume the exploration guard in Algorithm~\ref{alg:ebtci} is enabled. Then for
all $t\ge K$,
$
  \min_{i\in[K]}N_{t,i}\ge \floor{\sqrt{t/K}}-1.
$
Consequently, for every $\eta\in(0,1)$,
$
  t_{\mathrm{sep}}(\eta)\le K\left(n_{\star}(\eta)+1\right)^2.
$
For every compact interval $U\Subset(0,1)$ containing $\beta$ in its interior,
there exists a constant $C_{\mu,\beta}(U)<\infty$ such that the earliest
$U$-stabilization time $t_U$ satisfies, for every $\eta\in(0,1/4]$,
\[
  \mathbb{P}_{\mu}\!\left(t_U\le C_{\mu,\beta}(U)(\log(e/\eta))^2\right)
  \ge1-\eta.
\]
\end{proposition}

This bound is summable (Appendix~\ref{app:guarded-closure}).

\paragraph{Guarded rounds.}
Guarded rounds bypass the $\Ber(\beta)$ leader--challenger coin, changing the
best-arm sampling fraction by at most $G_t^{\mathrm{tot}}/t$.
Proposition~\ref{prop:alpha} in Appendix~\ref{app:alpha} gives concentration;
Lemma~\ref{lem:guard-count} in Appendix~\ref{app:guard-count} gives
$G_t^{\mathrm{tot}}\le\sqrt{Kt}+K$.

\begin{theorem}[Expected stopping time of the guarded algorithm]
\label{thm:guarded}
For any Bernoulli instance $\mu$ with a unique best arm,
Algorithm~\ref{alg:ebtci} with the exploration guard enabled admits
instance-dependent constants $C_{\mu,\beta,1},C_{\mu,\beta,2},C_{\mu,\beta,3}<\infty$ such that, for all
$\delta\in(0,1/2]$,
\[
  \mathbb{E}_{\mu}[\tau_{\delta}]
  \le
  \Tstar_{\beta}(\mu)\log(1/\delta)
  +C_{\mu,\beta,1}\sqrt{\log(1/\delta)\log\log(e/\delta)}
  +C_{\mu,\beta,2}\log\log(e/\delta)
  +C_{\mu,\beta,3}.
\]
\end{theorem}

The guard recovers the fixed-$\beta$ leading term
(Appendix~\ref{app:guarded-closure}).

\paragraph{The unguarded expected bound.}
For the unguarded algorithm under the default convention, finite-mean
sufficient exploration and expected stabilization time remain unresolved when
suboptimal means coincide. We examine one proof route.

Fix a provisional leader $b\in[K]\setminus\{\istar\}$ and a candidate
challenger $u\in[K]\setminus\{b\}$. For positive integer counts $n,m$, define
the population penalized index with the leader count held at $n$ by
$
  \Phi_u^{(b)}(n,m)
  :=
  \inf_{x\in(0,1)}
  \left\{n\KL{\mu_b}{x}+m\KL{\mu_u}{x}\right\}
  +\log m.
$

\begin{proposition}[Limitation of one-step penalized-index comparisons]
\label{prop:barrier}
Fix distinct arms $b,j,u\in[K]$ such that $\mu_b=\mu_j=\lambda<\mu_u$. Let $C>0$ and,
for $\eta\in(0,e^{-1}]$, define $m_u(\eta):=\ceil{C\log(e/\eta)}$. Suppose
$M:(0,e^{-1}]\to\{1,2,\dots\}$ is such that for every sufficiently small $\eta$ and
all sufficiently large $n$,
$
  \Phi_u^{(b)}(n,m_u(\eta))\le \Phi_j^{(b)}(n,M(\eta)).
$
Then for every $a<C\KL{\mu_u}{\lambda}$,
$
  M(\eta)\ge\eta^{-a}
$
for all sufficiently small $\eta$.
\end{proposition}

With $m_u(\eta)=\ceil{C\log(e/\eta)}$ pulls of $u$, the comparison required for
all sufficiently large $n$ requires polynomially
many pulls of the equal-mean suboptimal arm $j$ (Appendix~\ref{app:barrier}).
A global, path-dependent analysis may still establish finite-mean sufficient
exploration with deterministic tie-breaking.

\section{Experiments}
\label{sec:experiments}

We compare empirical mean stopping times with the $\Tstar_{\beta}(\mu)$ predictor
and examine sensitivity to the leader-sampling probability $\beta$, the
guarded-pull fraction, and removal of the logarithmic count penalty. All experiments use
Bernoulli rewards, the stopping threshold in \eqref{eq:threshold}, and
smallest-index tie-breaking.

\paragraph{Protocol.}
We use three representative mean vectors:
$\textsc{sep5}=[0.70,0.50,0.40,0.30,0.20]$ has well-separated arms,
$\textsc{cluster-close5}=[0.70,0.64,0.63,0.62,0.61]$ has distinct but closely
clustered suboptimal arms, and
$\textsc{tie-close4}=[0.60,0.56,0.56,0.56]$ has three equal-mean suboptimal arms.
The evaluation grid combines these instances with $\beta\in\{0.3,0.5,0.7\}$ and
$\delta\in\{10^{-4},10^{-5},10^{-6},10^{-7},10^{-8}\}$. For each setting we
run both the unguarded and guarded algorithms with $200$ independent
replications. We compare the empirical mean stopping time with the
threshold-crossing predictor
\[
  t_{\mathrm{pred}}(\mu,\beta,\delta)
  :=
  \min\{t\ge K:t/\Tstar_{\beta}(\mu)>c(t,\delta)\}.
\]
The minimum is over integer times: the predictor is the first round when the
leading information term $t/\Tstar_{\beta}(\mu)$ exceeds the stopping threshold.
Appendix~\ref{app:exp-details} gives the full grids, random-seed policy,
numerical solver, uncertainty summaries, stopping caps, and additional
diagnostics, together with the code and reproduction commands.

\paragraph{Predictor agreement and sensitivity to $\beta$.}
Figure~\ref{fig:exp-scaling-revised} compares empirical mean stopping times with
$t_{\mathrm{pred}}$ at $\beta=0.5$. The empirical means track the predictor on
both \textsc{sep5} and \textsc{tie-close4}. Across the $45$ guarded settings,
the mean and maximum absolute relative errors of the empirical means
with respect to $t_{\mathrm{pred}}$ are $1.53\%$ and $3.86\%$; for the
unguarded settings, they are $1.66\%$ and $4.20\%$. None of the $18{,}000$
penalized runs reaches the $700{,}000$-round cap or recommends a wrong arm.
Figure~\ref{fig:exp-main-diagnostics} also shows that the dependence on $\beta$
is nonmonotone and instance-specific, as predicted by
$\Tstar_{\beta}(\mu)$; the mean gaps alone do not determine the best
choice of $\beta$.

\begin{figure*}[t]
  \centering
  \includegraphics[width=0.75\textwidth]{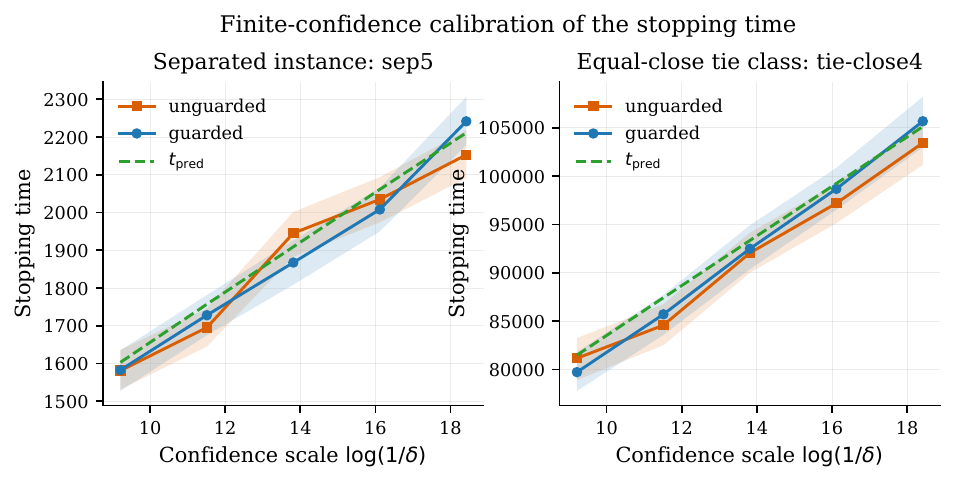}
  \caption{Agreement between observed stopping times and the predictor at $\beta=0.5$.
  Shaded bands are bootstrap percentile $95\%$ confidence intervals for the
  empirical mean over $200$ replications per point. The dashed curve is the
  threshold-crossing predictor $t_{\mathrm{pred}}$.}
  \label{fig:exp-scaling-revised}
\end{figure*}

\begin{figure*}[t]
  \centering
  \begin{minipage}[t]{0.36\textwidth}
    \centering
    \includegraphics[width=\linewidth]{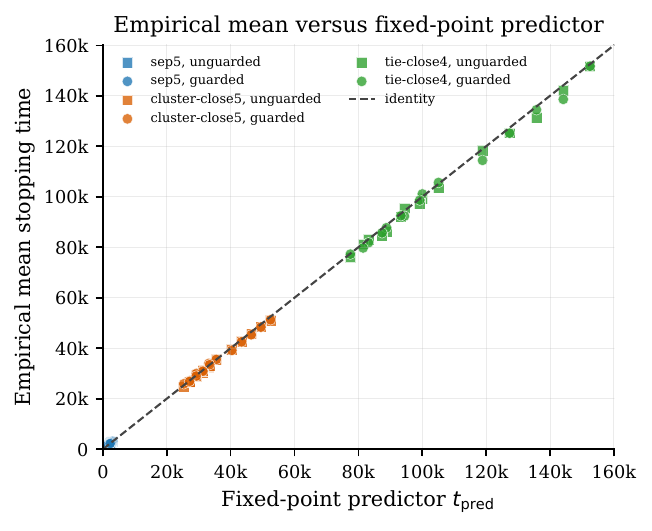}
    \vspace{-0.5em}

    {\small (a) Empirical mean versus $t_{\mathrm{pred}}$.}
  \end{minipage}
  \hfill
  \begin{minipage}[t]{0.60\textwidth}
    \centering
    \includegraphics[width=\linewidth]{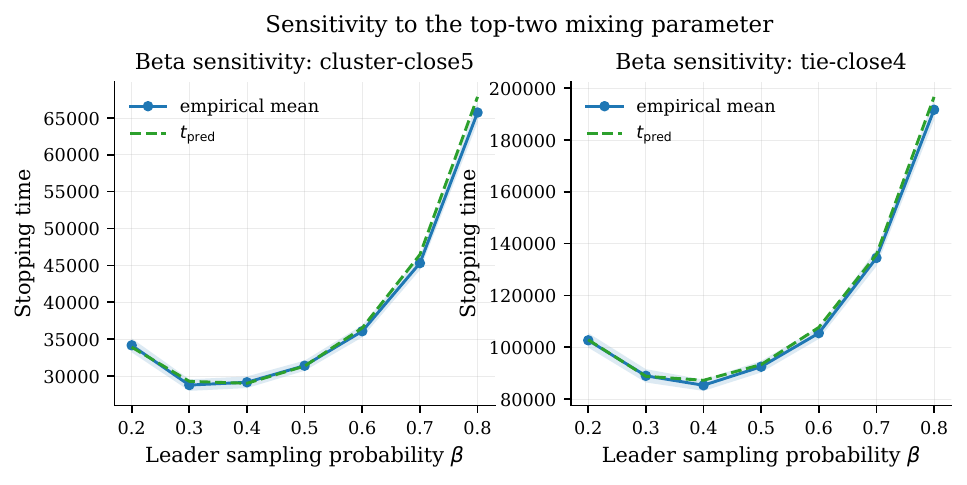}
    \vspace{-0.5em}

    {\small (b) Sensitivity to the leader-sampling probability $\beta$.}
  \end{minipage}
  \caption{Main experimental diagnostics. Left: empirical mean stopping time
  versus the threshold-crossing predictor over the complete evaluation grid for the
  guarded and unguarded algorithms. Right: $\beta$-sensitivity on
  \textsc{cluster-close5} and \textsc{tie-close4} with $\delta=10^{-6}$.}
  \label{fig:exp-main-diagnostics}
\end{figure*}

\paragraph{Guard overhead and penalty ablation.}
On \textsc{tie-close4} with $\beta=0.5$ and $\delta=10^{-8}$, using $600$
replications per penalized variant, the guard changes the median stopping time
from $104{,}918$ to $103{,}743$ and the
empirical $95$th percentile from $134{,}140$ to $130{,}792$. The largest mean
guarded-pull fraction $G_{\tau_{\delta}}^{\mathrm{tot}}/\tau_{\delta}$ in the
tested equal-mean instances is $0.0069\%$.
Removing the penalty $\log N_{t,j}$ raises the mean recorded time from
$105{,}593$ for the unguarded penalized algorithm to $188{,}036$ for the unpenalized
ablation ($500$ replications). In the ablation, $12.8\%$ of runs reach the cap
and contribute the cap value to this mean.
Appendix~\ref{app:exp-tracking-reference} gives an idealized population-allocation
reference that tracks a globally optimal allocation computed from the true means.
It compares stopping-time scales rather than ranking algorithms and is distinct
from empirical Track-and-Stop.

\paragraph{Scope of the comparisons.}
All variants share the stopping threshold and tie-breaking, so differences reflect
the sampling rule alone; the results concern the tested Bernoulli instances and
grid and do not claim superiority over other best-arm identification algorithms.

\section{Conclusion}
\label{sec:conclusion}

For the exact penalized challenger of Bernoulli $\beta$-EB-TCI, we proved linear
growth of every suboptimal-arm count in the stabilized regime and stopping bounds
with exact leading term $\Tstar_{\beta}(\mu)\log(1/\delta)$: with high probability
for the unguarded rule, in expectation under finite-mean sufficient exploration,
and in expectation for every unique-best instance with the guard. Extensions to
other reward distributions, adaptive $\beta$, and finite-mean stabilization of the
unguarded rule with equal-mean suboptimal arms remain open.

\begin{ack}
We thank the anonymous reviewers and the Area Chair for their valuable
feedback, which helped improve this work. We are also grateful to
Marc Jourdan for helpful discussions and comments on the initial
direction of this research.

Nam Nguyen is supported by grant No.~VUNI.2526.CAIR.04.
\end{ack}
\bibliographystyle{plainnat}
\bibliography{references}

@article{mannor2004sample,
  title   = {The Sample Complexity of Exploration in the Multi-Armed Bandit Problem},
  author  = {Mannor, Shie and Tsitsiklis, John N.},
  journal = {Journal of Machine Learning Research},
  volume  = {5},
  pages   = {623--648},
  year    = {2004},
  url     = {https://www.jmlr.org/papers/v5/mannor04b.html}
}

@article{even2006action,
  title   = {Action Elimination and Stopping Conditions for the Multi-Armed Bandit and Reinforcement Learning Problems},
  author  = {Even-Dar, Eyal and Mannor, Shie and Mansour, Yishay},
  journal = {Journal of Machine Learning Research},
  volume  = {7},
  number  = {39},
  pages   = {1079--1105},
  year    = {2006},
  url     = {https://jmlr.org/papers/v7/evendar06a.html}
}

@InProceedings{audibert2010best,
  title     = {Best Arm Identification in Multi-Armed Bandits},
  author    = {Audibert, Jean-Yves and Bubeck, S{\'e}bastien and Munos, R{\'e}mi},
  booktitle = {Proceedings of the Twenty-Third Conference on Learning Theory},
  pages     = {41--53},
  year      = {2010},
  editor    = {Kalai, Adam Tauman and Mohri, Mehryar},
  publisher = {Omnipress},
  isbn      = {978-0-9822529-2-5},
  url       = {https://www.learningtheory.org/colt2010/papers/59Audibert.pdf}
}

@InProceedings{kalyanakrishnan2012pac,
  title     = {{PAC} Subset Selection in Stochastic Multi-Armed Bandits},
  author    = {Kalyanakrishnan, Shivaram and Tewari, Ambuj and Auer, Peter and Stone, Peter},
  booktitle = {Proceedings of the Twenty-Ninth International Conference on Machine Learning},
  pages     = {655--662},
  year      = {2012},
  editor    = {Langford, John and Pineau, Joelle},
  publisher = {Omnipress},
  address   = {New York, NY, USA},
  location  = {Edinburgh, UK},
  month     = {June--July},
  isbn      = {978-1-4503-1285-1},
  url       = {https://www.cs.utexas.edu/~pstone/Papers/bib2html/b2hd-ICML12-shivaram.html}
}

@InProceedings{karnin2013almost,
  title     = {Almost Optimal Exploration in Multi-Armed Bandits},
  author    = {Karnin, Zohar and Koren, Tomer and Somekh, Oren},
  booktitle = {Proceedings of the 30th International Conference on Machine Learning},
  pages     = {1238--1246},
  year      = {2013},
  editor    = {Dasgupta, Sanjoy and McAllester, David},
  volume    = {28},
  number    = {3},
  series    = {Proceedings of Machine Learning Research},
  address   = {Atlanta, Georgia, USA},
  month     = {17--19 Jun},
  publisher = {PMLR},
  url       = {https://proceedings.mlr.press/v28/karnin13.html}
}

@InProceedings{kaufmann2013information,
  title     = {Information Complexity in Bandit Subset Selection},
  author    = {Kaufmann, Emilie and Kalyanakrishnan, Shivaram},
  booktitle = {Proceedings of the 26th Annual Conference on Learning Theory},
  pages     = {228--251},
  year      = {2013},
  editor    = {Shalev-Shwartz, Shai and Steinwart, Ingo},
  volume    = {30},
  series    = {Proceedings of Machine Learning Research},
  address   = {Princeton, NJ, USA},
  month     = {12--14 Jun},
  publisher = {PMLR},
  url       = {https://proceedings.mlr.press/v30/Kaufmann13.html}
}

@InProceedings{jamieson2014lilucb,
  title     = {lil' {UCB}: An Optimal Exploration Algorithm for Multi-Armed Bandits},
  author    = {Jamieson, Kevin and Malloy, Matthew and Nowak, Robert and Bubeck, S{\'e}bastien},
  booktitle = {Proceedings of The 27th Conference on Learning Theory},
  pages     = {423--439},
  year      = {2014},
  editor    = {Balcan, Maria Florina and Feldman, Vitaly and Szepesv{\'a}ri, Csaba},
  volume    = {35},
  series    = {Proceedings of Machine Learning Research},
  address   = {Barcelona, Spain},
  month     = {13--15 Jun},
  publisher = {PMLR},
  url       = {https://proceedings.mlr.press/v35/jamieson14.html}
}

@article{kaufmann2016complexity,
  title   = {On the Complexity of Best-Arm Identification in Multi-Armed Bandit Models},
  author  = {Kaufmann, Emilie and Capp{\'e}, Olivier and Garivier, Aur{\'e}lien},
  journal = {Journal of Machine Learning Research},
  volume  = {17},
  number  = {1},
  pages   = {1--42},
  year    = {2016},
  url     = {https://jmlr.org/papers/v17/kaufman16a.html}
}

@InProceedings{garivier2016track,
  title     = {Optimal Best Arm Identification with Fixed Confidence},
  author    = {Garivier, Aur{\'e}lien and Kaufmann, Emilie},
  booktitle = {29th Annual Conference on Learning Theory},
  pages     = {998--1027},
  year      = {2016},
  editor    = {Feldman, Vitaly and Rakhlin, Alexander and Shamir, Ohad},
  volume    = {49},
  series    = {Proceedings of Machine Learning Research},
  address   = {Columbia University, New York, New York, USA},
  month     = {23--26 Jun},
  publisher = {PMLR},
  url       = {https://proceedings.mlr.press/v49/garivier16a.html}
}

@InProceedings{russo2016simple,
  title     = {Simple Bayesian Algorithms for Best Arm Identification},
  author    = {Russo, Daniel},
  booktitle = {29th Annual Conference on Learning Theory},
  pages     = {1417--1418},
  year      = {2016},
  editor    = {Feldman, Vitaly and Rakhlin, Alexander and Shamir, Ohad},
  volume    = {49},
  series    = {Proceedings of Machine Learning Research},
  address   = {Columbia University, New York, New York, USA},
  month     = {23--26 Jun},
  publisher = {PMLR},
  url       = {https://proceedings.mlr.press/v49/russo16.html}
}

@InProceedings{shang2020fixed,
  title     = {Fixed-Confidence Guarantees for Bayesian Best-Arm Identification},
  author    = {Shang, Xuedong and de Heide, Rianne and M{\'e}nard, Pierre and Kaufmann, Emilie and Valko, Michal},
  booktitle = {Proceedings of the Twenty Third International Conference on Artificial Intelligence and Statistics},
  pages     = {1823--1832},
  year      = {2020},
  editor    = {Chiappa, Silvia and Calandra, Roberto},
  volume    = {108},
  series    = {Proceedings of Machine Learning Research},
  month     = {26--28 Aug},
  publisher = {PMLR},
  url       = {https://proceedings.mlr.press/v108/shang20a.html}
}

@article{howard2020time,
  title   = {Time-Uniform Chernoff Bounds via Nonnegative Supermartingales},
  author  = {Howard, Steven R. and Ramdas, Aaditya and McAuliffe, Jon and Sekhon, Jasjeet},
  journal = {Probability Surveys},
  volume  = {17},
  pages   = {257--317},
  year    = {2020},
  doi     = {10.1214/18-PS321},
  url     = {https://doi.org/10.1214/18-PS321}
}

@article{kaufmann2021mixture,
  title   = {Mixture Martingales Revisited with Applications to Sequential Tests and Confidence Intervals},
  author  = {Kaufmann, Emilie and Koolen, Wouter M.},
  journal = {Journal of Machine Learning Research},
  volume  = {22},
  number  = {246},
  pages   = {1--44},
  year    = {2021},
  url     = {https://jmlr.org/papers/v22/18-798.html}
}

@InProceedings{jourdan2022top,
  title     = {Top Two Algorithms Revisited},
  author    = {Jourdan, Marc and Degenne, R{\'e}my and Baudry, Dorian and de Heide, Rianne and Kaufmann, Emilie},
  booktitle = {Advances in Neural Information Processing Systems},
  pages     = {26791--26803},
  year      = {2022},
  editor    = {Koyejo, Sanmi and Mohamed, S. and Agarwal, Alekh and Belgrave, Danielle and Cho, Kyunghyun and Oh, Alice},
  volume    = {35},
  publisher = {Curran Associates, Inc.},
  doi       = {10.52202/068431-1943},
  url       = {https://proceedings.neurips.cc/paper_files/paper/2022/hash/ab5f5f22e3e09f4424592ffb06840ab0-Abstract-Conference.html}
}

@InProceedings{jourdan2023ucbtop,
  title     = {Non-Asymptotic Analysis of a {UCB}-Based Top Two Algorithm},
  author    = {Jourdan, Marc and Degenne, R{\'e}my},
  booktitle = {Advances in Neural Information Processing Systems},
  pages     = {68980--69020},
  year      = {2023},
  editor    = {Oh, Alice and Naumann, Tristan and Globerson, Amir and Saenko, Kate and Hardt, Moritz and Levine, Sergey},
  volume    = {36},
  publisher = {Curran Associates, Inc.},
  doi       = {10.52202/075280-3019},
  url       = {https://proceedings.neurips.cc/paper_files/paper/2023/hash/d9b564716709357b4bccec9fc9ad04d2-Abstract-Conference.html}
}

@InProceedings{balagopalan2025fixing,
  title     = {Fixing the Loose Brake: Exponential-Tailed Stopping Time in Best Arm Identification},
  author    = {Balagopalan, Kapilan and Nguyen, Ngo Tuan and Zhao, Yao and Jun, Kwang-Sung},
  booktitle = {Proceedings of the 42nd International Conference on Machine Learning},
  pages     = {2603--2645},
  year      = {2025},
  editor    = {Singh, Aarti and Fazel, Maryam and Hsu, Daniel and Lacoste-Julien, Simon and Berkenkamp, Felix and Maharaj, Tegan and Wagstaff, Kiri and Zhu, Jerry},
  volume    = {267},
  series    = {Proceedings of Machine Learning Research},
  month     = {13--19 Jul},
  publisher = {PMLR},
  url       = {https://proceedings.mlr.press/v267/balagopalan25a.html}
}

@article{howard2021confidence,
  title   = {Time-Uniform, Nonparametric, Nonasymptotic Confidence Sequences},
  author  = {Howard, Steven R. and Ramdas, Aaditya and McAuliffe, Jon and Sekhon, Jasjeet},
  journal = {The Annals of Statistics},
  volume  = {49},
  number  = {2},
  pages   = {1055--1080},
  year    = {2021},
  doi     = {10.1214/20-AOS1991},
  url     = {https://doi.org/10.1214/20-AOS1991}
}

@book{lattimore2020bandit,
  title     = {Bandit Algorithms},
  author    = {Lattimore, Tor and Szepesv{\'a}ri, Csaba},
  publisher = {Cambridge University Press},
  year      = {2020}
}
\newpage
\appendix
\begin{center}
{\bfseries\Large Appendix to ``Sharp Non-Asymptotic Analysis of the \\
Penalized Challenger in $\beta$-EB-TCI for Bernoulli Bandits''}
\end{center}
\DoToC
\medskip
\noindent\textbf{Main notation.}

\begin{center}
\begin{tabular}{ll}
\toprule
Symbol & Meaning \\
\midrule
$K$ & number of arms, with $K\ge2$ \\
$\mu=(\mu_1,\dots,\mu_K)$ & vector of Bernoulli means \\
$\istar$ & unique best arm \\
$\delta$ & target recommendation-error probability \\
$\beta$ & probability of sampling the empirical leader on unguarded rounds \\
$N_{t,i}$ & number of pulls of arm $i$ after $t$ total pulls \\
$\widehat{\mu}_{t,i}$ & empirical mean of arm $i$ after $t$ total pulls \\
$B_{t+1}$ & empirical leader at round $t+1$ \\
$C_{t+1}$ & challenger on unguarded round $t+1$ \\
$V_t(i,j)$ & symmetric empirical pairwise Bernoulli GLR \\
$V_t^{\mu}(i,j)$ & population counterpart of $V_t(i,j)$ \\
$W_t(i,j)$ & one-sided transportation cost used by the algorithm \\
$N_{\mathrm{eff},t}(i,j)$ & effective sample size $N_{t,i}N_{t,j}/(N_{t,i}+N_{t,j})$ \\
$\tau_{\delta}$ & first threshold-crossing time \\
$c(t,\delta)$ & stopping threshold \\
$\alpha_t$ & best-arm sampling fraction $N_{t,\istar}/t$ \\
$t_U$ & time from which $B_{t+1}=\istar$ and $\alpha_t\in U$ persist \\
$\Gdisc_{n,s}(\mu)$ & discrete penalized allocation benchmark \\
$\Tstar_{\alpha}(\mu)$ & characteristic time with best-arm fraction fixed to $\alpha$ \\
$\Tstar(\mu)$ & global characteristic time $\min_{\alpha\in(0,1)}\Tstar_{\alpha}(\mu)$ \\
\bottomrule
\end{tabular}
\end{center}

\section{Technical Ingredients from Prior Work}
\label{app:prior-results}

We collect the results from prior work used in the analysis and explain
their roles in our guarantees. Error control, concentration under adaptive
sampling, and sufficient exploration serve distinct roles.

\paragraph{Stopping-rule error guarantee.}
For every $\delta\in(0,1)$, the bounded-reward result of
\citet[Lemma~2]{jourdan2022top}, specialized to the Bernoulli transportation
cost and threshold~\eqref{eq:threshold}, gives
\[
  \mathbb{P}_{\mu}\!\left(
    \tau_{\delta}<\infty,
    \widehat{i}_{\tau_{\delta}}\neq \istar
  \right)\le\delta,
\]
for any adaptive sampling rule. This bounds the probability of stopping with
an incorrect recommendation. It does not establish almost-sure termination,
a stopping-time upper bound, or a conditional error bound given stopping.
Appendix~\ref{app:correctness} verifies that our Bernoulli stopping statistic
is covered by this result.

\paragraph{Time-uniform Bernoulli concentration.}
We use the one-dimensional KL confidence-sequence bound of
\citet[Section~3.3.2, Equation~(14)]{kaufmann2021mixture}.
For every $\eta\in(0,1)$, Lemma~\ref{lem:armwisecs} applies this bound to each
arm and evaluates it at the adaptive pull counts, yielding an event
$\cK_{\eta}$ with $\mathbb{P}_{\mu}(\cK_{\eta})\ge1-\eta$ on which, for a
positive finite universal constant $c_0$,
\[
  N_{t,i}\KL{\widehat{\mu}_{t,i}}{\mu_i}
  \le c_0\!\left(\log\frac{K}{\eta}+\log\log(et)\right)
  \quad\text{for all }t\ge K\text{ and }i\in[K].
\]
This event underlies the empirical-to-population pairwise-GLR comparison.
Entry into the stabilized regime, challenger growth, and stopping require
further analysis of the algorithm.

\paragraph{Finite-mean sufficient exploration for pairwise-distinct means.}
Corollary~\ref{cor:distinctmeans} uses the finite-mean sufficient-exploration theorem of
\citet{jourdan2022top} for pairwise-distinct means and the unguarded variant
specified in that corollary. The cited theorem
is not used in the unconditional high-probability analysis or the guarded
expected stopping-time bound.

\paragraph{Results established in this paper.}
We prove a deterministic approximation of the discrete penalized allocation
benchmark. Using the concentration event above, we establish the simultaneous
empirical-to-population pairwise-GLR comparison, linear challenger growth in
the stabilized regime, and the uniform evidence lower bound. We also prove a
high-probability bound on the time to reach that regime and invert the
stopping-threshold inequality. The exploration guard gives deterministic sufficient
exploration with sublinear pull overhead. The one-step index-comparison
limitation concerns only the stated index-domination proof route; it is not
an impossibility theorem for the algorithm.

\paragraph{Connection to the main theorems.}
The deterministic allocation-benchmark analysis and time-uniform concentration
yield Theorem~\ref{thm:conditional} and the challenger-growth results. The
high-probability stabilization bound gives Corollary~\ref{cor:unguarded-hp}
and almost-sure termination of the unguarded algorithm under a unique best
arm.

If the stabilization time also admits a summable upper-quantile bound,
integration gives the expected stopping-time bound in
Theorem~\ref{thm:unguarded-suffexp}. For pairwise-distinct means, the
finite-mean sufficient-exploration theorem of \citet{jourdan2022top} verifies this premise
for the tie-breaking variant in Corollary~\ref{cor:distinctmeans}. A
high-probability bound on a random time alone need not give a finite
expectation.

For the guarded algorithm, Proposition~\ref{prop:guard} verifies finite-mean
sufficient exploration directly and leads to Theorem~\ref{thm:guarded}. Whenever our
analysis establishes almost-sure termination, the stopping-rule error
guarantee above gives $\delta$-correctness.

\section{Error Control and \texorpdfstring{$\delta$}{delta}-Correctness}
\label{app:correctness}

We derive the Bernoulli transportation cost and apply the bounded-reward
stopping-rule guarantee to control recommendation errors. Combined with
almost-sure termination, this guarantee yields $\delta$-correctness.

\subsection{Bernoulli Specialization of the Transportation Cost}
\label{app:bernoulli-specialization}

We identify the one-sided transportation cost with the symmetric pairwise GLR
when the empirical means are ordered. In the bounded-reward framework of
\citet{jourdan2022top}, let $\mathcal{M}$ be the class of all probability
distributions supported on $[0,1]$, $F\in\mathcal{M}$ a reward distribution,
and $u\in(0,1)$ a test mean. For a reward $X$ with law $G$, define
\[
  K_{\inf}^{+}(F,u)
  :=
  \inf\left\{\operatorname{KL}(F,G):G\in\mathcal{M},\ \mathbb{E}_{G}[X]>u\right\},
\]
\[
  K_{\inf}^{-}(F,u)
  :=
  \inf\left\{\operatorname{KL}(F,G):G\in\mathcal{M},\ \mathbb{E}_{G}[X]<u\right\}.
\]
For distinct arms $i,j\in[K]$ and $t\ge K$, the one-sided transportation cost is
\[
  W_t^{\mathrm{gen}}(i,j)
  :=
  \inf_{x\in(0,1)}
  \left\{
    N_{t,i}K_{\inf}^{-}(F_{t,i},x)
    +
    N_{t,j}K_{\inf}^{+}(F_{t,j},x)
  \right\},
\]
where $F_{t,i}$ is arm $i$'s empirical reward distribution after $t$ pulls.

In our Bernoulli setting, $F_{t,i}=\Ber(\widehat{\mu}_{t,i})$. The next
proposition shows that the information infima over $\mathcal{M}$ reduce to
infima over Bernoulli mean parameters. Remark~\ref{rem:optimal-root} treats
empirical means at $0$ or $1$ using the boundary convention in
Section~\ref{sec:setting-algorithm}.

\begin{proposition}[Bernoulli specialization of the transportation cost]
\label{prop:bernoulli-specialization}
For any $p,u\in(0,1)$,
\[
  K_{\inf}^{+}(\Ber(p),u)
  =
  \begin{cases}
    0, & u<p,\\
    \KL{p}{u}, & u\ge p,
  \end{cases}
  \qquad
  K_{\inf}^{-}(\Ber(p),u)
  =
  \begin{cases}
    \KL{p}{u}, & u\le p,\\
    0, & u>p,
  \end{cases}.
\]
Consequently, for any $p_i,p_j\in(0,1)$ with $p_i\ge p_j$ and any
$N_i,N_j\ge0$,
\[
  \inf_{x\in(0,1)}
  \left\{
    N_iK_{\inf}^{-}(\Ber(p_i),x)
    +
    N_jK_{\inf}^{+}(\Ber(p_j),x)
  \right\}
  =
  \inf_{x\in(0,1)}
  \left\{
    N_i \KL{p_i}{x}+N_j \KL{p_j}{x}
  \right\}.
\]
\end{proposition}

\begin{proof}
We reduce the information infima to Bernoulli candidates and use KL
monotonicity to restrict the ordered-pair minimization to the interval between
the two means. The probability kernel $x\mapsto\Ber(x)$ maps any
$G\in\mathcal{M}$ to $\Ber(\mathbb{E}_{G}[X])$ and leaves $\Ber(p)$ unchanged.
The data-processing inequality therefore gives
$\operatorname{KL}(\Ber(p),G)\ge\KL{p}{\mathbb{E}_{G}[X]}$.
Bernoulli candidates with means in $(0,1)$ attain the resulting infima or
approach their boundary values, so
\[
  K_{\inf}^{+}(\Ber(p),u)=\inf_{q>u}\KL{p}{q},
  \qquad
  K_{\inf}^{-}(\Ber(p),u)=\inf_{q<u}\KL{p}{q},
\]
where both infima are over $q\in(0,1)$.
For $q\in(0,1)$,
\[
  \frac{\partial}{\partial q}\KL{p}{q}
  =
  -\frac{p}{q}+\frac{1-p}{1-q}
  =
  \frac{q-p}{q(1-q)}.
\]
The map $q\mapsto \KL{p}{q}$ is decreasing on $(0,p]$, increasing on $[p,1)$, and
attains its unique minimum $0$ at $q=p$. The formulas for $K_{\inf}^{+}$ and
$K_{\inf}^{-}$ follow, with limiting values at the strict constraint boundaries.

For the second claim, let $x\in(0,1)$ and define
\[
  \Psi(x)
  :=
  N_iK_{\inf}^{-}(\Ber(p_i),x)
  +
  N_jK_{\inf}^{+}(\Ber(p_j),x),
  \qquad p_i\ge p_j.
\]
If $x<p_j$, then
\[
  \Psi(x)=N_i \KL{p_i}{x}\ge N_i \KL{p_i}{p_j}=\Psi(p_j),
\]
because $\KL{p_i}{\cdot}$ is decreasing on $(0,p_i]$. If $x>p_i$, then
\[
  \Psi(x)=N_j \KL{p_j}{x}\ge N_j \KL{p_j}{p_i}=\Psi(p_i),
\]
because $\KL{p_j}{\cdot}$ is increasing on $[p_j,1)$. The infimum is therefore
attained on $[p_j,p_i]$, where the one-sided terms give
\[
  \Psi(x)=N_i \KL{p_i}{x}+N_j\KL{p_j}{x}.
\]
The two infima agree.
\end{proof}

\begin{remark}
\label{rem:optimal-root}
For $x\in(0,1)$, $\KL{0}{x}=-\log(1-x)$ and $\KL{1}{x}=-\log x$.
The same kernel argument and KL monotonicity give the one-sided
formulas at $p=0$ and $p=1$.
The ordered-pair argument then applies on $[p_j,p_i]\cap(0,1)$; if both means
equal the same endpoint, both infima tend to zero as $x$ approaches that
endpoint. The cost identity therefore also covers endpoint empirical means.
For distinct arms $i,j\in[K]$ and $t\ge K$, whenever
$\widehat{\mu}_{t,i}\ge\widehat{\mu}_{t,j}$,
\[
  W_t^{\mathrm{gen}}(i,j)
  =
  \inf_{x\in(0,1)}
  \left\{
    N_{t,i}\KL{\widehat{\mu}_{t,i}}{x}
    +
    N_{t,j}\KL{\widehat{\mu}_{t,j}}{x}
  \right\}
  =V_t(i,j).
\]
For $p_i,p_j\in[0,1]$ and $N_i,N_j\ge0$ with $N_i+N_j>0$, the pooled mean is
\[
  x^{\star}=\frac{N_i p_i+N_j p_j}{N_i+N_j}.
\]
The derivative of $x\mapsto N_i\KL{p_i}{x}+N_j\KL{p_j}{x}$ on $(0,1)$ has the
sign of
$(N_i+N_j)x-N_i p_i-N_j p_j$. If $x^{\star}\in(0,1)$, it is the minimizer;
if $x^{\star}=0$ or $1$, the infimum is approached as $x$ tends to that
endpoint and is attained on $[0,1]$ with the boundary convention in
Section~\ref{sec:setting-algorithm}.
\end{remark}

\subsection{Uniform Error Control}
\label{app:correctness-proof}

The bounded-reward stopping-rule guarantee controls recommendation errors
under either sampling rule; almost-sure termination is established separately.

\begin{theorem}[Uniform error control for the Bernoulli $\beta$-EB-TCI stopping rule]
\label{thm:bernoulli-ebtci-correctness}
For $\delta\in(0,1)$, with the threshold
\[
  c(t,\delta)
  =
  \log(1/\delta)+2\log(1+t/2)+2+\log(K-1),
\]
Algorithm~\ref{alg:ebtci}, with or without the optional exploration guard,
satisfies
\[
  \mathbb{P}_{\mu}
  \left(
    \tau_{\delta}<\infty,
    \widehat{i}_{\tau_{\delta}}\neq \istar
  \right)
  \le\delta.
\]
\end{theorem}

\begin{proof}
We apply \citet[Lemma~2]{jourdan2022top} with the bounded-reward class
$\mathcal{M}$ defined above. By Remark~\ref{rem:optimal-root}, the cost used
by Algorithm~\ref{alg:ebtci} agrees with that lemma's transportation cost,
including at endpoint empirical means. The stopping rule therefore matches
the rule in the lemma, with recommendation
\[
  \widehat{i}_t\in\argmax_{i\in[K]}\widehat{\mu}_{t,i}.
\]
For this threshold and any adaptive sampling rule, the lemma bounds
the probability of stopping with a wrong recommendation by $\delta$.

The optional guard changes how arms are sampled; the statistic
$W_t(\widehat{i}_t,j)$, threshold, and recommendation rule remain the same.
The time-uniform error guarantee therefore applies to both the guarded and
unguarded algorithms.
\end{proof}

Theorem~\ref{thm:bernoulli-ebtci-correctness} controls errors on runs that stop;
it does not establish almost-sure termination for an arbitrary sampling rule.
Termination follows in the regimes of Corollary~\ref{cor:unguarded-hp},
Theorem~\ref{thm:unguarded-suffexp}, Corollary~\ref{cor:distinctmeans}, and
Theorem~\ref{thm:guarded}; combining termination with the displayed error
bound gives $\delta$-correctness in those regimes.

\section{Population Allocation Calculus and the Discrete Penalized Benchmark}
\label{app:oracle}

We characterize the optimal continuous population allocation at a fixed
best-arm fraction and establish its regularity as that fraction varies.
Rounding this allocation gives the discrete penalized allocation benchmark
approximation in Proposition~\ref{prop:oracle}.

\subsection{Bernoulli Two-Arm Calculus}
\label{app:bernoulli-two-arm-calculus}

We first compute the pairwise information rate and its derivatives. Fix
$\alpha\in(0,1)$, $j\in[K]\setminus\{\istar\}$, and a scalar weight $\omega>0$.
Define the population pooled mean
\[
  m_j^{\alpha}(\omega)
  :=
  \frac{\alpha\mu_{\istar}+\omega\mu_j}{\alpha+\omega}.
\]
By Remark~\ref{rem:optimal-root}, the infimum in the definition of
$I_j^{\alpha}(\omega)$ is uniquely attained at $x=m_j^{\alpha}(\omega)$, and
\[
  I_j^{\alpha}(\omega)
  =
  \alpha \KL{\mu_{\istar}}{m_j^{\alpha}(\omega)}
  +
  \omega \KL{\mu_j}{m_j^{\alpha}(\omega)}.
\]

\begin{proposition}[Bernoulli population-allocation calculus]
\label{prop:berncalc}
The map $I_j^{\alpha}$ is $C^{\infty}$ on $(0,\infty)$, strictly increasing,
and strictly concave. More precisely,
\[
  \frac{d}{d\omega}I_j^{\alpha}(\omega)
  =
  \KL{\mu_j}{m_j^{\alpha}(\omega)}>0,
\]
and
\[
  \frac{d^2}{d\omega^2}I_j^{\alpha}(\omega)
  =
  -
  \frac{\alpha^2(\mu_{\istar}-\mu_j)^2}
  {(\alpha+\omega)^3m_j^{\alpha}(\omega)(1-m_j^{\alpha}(\omega))}
  <0.
\]
\end{proposition}

\begin{proof}
We differentiate the two-arm objective at its pooled minimizer.
For $x\in(0,1)$, let
\[
  f_{\omega}(x):=\alpha \KL{\mu_{\istar}}{x}+\omega \KL{\mu_j}{x}.
\]
For $p,x\in(0,1)$,
\[
  \frac{\partial}{\partial x}\KL{p}{x}
  =
  -\frac{p}{x}+\frac{1-p}{1-x}
  =
  \frac{x-p}{x(1-x)}.
\]
Applying this identity gives
\[
  f_{\omega}'(x)
  =
  \frac{(\alpha+\omega)x-(\alpha\mu_{\istar}+\omega\mu_j)}{x(1-x)}.
\]
Hence $f_{\omega}'(x)=0$ if and only if $x=m_j^{\alpha}(\omega)$. Also,
\[
  f_{\omega}''(x)
  =
  \frac{\alpha\mu_{\istar}+\omega\mu_j}{x^2}
  +
  \frac{\alpha(1-\mu_{\istar})+\omega(1-\mu_j)}{(1-x)^2}>0,
\]
so $f_{\omega}$ is strictly convex and the minimizer is unique.

Because $\omega\mapsto m_j^{\alpha}(\omega)$ is smooth, so is $I_j^{\alpha}$.
The envelope theorem gives
\[
  \frac{d}{d\omega}I_j^{\alpha}(\omega)
  =
  \frac{\partial}{\partial \omega}f_{\omega}(m_j^{\alpha}(\omega))
  =
  \KL{\mu_j}{m_j^{\alpha}(\omega)}.
\]
Since $\mu_j<\mu_{\istar}$,
$m_j^{\alpha}(\omega)\in(\mu_j,\mu_{\istar})$, and therefore
$\KL{\mu_j}{m_j^{\alpha}(\omega)}>0$.

For $p,q\in(0,1)$,
\[
  \frac{d}{dq}\KL{p}{q}=\frac{q-p}{q(1-q)},
  \qquad
  \frac{d}{d\omega}m_j^{\alpha}(\omega)
  =
  \frac{\alpha(\mu_j-\mu_{\istar})}{(\alpha+\omega)^2}.
\]
Differentiating once more yields
\[
  \frac{d^2}{d\omega^2}I_j^{\alpha}(\omega)
  =
  \frac{m_j^{\alpha}(\omega)-\mu_j}{m_j^{\alpha}(\omega)(1-m_j^{\alpha}(\omega))}
  \cdot
  \frac{\alpha(\mu_j-\mu_{\istar})}{(\alpha+\omega)^2}.
\]
Using
\[
  m_j^{\alpha}(\omega)-\mu_j
  =
  \frac{\alpha(\mu_{\istar}-\mu_j)}{\alpha+\omega},
\]
we obtain the stated second derivative, which is strictly negative.
\end{proof}

\subsection{Boundary Behavior and Inverse Maps}
\label{app:boundary-inverses}

The endpoint limits determine the ranges of the pairwise information rates
and hence the domains of their inverse maps. These inverses characterize the
unique fixed-$\alpha$ optimizer and its regularity.

\begin{lemma}[Boundary behavior]
\label{lem:bern-boundary}
Fix $\alpha\in(0,1)$ and $j\in[K]\setminus\{\istar\}$. Then
\[
  \lim_{\omega\downarrow0}I_j^{\alpha}(\omega)=0,
  \qquad
  \lim_{\omega\uparrow\infty}I_j^{\alpha}(\omega)
  =
  \alpha \KL{\mu_{\istar}}{\mu_j}.
\]
Consequently, $I_j^{\alpha}$ is a $C^{\infty}$ strictly increasing bijection
from $(0,\infty)$ onto $(0,\alpha \KL{\mu_{\istar}}{\mu_j})$.
\end{lemma}

\begin{proof}
We evaluate the two terms in the information rate as the competitor weight
tends to zero or infinity.
By Proposition~\ref{prop:berncalc},
\[
  I_j^{\alpha}(\omega)
  =
  \alpha \KL{\mu_{\istar}}{m_j^{\alpha}(\omega)}
  +
  \omega \KL{\mu_j}{m_j^{\alpha}(\omega)}.
\]
As $\omega\downarrow0$, $m_j^{\alpha}(\omega)\to\mu_{\istar}$, so both terms vanish. As
$\omega\uparrow\infty$, $m_j^{\alpha}(\omega)\to\mu_j$, and hence
\[
  \alpha \KL{\mu_{\istar}}{m_j^{\alpha}(\omega)}
  \to
  \alpha \KL{\mu_{\istar}}{\mu_j}.
\]
Moreover,
\[
  m_j^{\alpha}(\omega)-\mu_j
  =
  \frac{\alpha(\mu_{\istar}-\mu_j)}{\alpha+\omega}
  =O(\omega^{-1}),
\]
and, for fixed $p\in(0,1)$, $\KL{p}{q}=O((q-p)^2)$ as $q\to p$.
Therefore
$\omega \KL{\mu_j}{m_j^{\alpha}(\omega)}\to0$. The bijection claim follows from
Proposition~\ref{prop:berncalc} and the two endpoint limits.
\end{proof}

\begin{proposition}[Unique fixed-$\alpha$ optimizer]
\label{prop:fixedalpha}
For every $\alpha\in(0,1)$ there exist a unique vector
$w^{\alpha}(\mu)\in\Delta_{1-\alpha}$ and a unique scalar $\lambda_{\alpha}>0$
such that
\[
  I_j^{\alpha}(w_j^{\alpha}(\mu))=\lambda_{\alpha}
  \qquad\forall j\neq \istar.
\]
Moreover, $\lambda_{\alpha}=\Tstar_{\alpha}(\mu)^{-1}$. On compact subsets of
$(0,1)$, the following map is $C^1$:
\[
  \alpha\mapsto
  \left(w^{\alpha}(\mu),\Tstar_{\alpha}(\mu)^{-1}\right).
\]
\end{proposition}

\begin{proof}
We solve for the common information rate using the inverse maps, then apply the
implicit function theorem to obtain regularity.
Fix $\alpha\in(0,1)$ and define
\[
  D_j:=\KL{\mu_{\istar}}{\mu_j}>0,
  \qquad j\neq \istar.
\]
By Lemma~\ref{lem:bern-boundary}, for each fixed $\alpha$ the map
$I_j^{\alpha}:(0,\infty)\to(0,\alpha D_j)$ is a strictly increasing
$C^{\infty}$ bijection. Define
\[
  G_j(\alpha,\lambda):=(I_j^{\alpha})^{-1}(\lambda),
  \qquad
  0<\lambda<\alpha D_j.
\]
To impose the total allocation constraint, define
\[
  D_{\min}:=\min_{j\neq \istar}D_j,
  \qquad
  \varphi(\alpha,\lambda):=\sum_{j\neq \istar}G_j(\alpha,\lambda)-(1-\alpha),
\]
for $(\alpha,\lambda)$ in the domain
\[
  \left\{(\alpha,\lambda):\alpha\in(0,1),\ 0<\lambda<\alpha D_{\min}\right\}.
\]
For fixed $\alpha$, the map $\lambda\mapsto\varphi(\alpha,\lambda)$ is
continuous and strictly increasing. In addition,
\[
  \lim_{\lambda\downarrow0}\varphi(\alpha,\lambda)=-(1-\alpha)<0,
\]
whereas
\[
  \lim_{\lambda\uparrow\alpha D_{\min}}\varphi(\alpha,\lambda)=+\infty,
\]
because the inverse for an arm attaining the minimum directed KL diverges as
$\lambda\uparrow\alpha D_{\min}$. Therefore there exists a unique
$\lambda_{\alpha}\in(0,\alpha D_{\min})$ satisfying
$\varphi(\alpha,\lambda_{\alpha})=0$. Set
$w_j^{\alpha}(\mu):=G_j(\alpha,\lambda_{\alpha})$. Then
$w^{\alpha}(\mu)\in\Delta_{1-\alpha}$ and
$I_j^{\alpha}(w_j^{\alpha}(\mu))=\lambda_{\alpha}$ for every
$j\neq \istar$.

To prove optimality, take a vector $w\in\Delta_{1-\alpha}$ and write
\[
  m(w):=\min_{j\neq \istar}I_j^{\alpha}(w_j).
\]
Since each $I_j^{\alpha}$ is increasing,
\[
  w_j\ge G_j(\alpha,m(w))
  \qquad\forall j\neq \istar.
\]
Summing over $j\neq \istar$ gives
\[
  1-\alpha=
  \sum_{j\neq \istar}w_j
  \ge
  \sum_{j\neq \istar}G_j(\alpha,m(w)).
\]
By strict monotonicity of the sum of inverse maps and the definition of
$\lambda_{\alpha}$, this implies
$m(w)\le\lambda_{\alpha}$. Hence
\[
  \max_{w\in\Delta_{1-\alpha}}\min_{j\neq \istar}I_j^{\alpha}(w_j)
  =\lambda_{\alpha},
\]
and therefore $\lambda_{\alpha}=\Tstar_{\alpha}(\mu)^{-1}$. If $w$ is any
maximizer, then $m(w)=\lambda_{\alpha}$, so
$w_j\ge G_j(\alpha,\lambda_{\alpha})=w_j^{\alpha}(\mu)$ for every
$j\neq \istar$. Both vectors sum to $1-\alpha$, so equality holds in every
coordinate, proving uniqueness.

For regularity, define
\[
  \Psi_j(\alpha,\omega,\lambda):=I_j^{\alpha}(\omega)-\lambda.
\]
Since $\Psi_j$ is $C^1$ and
$\partial_{\omega}\Psi_j(\alpha,\omega,\lambda)=\partial_{\omega} I_j^{\alpha}(\omega)>0$,
the inverse function theorem implies that $G_j$ is $C^1$ on its domain. The residual
$\varphi$ is therefore $C^1$, with
\[
  \partial_{\lambda}\varphi(\alpha,\lambda)
  =
  \sum_{j\neq \istar}\partial_{\lambda}G_j(\alpha,\lambda)>0.
\]
Applying the implicit function theorem to $\varphi(\alpha,\lambda_{\alpha})=0$
shows that $\alpha\mapsto\lambda_{\alpha}$ is $C^1$ on compact subsets of $(0,1)$.
Each coordinate
$\alpha\mapsto w_j^{\alpha}(\mu)=G_j(\alpha,\lambda_{\alpha})$ is therefore
$C^1$ as well.
\end{proof}

\begin{lemma}[Lipschitz control of $\Tstar_{\alpha}(\mu)^{-1}$ on compacts]
\label{lem:T-inv-lipschitz}
Let $U\Subset(0,1)$ be compact. Then there exists a finite constant
$L_{\mu}(U)$ such that for all $\alpha,\alpha'\in U$,
\[
  \left\lvert
  \frac{1}{\Tstar_{\alpha}(\mu)}
  -
  \frac{1}{\Tstar_{\alpha'}(\mu)}
  \right\rvert
  \le L_{\mu}(U)\lvert\alpha-\alpha'\rvert.
\]
\end{lemma}

\begin{proof}
By Proposition~\ref{prop:fixedalpha}, the map
$\alpha\mapsto \Tstar_{\alpha}(\mu)^{-1}$ is $C^1$ on $(0,1)$.
Its derivative is bounded on the compact convex hull of $U$, which lies in
$(0,1)$. The interval between any two points of $U$ lies in this hull, so the
mean value theorem yields the displayed Lipschitz bound.
\end{proof}

\subsection{Discrete Rounding and Proof of Proposition~\ref{prop:oracle}}
\label{app:oracle-proof}

The rounding lemma preserves the total count while changing each coordinate by
at most one. Applied to the continuous optimizer, it bounds the cost of integer
counts in the discrete penalized allocation benchmark.

\begin{lemma}[Integer rounding with a fixed sum]
\label{lem:rounding}
Let $a_1,\dots,a_{K-1}\ge1$ be real numbers satisfying
$\sum_{j=1}^{K-1}a_j=s\in\mathbb{N}$. Then there exist integers $m_1,\dots,m_{K-1}\ge1$ such that
\[
  \sum_{j=1}^{K-1}m_j=s,
  \qquad
  \lvert m_j-a_j\rvert\le1
  \quad\forall j.
\]
\end{lemma}

\begin{proof}
Let $\widetilde{m}_j:=\floor{a_j}$, so $\widetilde{m}_j\ge1$ for all
$j$, and define $r:=s-\sum_{j=1}^{K-1}\widetilde{m}_j$. Because
$\sum_j a_j=s$ and each fractional part satisfies
$a_j-\floor{a_j}<1$, the remainder satisfies $r\in\{0,1,\dots,K-2\}$.
Add $1$ to the $r$ coordinates with largest fractional parts. The resulting
integers satisfy the required properties.
\end{proof}

\begin{proof}[Proof of Proposition~\ref{prop:oracle}]
We obtain an upper bound from the continuous optimum and a lower bound by
rounding its optimizer and bounding the logarithmic count penalty.
Fix a compact interval $U\Subset(0,1)$ and integers $n\ge1$
and $s\ge K-1$ with $n/(n+s)\in U$. Set
\[
  \alpha:=\frac{n}{n+s}\in U,
  \qquad
  t:=n+s.
\]
We show that there exists $C_{\mu}(U)<\infty$, independent of $n$ and $s$, such
that
\[
  \left\lvert
    \Gdisc_{n,s}(\mu)
    -
    \left(\frac{t}{\Tstar_{\alpha}(\mu)}+\log t\right)
  \right\rvert
  \le C_{\mu}(U).
\]

Let $m=(m_j)_{j\neq \istar}$ be any feasible integer allocation and define
$w_j:=m_j/t$. Then $w\in\Delta_{1-\alpha}$ and
\[
  \Phi_j(n,m_j)
  =
  tI_j^{\alpha}(w_j)+\log t+
  \log w_j.
\]
Since $0<w_j<1$, the residual penalty satisfies $\log w_j\le0$, so
\[
  \min_{j\neq \istar}\Phi_j(n,m_j)
  \le
  \log t+t\min_{j\neq \istar}I_j^{\alpha}(w_j)
  \le
  \log t+\frac{t}{\Tstar_{\alpha}(\mu)}.
\]
Maximizing over feasible $m$ yields
\[
  \Gdisc_{n,s}(\mu)
  \le
  \frac{t}{\Tstar_{\alpha}(\mu)}+\log t.
\]

For the lower bound, let $w^{\alpha}(\mu)$ be the optimizer from
Proposition~\ref{prop:fixedalpha}. Since
$\alpha\mapsto w^{\alpha}(\mu)$ is continuous on $U$ and all coordinates are
positive,
\[
  c_U:=\inf_{\alpha\in U}\min_{j\neq \istar}w_j^{\alpha}(\mu)>0.
\]
Proposition~\ref{prop:berncalc} also gives
\[
  L_U:=
  \sup_{\alpha\in U}
  \sup_{j\neq \istar}
  \sup_{\omega\in[c_U/2,1]}
  \left\lvert\frac{d}{d\omega}I_j^{\alpha}(\omega)\right\rvert<\infty.
\]
Assume first that $t\ge2/c_U$. Then
$a_j:=t w_j^{\alpha}(\mu)\ge2$, so Lemma~\ref{lem:rounding} gives integers
$m_j\ge1$ such that
\[
  \sum_{j\neq \istar}m_j=s,
  \qquad
  \lvert m_j-tw_j^{\alpha}(\mu)\rvert\le1.
\]
Hence
\[
  \left\lvert\frac{m_j}{t}-w_j^{\alpha}(\mu)\right\rvert\le\frac1t,
  \qquad
  \frac{m_j}{t}\ge\frac{c_U}{2}.
\]
By the mean value theorem,
\[
  \left\lvert I_j^{\alpha}\left(\frac{m_j}{t}\right)-I_j^{\alpha}(w_j^{\alpha}(\mu))\right\rvert
  \le\frac{L_U}{t}.
\]
Since $I_j^{\alpha}(w_j^{\alpha}(\mu))=\Tstar_{\alpha}(\mu)^{-1}$,
\[
  tI_j^{\alpha}\left(\frac{m_j}{t}\right)
  \ge
  \frac{t}{\Tstar_{\alpha}(\mu)}-L_U.
\]
Moreover, $\log(m_j/t)\ge\log(c_U/2)$. Therefore
\[
  \Phi_j(n,m_j)
  \ge
  \frac{t}{\Tstar_{\alpha}(\mu)}+
  \log t-L_U+
  \log(c_U/2)
  \qquad\forall j\neq \istar.
\]
Taking the minimum over $j$ gives
\[
  \Gdisc_{n,s}(\mu)
  \ge
  \frac{t}{\Tstar_{\alpha}(\mu)}+
  \log t-C_{\mu,1}(U),
\]
where $C_{\mu,1}(U):=L_U-\log(c_U/2)$.

If $t<2/c_U$, only finitely many admissible triples $(n,s,t)$ satisfy
$n+s=t$ and $n/(n+s)\in U$, so their approximation errors can be absorbed
into a finite constant $C_{\mu,0}(U)$. Taking
$C_{\mu}(U):=\max\{C_{\mu,0}(U),C_{\mu,1}(U)\}$ completes the proof.
\end{proof}

\section{Time-Uniform Pairwise GLR Control}
\label{app:glr}

We prove Theorem~\ref{thm:symmetric-glr-control} by combining armwise KL
confidence sequences with a local Lipschitz bound for the Bernoulli GLR.

\subsection{Armwise KL Confidence Sequences}
\label{app:armwise-cs}

We first control each arm's empirical mean uniformly over its pull count.
These bounds then hold at the adaptive counts used by the algorithm.

\begin{lemma}[Armwise KL confidence sequences]
\label{lem:armwisecs}
There exists a positive finite universal constant $c_0$ such that for every $\eta\in(0,1)$
there is an event $\cK_{\eta}$ with
$\mathbb{P}_{\mu}(\cK_{\eta})\ge1-\eta$ and, on
$\cK_{\eta}$, simultaneously for all $t\ge K$ and all arms $i\in[K]$,
\[
  N_{t,i}\KL{\widehat{\mu}_{t,i}}{\mu_i}
  \le
  c_0\left(\log\frac{K}{\eta}+\log\log(et)\right).
\]
Consequently,
\[
  \lvert \widehat{\mu}_{t,i}-\mu_i\rvert
  \le
  \sqrt{\frac{c_0H_t(\eta)}{2N_{t,i}}}.
\]
\end{lemma}

\begin{proof}
We combine the pull-indexed confidence sequences by a union bound, then
evaluate them at the adaptive counts.

For each arm $i\in[K]$, let $(Y_{i,s})_{s\ge1}$ be the reward sequence of arm
$i$, indexed by pull number. These rewards are i.i.d. with law
$\Ber(\mu_i)$. For each positive integer $n$, define
\[
  \overline{Y}_{i,n}:=\frac1n\sum_{s=1}^{n}Y_{i,s}.
\]
By construction,
\[
  \widehat{\mu}_{t,i}=\overline{Y}_{i,N_{t,i}}
  \qquad\text{for all }t\ge K.
\]
Since the algorithm pulls each arm once initially, $N_{t,i}\ge1$.

Apply the one-dimensional time-uniform KL confidence-sequence bound for
one-parameter exponential families from \citet[Section~3.3.2, Equation~(14)]{kaufmann2021mixture},
specialized to Bernoulli distributions. For $\gamma\in(0,1/2]$, the
confidence-dependent term in that bound is at most a universal multiple of
$\log(1/\gamma)$. There is therefore a positive finite universal constant $c_0$
such that, for every $\gamma\in(0,1/2]$, there exists an event
$\cE_i(\gamma)$ with
$\mathbb{P}_{\mu}(\cE_i(\gamma))\ge1-\gamma$. On
$\cE_i(\gamma)$, simultaneously for all $n\ge1$,
\[
  n \KL{\overline{Y}_{i,n}}{\mu_i}
  \le
  c_0\left(\log\frac1\gamma+\log\log(en)\right).
\]
Choose $\gamma=\eta/K<1/2$, using $K\ge2$, and define
$\cK_{\eta}:=\bigcap_{i=1}^{K}\cE_i(\eta/K)$. A union bound gives
$\mathbb{P}_{\mu}(\cK_{\eta})\ge1-\eta$.

Because each arm's bound holds for every sample size $n\ge1$, it also holds
at the adaptive count $n=N_{t,i}$. On $\cK_{\eta}$,
simultaneously for all $t\ge K$ and all $i\in[K]$,
\[
  N_{t,i}\KL{\widehat{\mu}_{t,i}}{\mu_i}
  \le
  c_0\left(\log\frac{K}{\eta}+\log\log(eN_{t,i})\right)
  \le
  c_0\left(\log\frac{K}{\eta}+\log\log(et)\right).
\]
Pinsker's inequality, $\KL{p}{q}\ge2(p-q)^2$, yields the displayed radius bound.
\end{proof}

The pull-indexed construction of $\cK_{\eta}$ above gives
$\cK_{\eta}\subseteq\cK^{\mathrm{sep}}_{\eta}$, the event defined in
Section~\ref{sec:unconditional-closures}. Indeed, Pinsker's inequality at
$n=N_{t,i}$ bounds the estimation error by
$\sqrt{c_0(\log(K/\eta)+\log\log(eN_{t,i}))/(2N_{t,i})}\le r(N_{t,i},\eta)$.
Hence $\mathbb{P}_{\mu}(\cK^{\mathrm{sep}}_{\eta})\ge1-\eta$.

\subsection{Local Lipschitz Control of the Bernoulli GLR}
\label{app:bregman}

We next use an entropy representation of the two-arm GLR to bound its
derivatives on interior mean intervals. The resulting Lipschitz bound scales
with the effective sample size.

\begin{lemma}[Bregman representation and local Lipschitz control]
\label{lem:bregman}
Fix $m,\ell\ge1$ and $p,q\in(0,1)$, and define
\[
  V_{m,\ell}(p,q):=\inf_{x\in(0,1)}\left\{m \KL{p}{x}+\ell \KL{q}{x}\right\}.
\]
For the local weight $\alpha:=m/(m+\ell)$, let
\[
  h(x):=x\log x+(1-x)\log(1-x)\qquad\text{for }x\in(0,1),
\]
and extend $h$ continuously to $[0,1]$. Then
\[
  V_{m,\ell}(p,q)
  =
  (m+\ell)
  \left[
    \alpha h(p)+(1-\alpha)h(q)-h(\alpha p+(1-\alpha)q)
  \right].
\]
Moreover, for every $\eps\in(0,1/2)$ and all
$p,q,u,v\in[\eps,1-\eps]$,
\[
  \lvert V_{m,\ell}(p,q)-V_{m,\ell}(u,v)\rvert
  \le
  \frac{2}{\eps(1-\eps)}
  \frac{m\ell}{m+\ell}
  \left(\lvert p-u\rvert+\lvert q-v\rvert\right).
\]
\end{lemma}

\begin{proof}
We substitute the pooled minimizer into the GLR and bound its derivatives on
the interior interval.

The minimizer in the definition of $V_{m,\ell}(p,q)$ is
$r:=\alpha p+(1-\alpha)q$. Substitution gives the Bregman representation.

For the Lipschitz bound, write
\[
  V_{m,\ell}(p,q)
  =
  (m+\ell)\left[\alpha h(p)+(1-\alpha)h(q)-h(r)\right],
  \qquad
  r=\alpha p+(1-\alpha)q.
\]
Then
\[
  \partial_p V_{m,\ell}(p,q)=m(h'(p)-h'(r)),
  \qquad
  \partial_q V_{m,\ell}(p,q)=\ell(h'(q)-h'(r)).
\]
Since $h''(x)=1/[x(1-x)]$, for $x,y\in[\eps,1-\eps]$,
\[
  \lvert h'(x)-h'(y)\rvert\le\frac{\lvert x-y\rvert}{\eps(1-\eps)}.
\]
Moreover,
\[
  \lvert p-r\rvert=(1-\alpha)\lvert p-q\rvert\le1-\alpha,
  \qquad
  \lvert q-r\rvert=\alpha\lvert p-q\rvert\le\alpha.
\]
It follows that
\[
  \lvert \partial_p V_{m,\ell}(p,q)\rvert
  \le
  \frac{m(1-\alpha)}{\eps(1-\eps)}
  =
  \frac{m\ell}{m+\ell}\frac{1}{\eps(1-\eps)},
\]
and the same bound holds for $\lvert \partial_q V_{m,\ell}(p,q)\rvert$.
Applying the mean value theorem first in the $p$ coordinate and then in the
$q$ coordinate keeps both segments inside the interior square and proves
the stated inequality, with the factor $2$ as slack.
\end{proof}

\subsection{Proof of Theorem~\ref{thm:symmetric-glr-control}}
\label{app:proof-glr-control}

We combine the armwise confidence radii with local Lipschitz control to prove
the uniform comparison between the empirical GLR and its population counterpart.

\begin{proof}[Proof of Theorem~\ref{thm:symmetric-glr-control}]
We bound both GLR statistics directly when one count is small and apply local
Lipschitz control when both counts are large.

Fix $\eta\in(0,1)$.
Work on the armwise confidence event $\cK_{\eta}$ from
Lemma~\ref{lem:armwisecs}. This event has probability at least $1-\eta$. Fix
$t\ge K$ and distinct arms $i,j\in[K]$, and set
\[
  m:=N_{t,i},
  \qquad
  \ell:=N_{t,j},
  \qquad
  H:=H_t(\eta),
  \qquad
  \Neff:=\frac{m\ell}{m+\ell}.
\]
Throughout this proof, $C_{\mu}$ denotes a positive finite constant depending only on
$\mu$; its value may increase between inequalities.
On $\cK_{\eta}$,
\[
  \lvert \widehat{\mu}_{t,i}-\mu_i\rvert
  \le
  \sqrt{\frac{c_0H}{2m}},
  \qquad
  \lvert \widehat{\mu}_{t,j}-\mu_j\rvert
  \le
  \sqrt{\frac{c_0H}{2\ell}}.
\]
Let
\[
  \eps_{\mu}
  :=
  \frac12\min_{a\in[K]}\{\mu_a,1-\mu_a\}.
\]

First suppose that
\[
  \min\{m,\ell\}\le\frac{2c_0}{\eps_{\mu}^2}H.
\]
Assume without loss of generality that $m\le\ell$. Evaluating the empirical
infimum defining $V_t(i,j)$ at $x=\mu_j$ gives
\[
  V_t(i,j)
  \le
  m \KL{\widehat{\mu}_{t,i}}{\mu_j}
  +
  \ell \KL{\widehat{\mu}_{t,j}}{\mu_j}.
\]
The first term is at most $C_{\mu}m$ because
$p\mapsto \KL{p}{\mu_j}$ is bounded on $[0,1]$, and the second term is at most
$c_0H$ on $\cK_{\eta}$. The small-count assumption therefore gives
$V_t(i,j)\le C_{\mu}H$. Evaluating the population infimum at $x=\mu_j$ gives
\[
  V_t^{\mu}(i,j)\le m \KL{\mu_i}{\mu_j}\le C_{\mu}m\le C_{\mu}H.
\]
The two bounds imply
\[
  \lvert V_t(i,j)-V_t^{\mu}(i,j)\rvert\le C_{\mu}H,
\]
in the small-count regime.

Now suppose that
\[
  \min\{m,\ell\}>\frac{2c_0}{\eps_{\mu}^2}H.
\]
Then both empirical means belong to
$[\eps_{\mu}/2,1-\eps_{\mu}/2]$, and all true means belong to the
same interval. Lemma~\ref{lem:bregman}, applied with
$\eps=\eps_{\mu}/2$, gives
\[
  \lvert V_t(i,j)-V_t^{\mu}(i,j)\rvert
  \le
  C_{\mu}\Neff
  \left(\lvert \widehat{\mu}_{t,i}-\mu_i\rvert+\lvert \widehat{\mu}_{t,j}-\mu_j\rvert\right).
\]
Using the armwise confidence radii,
\[
  \lvert V_t(i,j)-V_t^{\mu}(i,j)\rvert
  \le
  C_{\mu}\Neff
  \left(\sqrt{\frac{H}{m}}+\sqrt{\frac{H}{\ell}}\right).
\]
Since $\Neff\le m$ and $\Neff\le\ell$,
\[
  \frac{\Neff}{\sqrt{m}}\le\sqrt{\Neff},
  \qquad
  \frac{\Neff}{\sqrt{\ell}}\le\sqrt{\Neff}.
\]
We obtain
\[
  \lvert V_t(i,j)-V_t^{\mu}(i,j)\rvert
  \le
  C_{\mu}\sqrt{\Neff H},
\]
in the large-count regime. The two regimes together give the bound in
Theorem~\ref{thm:symmetric-glr-control}. Since the armwise bounds hold on the
same event $\cK_{\eta}$ for all arms and all pull counts, the conclusion holds
simultaneously over all $t\ge K$ and all pairs $i\neq j$.
Taking $\cG_{\eta}:=\cK_{\eta}$ completes the proof.
\end{proof}

\section{Evidence Bounds in the Stabilized Regime}
\label{app:envelope}

We compare the population challenger indices with the discrete penalized
allocation benchmark to derive the evidence and count-growth bounds, then
bound the separation time. Algorithmic time indices are integers on the
continued sampling process of Section~\ref{sec:stable-entry}.
Generic constants may change between inequalities, with their dependencies
indicated by subscripts and arguments.

For $t\ge K$, let $G_{t+1}\in\{0,1\}$ indicate whether the exploration guard
is used at round $t+1$; for the unguarded algorithm, $G_{t+1}\equiv0$. We write
\[
  \mathcal{J}:=[K]\setminus\{\istar\},
  \qquad
  d_K:=\lvert\mathcal{J}\rvert=K-1.
\]
For each positive integer $n$ and $j\in\mathcal{J}$, extend the population
penalized index to real $m>0$ by
\[
  \Phi_j(n,m)
  :=
  \inf_{x\in(0,1)}
  \left\{n \KL{\mu_{\istar}}{x}+m \KL{\mu_j}{x}\right\}
  +\log m.
\]
For integer budgets $S\ge d_K$, the discrete benchmark is
\[
  \Gdisc_{n,S}(\mu)
  =
  \max_{m_j\in\{1,2,\dots\}:\sum_{j\in\mathcal{J}}m_j=S}
  \min_{j\in\mathcal{J}}\Phi_j(n,m_j).
\]

\subsection{Continuous Equalizers and Slope Bounds}
\label{app:equalizers-slopes}

We first bound the slopes of the population indices and construct a real
allocation at which all indices agree. Rounding this continuous equalizer
relates its value to the discrete benchmark.

\begin{lemma}[Uniform slope bounds for $\Phi_j$]
\label{lem:phi-slopes}
Fix a compact interval $U\Subset(0,1)$. There exist constants
$0<c_{\mu}(U)\le C_{\mu}(U)<\infty$ such that, whenever
$\alpha=n/(n+S)\in U$, $1\le m\le S+1$, and $j\in\mathcal{J}$,
\[
  c_{\mu}(U)\le\partial_m\Phi_j(n,m)\le C_{\mu}(U).
\]
\end{lemma}

\begin{proof}
We differentiate the index and use the compact interval of allocation
fractions to keep its pooled mean separated from the suboptimal-arm mean.
By the envelope theorem,
\[
  \partial_m\Phi_j(n,m)
  =
  \KL{\mu_j}{x_j(n,m)}+\frac1m,
  \qquad
  x_j(n,m):=\frac{n\mu_{\istar}+m\mu_j}{n+m}.
\]
If $a_U:=\inf U$, then for $1\le m\le S+1$,
\[
  \frac{n}{n+m}
  \ge
  \frac{n}{n+S+1}
  =
  \frac{\alpha(n+S)}{n+S+1}
  \ge
  \frac{a_U}{2},
\]
since $n+S\ge1$. Hence
$x_j(n,m)$ stays uniformly separated from $\mu_j$. Since
$\mu_{\istar}>\mu_j$, the term $\KL{\mu_j}{x_j(n,m)}$ has a positive lower
bound depending only on $(\mu,U)$. The upper bound follows from
$x_j(n,m)\in[\mu_j,\mu_{\istar}]$ and $1/m\le1$.
\end{proof}

\begin{lemma}[Continuous equalizer and small-budget cutoff]
\label{lem:continuous-equalizer}
Fix a compact interval $U\Subset(0,1)$. For each integer $n\ge1$, define
\[
  \gamma_{\min}(n):=\max_{k\in\mathcal{J}}\Phi_k(n,1),
  \qquad
  S_0(n,\mu):=\sum_{j\in\mathcal{J}}\Phi_j^{-1}(n,\gamma_{\min}(n)),
\]
where $\Phi_j^{-1}(n,\gamma)$ is the inverse of
$m\mapsto\Phi_j(n,m)$ on $[1,\infty)$. There exists
$\overline{S}_{\mu}<\infty$ such that $S_0(n,\mu)\le\overline{S}_{\mu}$ for all
$n\ge1$. For every real budget $S\ge S_0(n,\mu)$, there is a unique pair
$(\gamma_{n,S},r_{n,S})$, $r_{n,S}=(r_{j,n,S})_{j\in\mathcal{J}}$, satisfying
\[
  \sum_{j\in\mathcal{J}}r_{j,n,S}=S,
  \qquad
  r_{j,n,S}\ge1,
  \qquad
  \Phi_j(n,r_{j,n,S})=\gamma_{n,S}
  \quad\forall j\in\mathcal{J}.
\]
Moreover, if $t=n+S$, $\alpha=n/t\in U$, and $S\ge S_0(n,\mu)$, then
\[
  \gamma_{n,S}=\frac{t}{\Tstar_{\alpha}(\mu)}+\log t+O_{\mu,U}(1),
  \qquad
  r_{j,n,S}=t w_j^{\alpha}(\mu)+O_{\mu,U}(1),
\]
uniformly over $\alpha\in U$, $S$, and $j\in\mathcal{J}$. If $S$ is an integer,
then
\[
  \lvert\gamma_{n,S}-\Gdisc_{n,S}(\mu)\rvert\le C_{\mu}(U).
\]
\end{lemma}

\begin{proof}
We invert the increasing indices to construct the equalizer, compare it with
the fixed-fraction optimizer, and round its coordinates for integer budgets.
For fixed $n$, the map $m\mapsto\Phi_j(n,m)$ is continuous on $(0,\infty)$ and
strictly increasing. Indeed,
\[
  \partial_m\Phi_j(n,m)=\KL{\mu_j}{x_j(n,m)}+\frac1m>0,
  \qquad
  x_j(n,m)=\frac{n\mu_{\istar}+m\mu_j}{n+m}.
\]
Since $\log m\to+\infty$, we have $\Phi_j(n,m)\to+\infty$ as $m\to\infty$.
The inverse on $[1,\infty)$ is therefore well defined for every
$\gamma\ge\Phi_j(n,1)$.

Let
\[
  D_{\mu,\star}^{\max}:=\max_{k\in\mathcal{J}}\KL{\mu_k}{\mu_{\istar}}.
\]
Evaluating $\Phi_k(n,1)$ at $x=\mu_{\istar}$ gives
$\gamma_{\min}(n)\le D_{\mu,\star}^{\max}$. Since $\Phi_j(n,m)\ge\log m$, any
$m\ge\exp(D_{\mu,\star}^{\max})$ satisfies $\Phi_j(n,m)\ge\gamma_{\min}(n)$. Hence
\[
  \Phi_j^{-1}(n,\gamma_{\min}(n))\le\exp(D_{\mu,\star}^{\max}),
  \qquad
  S_0(n,\mu)\le\overline{S}_{\mu}:=d_K\exp(D_{\mu,\star}^{\max}).
\]

For $S\ge S_0(n,\mu)$, define
\[
  F_n(\gamma):=\sum_{j\in\mathcal{J}}\Phi_j^{-1}(n,\gamma),
  \qquad \gamma\ge\gamma_{\min}(n).
\]
Then $F_n$ is continuous and strictly increasing,
$F_n(\gamma_{\min}(n))=S_0(n,\mu)$, and $F_n(\gamma)\to\infty$ as
$\gamma\to\infty$. Therefore $F_n(\gamma)=S$ has a unique solution
$\gamma_{n,S}$, and the corresponding coordinates
$r_{j,n,S}:=\Phi_j^{-1}(n,\gamma_{n,S})$ give the equalizer. This equalizer
maximizes the real-valued max-min problem: if every index is at least $\gamma$,
then each coordinate must satisfy $m_j\ge\Phi_j^{-1}(n,\gamma)$.
Feasibility therefore requires
$\sum_{j\in\mathcal{J}}\Phi_j^{-1}(n,\gamma)\le S$; the largest such $\gamma$ is
exactly $\gamma_{n,S}$.

Let $\lambda_{\alpha}=\Tstar_{\alpha}(\mu)^{-1}$ and let $w^{\alpha}(\mu)$ be the
unique continuous optimizer from Proposition~\ref{prop:fixedalpha}. Writing
$q_j=r_{j,n,S}/t$, the equalizer identity gives
\[
  \gamma_{n,S}
  =
  tI_j^{\alpha}(q_j)+\log t+\log q_j
  \qquad\forall j\in\mathcal{J}.
\]
For the upper bound, choose $j$ minimizing $I_j^{\alpha}(q_j)$. Since
$\sum_j q_j=1-\alpha$, Proposition~\ref{prop:fixedalpha} gives
$\min_jI_j^{\alpha}(q_j)\le\lambda_{\alpha}$, and $\log q_j\le0$. Hence
$\gamma_{n,S}\le t\lambda_{\alpha}+\log t$. For the lower bound, use
$m_j=tw_j^{\alpha}(\mu)$. Since
$\inf_{\alpha\in U,j}w_j^{\alpha}(\mu)>0$, these coordinates are at least one
for all sufficiently large $t$, so the allocation is feasible and
\[
  \gamma_{n,S}
  \ge
  \min_j\left\{tI_j^{\alpha}(w_j^{\alpha}(\mu))+\log t+\log w_j^{\alpha}(\mu)\right\}
  \ge
  t\lambda_{\alpha}+\log t-C_{\mu}(U).
\]
This proves the estimate for $\gamma_{n,S}$ for all sufficiently large $t$.

We next control $r_{j,n,S}$. Let
$c_U:=\inf_{\alpha\in U,j}w_j^{\alpha}(\mu)>0$. If $q_j<c_U/2$, then uniformly in
$\alpha\in U$ the strict monotonicity of $I_j^{\alpha}$ gives
$I_j^{\alpha}(q_j)\le\lambda_{\alpha}-\rho_U$ for some $\rho_U>0$. For large
$t$, this contradicts the lower bound on $\gamma_{n,S}$. Thus, for all
sufficiently large $t$, $q_j\ge c_U/2$. Since
$\log q_j=O_{\mu,U}(1)$, the equalizer identity and the estimate for $\gamma_{n,S}$
imply
\[
  I_j^{\alpha}(q_j)=\lambda_{\alpha}+O_{\mu,U}(1/t).
\]
On $[c_U/2,1]$, the derivative $\partial_{\omega}I_j^{\alpha}(\omega)$ is
uniformly bounded away from zero. The mean value theorem yields
$\lvert q_j-w_j^{\alpha}(\mu)\rvert\le C_{\mu}(U)/t$, the asserted estimate for
$r_{j,n,S}$ for all sufficiently large $t$.
For any fixed upper bound on $t$, only finitely many integer $n$ occur, and
their admissible real budgets $S$ form compact sets. Since
$1\le r_{j,n,S}\le S$, continuity of $\Phi_j(n,\cdot)$ bounds
$\gamma_{n,S}=\Phi_j(n,r_{j,n,S})$ on these sets. The quantities
$t\lambda_{\alpha}$, $\log t$, and $tw_j^{\alpha}(\mu)$ are bounded there as
well. Enlarging the remainder constants therefore gives both estimates
uniformly over the small-budget range.

Finally, assume $S$ is an integer. The continuous equalizer maximizes the
real-valued max-min problem, so
$\Gdisc_{n,S}(\mu)\le\gamma_{n,S}$. Conversely,
Lemma~\ref{lem:rounding} rounds $r_{n,S}$ to integers
$\overline{r}_j\ge1$ with $\sum_j\overline{r}_j=S$ and
$\lvert\overline{r}_j-r_{j,n,S}\rvert\le1$. Lemma~\ref{lem:phi-slopes} gives
$\Phi_j(n,\overline{r}_j)\ge\gamma_{n,S}-C_{\mu}(U)$, and hence
$\Gdisc_{n,S}(\mu)\ge\gamma_{n,S}-C_{\mu}(U)$.
\end{proof}

\subsection{Guard-Count Control}
\label{app:guard-count}

The number of guard activations grows at most as the square root of time.
We prove this by counting activations within epochs of the guard threshold.

\begin{lemma}[Deterministic bound on the number of guard activations]
\label{lem:guard-count}
For the enabled guard in Algorithm~\ref{alg:ebtci}, $G_{t+1}=1$ if and only if
$\min_{i\in[K]}N_{t,i}<\sqrt{t/K}$, and each guarded pull selects a
least-sampled arm. Then for all $t\ge K$,
\[
  G_t^{\mathrm{tot}}:=\sum_{u=K}^{t-1}G_{u+1}
  \le
  K\left(1+\floor{\sqrt{t/K}}\right)
  \le
  \sqrt{Kt}+K.
\]
\end{lemma}

\begin{proof}
There are at most $K$ guard activations before the threshold leaves each epoch.
For each integer $r\ge0$, define
\[
  \mathcal{I}_r:=\left\{u\ge K:r\le\sqrt{u/K}<r+1\right\}.
\]
We first show by induction that every arm has count at least $r$ at time
$Kr^2$, for each integer $r\ge1$. Initialization gives the case $r=1$.
Assume the claim at $Kr^2$. Within $\mathcal{I}_r$, every guarded pull raises
an arm's count from $r$ to $r+1$, so at most $K$ such pulls are possible.
After the first round of this epoch, the threshold strictly exceeds $r$;
while any count remains $r$, every subsequent round is guarded. The epoch
has $K(2r+1)$ rounds, so all counts reach $r+1$ before the next epoch begins.
This proves the induction and the bound of $K$ guard pulls per epoch.
Once all counts reach $r+1$, the guard cannot activate again in
$\mathcal{I}_r$. There are at most
$1+\floor{\sqrt{t/K}}$ epochs up to time $t$, which proves the claim.
\end{proof}

\subsection{Uniform Challenger-Index Lower Bound}
\label{app:spread}

We track deficits in the challenger counts to bound the smallest population
index relative to the discrete benchmark. Let $U\Subset(0,1)$ be a compact
interval containing $\beta$, and let $t_U$ be a $U$-stabilization time:
\[
  B_{t+1}=\istar,
  \qquad
  \alpha_t:=\frac{N_{t,\istar}}{t}\in U,
  \qquad
  \forall t\ge t_U.
\]
Here the local term $U$-stabilization time means a random time
$t_U\in\{K,K+1,\dots\}\cup\{\infty\}$
satisfying these conditions, even if $\beta$ is an endpoint of $U$.
The main-text convention additionally requires $\beta$ to lie in the
interior of $U$; all applications to Proposition~\ref{prop:envelope} use
that convention.
For $s\ge t_U$, define
\[
  n_s:=N_{s,\istar},
  \qquad
  S_s:=s-n_s,
  \qquad
  m_{s,j}:=N_{s,j},\quad j\in\mathcal{J}.
\]
Then $\sum_{j\in\mathcal{J}}m_{s,j}=S_s$ and
$Q_s(j)=\Phi_j(n_s,m_{s,j})$. Define the empirical penalized index and recall
its population counterpart:
\[
  P_t(j):=V_t(\istar,j)+\log N_{t,j},
  \qquad
  Q_t(j)=V_t^{\mu}(\istar,j)+\log N_{t,j}.
\]
For $\eta\in(0,1)$, define the uniform error bound
\[
  \eps_t
  :=
  C_{\mu}\left(\sqrt{tH_t(\eta)}+H_t(\eta)\right).
\]
On the event $\cG_{\eta}$ from Theorem~\ref{thm:symmetric-glr-control},
simultaneously for all $t\ge K$ and all $j\in\mathcal{J}$,
\[
  \lvert P_t(j)-Q_t(j)\rvert
  =
  \lvert V_t(\istar,j)-V_t^{\mu}(\istar,j)\rvert
  \le
  \eps_t.
\]

\begin{lemma}[Challenger-budget uniform lower bound]
\label{lem:spread}
Fix a compact interval $U\Subset(0,1)$ containing $\beta$. There exists a
constant $C_{\mu}(U)<\infty$ such that, on $\cG_{\eta}$, for every
$U$-stabilization time $t_U$ and every $s\ge t_U$,
\begin{equation}
\label{eq:spread-lower-envelope}
  \min_{j\in\mathcal{J}}Q_s(j)
  \ge
  \Gdisc_{n_s,S_s}(\mu)
  -
  C_{\mu}(U)
  \left(
    t_U+G_s^{\mathrm{tot}}+\eps_s+
    \sup_{t_U\le u\le s}u\lvert\alpha_u-\beta\rvert+1
  \right).
\end{equation}
For the unguarded rule, $G_s^{\mathrm{tot}}\equiv0$.
\end{lemma}

\begin{proof}
We show that unguarded challenger pulls serve a sufficiently large deficit,
then bound the total positive deficit to control every population index.
Let $\overline{S}_{\mu}$ be the uniform cutoff from
Lemma~\ref{lem:continuous-equalizer} and set
\[
  \overline{\alpha}_U:=\sup U<1,
  \qquad
  s_{\mathrm{eq}}:=\ceil{\frac{\overline{S}_{\mu}}{1-\overline{\alpha}_U}}.
\]
If $u\ge s_{\mathrm{eq}}$ and $\alpha_u\in U$, then
$S_u=(1-\alpha_u)u\ge\overline{S}_{\mu}$, so the continuous equalizer
$(\gamma_{n_u,S_u},r_{n_u,S_u})$ exists. If $s<s_{\mathrm{eq}}$, only finitely
many count configurations are possible; after enlarging $C_{\mu}(U)$,
\eqref{eq:spread-lower-envelope} holds because the factor in parentheses is
at least one.
We therefore assume $s\ge s_{\bullet}:=\max\{t_U,s_{\mathrm{eq}}\}$. Then the
equalizer exists for every $u\in[s_{\bullet},s]$.

Define the fixed-$\beta$ challenger target
\[
  p_j^{\beta}:=\frac{w_j^{\beta}(\mu)}{1-\beta},
  \qquad j\in\mathcal{J},
\]
so that $p_j^{\beta}>0$ and $\sum_jp_j^{\beta}=1$. For
$u\in[s_{\bullet},s]$, set
\[
  D_{u,j}:=p_j^{\beta}S_u-N_{u,j},
  \qquad j\in\mathcal{J}.
\]
Write $D_u:=(D_{u,j})_{j\in\mathcal{J}}$; then $\sum_jD_{u,j}=0$.
Lemma~\ref{lem:continuous-equalizer} and the identity
$S_u=(1-\alpha_u)u$ give
\[
  r_{j,n_u,S_u}-p_j^{\beta}S_u
  =
  u\left(
    w_j^{\alpha_u}(\mu)-\frac{1-\alpha_u}{1-\beta}w_j^{\beta}(\mu)
  \right)+O_{\mu,U}(1).
\]
Consider the map
\[
  \alpha\mapsto w_j^{\alpha}(\mu)-\frac{1-\alpha}{1-\beta}w_j^{\beta}(\mu).
\]
By Proposition~\ref{prop:fixedalpha}, its derivative
$\partial_{\alpha}w_j^{\alpha}(\mu)+p_j^{\beta}$ is uniformly bounded for
$\alpha,\beta\in U$, since $p_j^{\beta}\le1$. The map vanishes at
$\alpha=\beta$, so, uniformly over
$u\in[s_{\bullet},s]$,
\begin{equation}
\label{eq:spread-target-continuous-error}
  \lvert r_{j,n_u,S_u}-p_j^{\beta}S_u\rvert
  \le
  C_{\mu}(U)(u\lvert\alpha_u-\beta\rvert+1).
\end{equation}

We first relate the selected arm's deficit to the largest deficit.
Let $u+1\le s$ be an unguarded challenger-pull round with
$u\ge s_{\bullet}$, and let $a=C_{u+1}$. Put
$M_u:=\max_{j\in\mathcal{J}}D_{u,j}$. We claim that, for constants
$\kappa_{\mu}(U)\in(0,1]$ and $A_{\mu}(U)<\infty$,
\begin{equation}
\label{eq:spread-largest-deficit}
  D_{u,a}
  \ge
  \kappa_{\mu}(U)M_u
  -
  A_{\mu}(U)(\eps_u+u\lvert\alpha_u-\beta\rvert+1).
\end{equation}
Let $j^{+}\in\argmax_{j\in\mathcal{J}}D_{u,j}$ and define
$e_{u,j}:=r_{j,n_u,S_u}-N_{u,j}$. By
\eqref{eq:spread-target-continuous-error},
\begin{equation}
\label{eq:spread-positive-excess}
  e_{u,j^{+}}
  \ge
  M_u-C_{\mu}(U)(u\lvert\alpha_u-\beta\rvert+1).
\end{equation}
Choose the threshold constant in the following case split larger than both
the error coefficient in \eqref{eq:spread-positive-excess} and the error
coefficient in the lower bound for $\gamma_{n_u,S_u}-Q_u(a)$ below divided
by $c_{\mu}(U)$. These coefficients are fixed by the slope bounds and
\eqref{eq:spread-target-continuous-error}.
If $M_u\le C_{\mu}(U)(\eps_u+u\lvert\alpha_u-\beta\rvert+1)$, then
\eqref{eq:spread-largest-deficit} follows after increasing $A_{\mu}(U)$,
because the zero-sum property implies $D_{u,a}\ge-d_KM_u$. Otherwise
$e_{u,j^{+}}>0$, and the slope lower bound in Lemma~\ref{lem:phi-slopes} gives
\[
  Q_u(j^{+})
  \le
  \gamma_{n_u,S_u}
  -c_{\mu}(U)M_u
  +C_{\mu}(U)(u\lvert\alpha_u-\beta\rvert+1).
\]
Since $a$ minimizes $P_u$ and $\lvert P_u(j)-Q_u(j)\rvert\le\eps_u$ on
$\cG_{\eta}$,
\[
  Q_u(a)\le Q_u(j^{+})+2\eps_u.
\]
The difference $\gamma_{n_u,S_u}-Q_u(a)$ is bounded below by
$c_{\mu}(U)M_u-C_{\mu}(U)(\eps_u+u\lvert\alpha_u-\beta\rvert+1)$. In the present
case the choice of threshold makes this lower bound positive, so
$e_{u,a}>0$, and the slope upper bound
yields
\[
  e_{u,a}
  \ge
  c'_{\mu}(U)M_u
  -C'_{\mu}(U)(\eps_u+u\lvert\alpha_u-\beta\rvert+1).
\]
Combining this with \eqref{eq:spread-target-continuous-error} proves
\eqref{eq:spread-largest-deficit}.

For a real $x$, write $(x)_{+}:=\max\{x,0\}$. Define
$D_u^{+}:=\sum_{j\in\mathcal{J}}(D_{u,j})_{+}$. Since the deficits sum to
zero, $M_u\le D_u^{+}$. Fix the final time $s$ and set
\[
  \Xi_s:=\eps_s+
  \sup_{t_U\le v\le s}v\lvert\alpha_v-\beta\rvert+1,
  \qquad
  \Lambda_s:=\frac{A_{\mu}(U)\Xi_s+1}{\kappa_{\mu}(U)}.
\]
Since $\eps_u$ is nondecreasing, \eqref{eq:spread-largest-deficit}
gives the following bound for every unguarded challenger pull
$u+1\le s$ with $u\ge s_{\bullet}$:
\[
  D_{u,C_{u+1}}
  \ge
  \kappa_{\mu}(U)M_u-A_{\mu}(U)\Xi_s.
\]
If $u+1$ is a leader pull, then $D_{u+1}=D_u$. If it is a challenger pull of arm
$a$, then
\[
  D_{u+1,j}=D_{u,j}+p_j^{\beta}-\one\{j=a\}.
\]
Consider an unguarded challenger pull. If $M_u>\Lambda_s$, then
$D_{u,a}>1$. Adding $p^{\beta}:=(p_j^{\beta})_{j\in\mathcal{J}}$ can increase
$D^{+}$ by at most
$\sum_jp_j^{\beta}=1$, while subtracting one from coordinate $a$ decreases the
positive mass by exactly one, since that coordinate remains positive.
Hence $D_{u+1}^{+}\le D_u^{+}$. If
$M_u\le\Lambda_s$, then $D_u^{+}\le d_K\Lambda_s$, and a single challenger
update can increase $D^{+}$ by at most one. Therefore every unguarded challenger
pull satisfies
\[
  D_{u+1}^{+}\le\max\{D_u^{+},d_K\Lambda_s+1\}.
\]
A guard pull need not serve a large deficit: if it pulls $\istar$, the deficit
vector $D$ is unchanged; if it pulls a suboptimal arm, the same update holds,
increasing $D^{+}$ by at most one.

Induction from $s_{\bullet}$ to $s$ gives
\[
  D_s^{+}
  \le
  \max\{D_{s_{\bullet}}^{+},d_K\Lambda_s+1\}
  +G_s^{\mathrm{tot}}.
\]
Moreover,
\[
  D_{s_{\bullet}}^{+}
  \le
  \sum_{j\in\mathcal{J}}p_j^{\beta}S_{s_{\bullet}}
  =S_{s_{\bullet}}
  \le s_{\bullet}
  \le t_U+s_{\mathrm{eq}},
\]
and $s_{\mathrm{eq}}$ is absorbed into the constant. Hence
\[
  D_s^{+}
  \le
  C_{\mu}(U)
  \left(
    t_U+G_s^{\mathrm{tot}}+\eps_s+
    \sup_{t_U\le u\le s}u\lvert\alpha_u-\beta\rvert+1
  \right).
\]
By \eqref{eq:spread-target-continuous-error}, this implies
\[
  \max_{j\in\mathcal{J}}(r_{j,n_s,S_s}-N_{s,j})_{+}
  \le
  C_{\mu}(U)
  \left(
    t_U+G_s^{\mathrm{tot}}+\eps_s+
    \sup_{t_U\le u\le s}u\lvert\alpha_u-\beta\rvert+1
  \right).
\]
Let the right-hand side be $E_s$. Then
$N_{s,j}\ge r_{j,n_s,S_s}-E_s$ for all $j$. If
$N_{s,j}\ge r_{j,n_s,S_s}$, monotonicity gives
$Q_s(j)\ge\gamma_{n_s,S_s}$. Otherwise Lemma~\ref{lem:phi-slopes} gives
$Q_s(j)\ge\gamma_{n_s,S_s}-C_{\mu}(U)E_s$. Thus
\[
  \min_{j\in\mathcal{J}}Q_s(j)
  \ge
  \gamma_{n_s,S_s}-C_{\mu}(U)E_s.
\]
Using $\gamma_{n_s,S_s}\ge\Gdisc_{n_s,S_s}(\mu)-C_{\mu}(U)$ from
Lemma~\ref{lem:continuous-equalizer} proves the lemma.
\end{proof}

\subsection{Proof of Proposition~\ref{prop:envelope}}
\label{app:proof-envelope}

We combine Lemma~\ref{lem:spread} with leader-fraction concentration to
prove the uniform evidence bound.

\begin{proof}[Proof of Proposition~\ref{prop:envelope}]
We absorb guard and allocation fluctuations into the confidence remainder,
then cancel the logarithmic count penalty using the discrete benchmark.
Fix the compact interval $U$ from Proposition~\ref{prop:envelope}, with
$\beta$ in its interior, and $\eta\in(0,1)$. Recall
\[
  \widetilde{H}_s(\eta):=\log\frac{K}{\eta}+\log(es).
\]
Let $\cA_{\eta}$ be the event from Proposition~\ref{prop:alpha}. Define
\[
  \cE_{\mathrm{env}}(\eta):=\cG_{\eta}\cap\cA_{\eta}.
\]
By Theorem~\ref{thm:symmetric-glr-control} and Proposition~\ref{prop:alpha},
$\mathbb{P}_{\mu}(\cE_{\mathrm{env}}(\eta))\ge1-2\eta$. On
$\cE_{\mathrm{env}}(\eta)$, Lemma~\ref{lem:spread} gives, for every
$s\ge t_U$,
\[
  \min_{j\in\mathcal{J}}Q_s(j)
  \ge
  \Gdisc_{n_s,S_s}(\mu)
  -C_{\mu}(U)
  \left(
    t_U+G_s^{\mathrm{tot}}+\eps_s+
    \sup_{t_U\le u\le s}u\lvert\alpha_u-\beta\rvert+1
  \right).
\]
On $\cA_{\eta}$, Proposition~\ref{prop:alpha} applied with $m=t_U$
gives, uniformly over $u\in[t_U,s]$,
\[
  u\lvert\alpha_u-\beta\rvert
  \le
  C_{\beta}t_U+C_{\beta}\sqrt{s\widetilde{H}_s(\eta)}+G_s^{\mathrm{tot}}.
\]
If the guard is disabled, $G_s^{\mathrm{tot}}=0$. If it is enabled,
Lemma~\ref{lem:guard-count} gives
$G_s^{\mathrm{tot}}\le\sqrt{Ks}+K\le C_K\sqrt{s\widetilde{H}_s(\eta)}$.
Since
\[
  \eps_s
  \le
  C_{\mu}\left(\sqrt{s\widetilde{H}_s(\eta)}+\widetilde{H}_s(\eta)\right),
\]
collecting terms gives the population-index bound with constants
$C_{\mu}(U)$ and $\overline{C}_{\mu}(U)$ as in
Proposition~\ref{prop:envelope}.

To pass from population to empirical indices, on $\cG_{\eta}$ we have
\[
  \lvert P_s(j)-Q_s(j)\rvert
  \le
  C_{\mu}
  \left(\sqrt{sH_s(\eta)}+H_s(\eta)\right)
  \le
  C_{\mu}
  \left(\sqrt{s\widetilde{H}_s(\eta)}+\widetilde{H}_s(\eta)\right).
\]
Hence
\[
  \min_jP_s(j)
  \ge
  \min_jQ_s(j)
  -C_{\mu}
  \left(\sqrt{s\widetilde{H}_s(\eta)}+\widetilde{H}_s(\eta)\right).
\]
After entry into the stabilized regime, $B_{s+1}=\istar$, so
$\widehat{\mu}_{s,\istar}\ge\widehat{\mu}_{s,j}$ for every
$j\neq \istar$. Therefore $W_s(\istar,j)=V_s(\istar,j)$ for every
$j\neq \istar$, and
\[
  W_s(\istar,j)=P_s(j)-\log N_{s,j}\ge P_s(j)-\log s.
\]
Combining this with Proposition~\ref{prop:oracle},
\[
  \Gdisc_{n_s,S_s}(\mu)
  \ge
  \frac{s}{\Tstar_{\alpha_s}(\mu)}+
  \log s-C_{\mu}(U),
\]
cancels the logarithmic count penalty $\log s$ and proves the
transportation-cost bound after enlarging the same two constants if needed.
\end{proof}

\subsection{Challenger Growth and Separation}
\label{app:challenger-growth}

The evidence bound forces every suboptimal-arm count to grow linearly.
We use this growth to bound the time at which all separation thresholds are met.

\begin{proof}[Proof of Proposition~\ref{prop:linear}]
We compare the population-index lower bound with an upper bound in terms of
the suboptimal-arm count.
Fix $\eta\in(0,1)$ and work on $\cE_{\mathrm{env}}(\eta)$ from
Proposition~\ref{prop:envelope}. For any $U$-stabilization time $t_U$,
$s\ge t_U$, and $j\neq \istar$, that proposition gives
\[
  Q_s(j)
  \ge
  \frac{s}{\Tstar_{\alpha_s}(\mu)}
  -C_{\mu}(U)
  \left(t_U+\sqrt{s\widetilde{H}_s(\eta)}+\widetilde{H}_s(\eta)\right)
  -\overline{C}_{\mu}(U).
\]
Since $\alpha_s\in U$, the leading term is at least
$\lambda_{\mu,U}s$. On the other hand, evaluating the population transportation
cost at $x=\mu_{\istar}$ yields
\[
  Q_s(j)
  =
  V_s^{\mu}(\istar,j)+\log N_{s,j}
  \le
  \KL{\mu_j}{\mu_{\istar}}N_{s,j}+\log s.
\]
Combining the two inequalities gives the claim.
\end{proof}

\begin{lemma}[Logarithmic separation thresholds]
\label{lem:logsep}
There exists $A_{\mu}<\infty$ such that for all $\eta\in(0,1/2]$,
\[
  n_{\star}(\eta)
  \le
  A_{\mu}
  \left(\log\frac{K}{\eta}+\log\log\frac{e}{\eta}+1\right).
\]
\end{lemma}

\begin{proof}
We compare the logarithmic confidence terms with a sample count proportional
to the confidence scale.
For each $j\neq \istar$, let $\Delta_j:=\mu_{\istar}-\mu_j>0$. The
condition defining $n_j(\eta)$ is
\[
  \frac{c_0}{2n}
  \left(\log\frac{8K^2}{\eta}+\log\log(en)\right)
  \le
  \frac{\Delta_j^2}{16}.
\]
For a fixed positive constant $A$ to be chosen, set
\[
  n=\ceil{A\left(\log\frac{K}{\eta}+\log\log\frac{e}{\eta}+1\right)}.
\]
At this chosen count, $n\le C_K(A+1)\log(e/\eta)$, and hence
$\log\log(en)\le C_K+\log(A+1)+\frac12\log(e/\eta)$ uniformly over
$\eta\in(0,1/2]$. The numerator in the defining condition is therefore
bounded by a constant multiple of the sum of the confidence scale above and
$\log(A+1)$. Since $A$ dominates $\log(A+1)$ as $A$ increases, a sufficiently
large $A$ makes the condition hold for every $j\neq\istar$. By the minimum
defining $n_j(\eta)$, $n_j(\eta)\le n$; no estimate at larger counts is needed.
Taking the maximum over $j$ and absorbing the ceiling into $A_{\mu}$ proves
the claim.
\end{proof}

\begin{proof}[Proof of Corollary~\ref{cor:separation}]
We choose a time proportional to the stabilization time and the confidence
scale, then apply the count-growth bound to meet every separation threshold.
Fix $\eta\in(0,1/4]$ and work on $\cE_{\mathrm{env}}(\eta)$ for any
$U$-stabilization time $t_U$. Let
\[
  R_{\eta}:=\log\frac{K}{\eta}+\log\log\frac{e}{\eta}+1.
\]
By Lemma~\ref{lem:logsep}, $n_{\star}(\eta)\le A_{\mu}R_{\eta}$. Choose
\[
  s=\ceil{C_{\mu,\beta}(U)(t_U+R_{\eta})},
\]
with the constant to be fixed. For any $\rho>0$, the elementary inequalities
$\sqrt{s\widetilde{H}_s(\eta)}\le\rho s+C_{\rho}\widetilde{H}_s(\eta)$ and
$\log(es)\le\rho s+C_{\rho}$ imply
\[
  \sqrt{s\widetilde{H}_s(\eta)}+
  \widetilde{H}_s(\eta)+\log s
  \le
  \rho s+C_{\mu,\rho}(U)R_{\eta}.
\]
Taking $\rho$ small in Proposition~\ref{prop:linear} gives, for each
$j\neq \istar$,
\[
  N_{s,j}
  \ge
  c_{\mu,j}(U)s-C'_{\mu,j}(U)t_U-C''_{\mu,j}(U)R_{\eta}.
\]
Choosing $C_{\mu,\beta}(U)$ sufficiently large ensures $N_{s,j}\ge n_j(\eta)$ for
all $j\neq \istar$. Since $s\ge t_U$, entry into the stabilized regime gives
$N_{s,\istar}=\alpha_s s\ge(\inf U)s$, and the same choice of
$C_{\mu,\beta}(U)$ ensures $N_{s,\istar}\ge n_{\star}(\eta)$. Therefore
$s\ge t_{\mathrm{sep}}(\eta)$, which proves the claim.
\end{proof}

\section{Leader-Fraction Control and Threshold Inversion}
\label{app:conditional}

We bound deviations of the best-arm fraction after permanent leader correctness
and invert the stopping threshold. Together with the evidence envelope, these estimates
prove the stopping bounds in Theorem~\ref{thm:conditional}.

\subsection{Uniform Leader-Fraction Concentration}
\label{app:alpha}

We bound fluctuations of the sampling coins simultaneously over all start times,
so the bound also applies at a random stabilization time. All algorithmic time
indices refer to the continued sampling process defined in
Section~\ref{sec:stable-entry}.

For the start-time range in Proposition~\ref{prop:alpha}, extend the indicators
to initialization rounds by setting $G_{u+1}:=1$ and $B_{u+1}:=I_{u+1}$ for
$1\le u<K$. Here $I_{u+1}$ is the prescribed initialization pull, so
$B_{u+1}$ is only an auxiliary arm label in this range. This convention counts
initialization pulls in $G_t^{\mathrm{tot}}(m)$ when $m<K$; the algorithm and
the total guard count $G_t^{\mathrm{tot}}$, which starts after initialization,
are unchanged. In this proposition, also set $\alpha_t:=N_{t,\istar}/t$ for
every integer $t\ge1$.

\begin{proposition}[Leader-fraction control, uniform in the start time]
\label{prop:alpha}
There exists a constant $C_{\beta}<\infty$ such that for every $\eta\in(0,1)$ there
exists an event $\cA_{\eta}$ with
$\mathbb{P}_{\mu}(\cA_{\eta})\ge1-\eta$ on which the following holds. For all
integers $m\ge1$ and $t\ge m$, if $B_{u+1}=\istar$ for every integer $u\ge m$, then
\[
  \left\lvert
  (N_{t,\istar}-N_{m,\istar})-\beta(t-m)
  \right\rvert
  \le
  C_{\beta}\sqrt{(t-m+1)\log\left(\frac{etm^2}{\eta}\right)}
  +
  G_t^{\mathrm{tot}}(m),
\]
where $G_t^{\mathrm{tot}}(m):=\sum_{u=m}^{t-1}G_{u+1}$ is the number of guard
activations and any initialization pulls in $\{m+1,\dots,t\}$. Consequently,
\[
  \lvert \alpha_t-\beta\rvert
  \le
  \frac{(1+\beta)m}{t}
  +C_{\beta}\sqrt{\frac1t\log\left(\frac{etm^2}{\eta}\right)}
  +\frac{G_t^{\mathrm{tot}}(m)}{t}.
\]
\end{proposition}

\begin{proof}
We apply Azuma--Hoeffding to the centered sampling coins and take a union
bound over horizons and start times. Guarded and initialization pulls contribute
at most their number to the error.
Fix $\eta\in(0,1)$. For each $m\ge1$, set
$\eta_m:=6\eta/(\pi^2m^2)$ so that $\sum_{m\ge1}\eta_m\le\eta$.

For the analysis, draw an independent $U_{u+1}\sim\Ber(\beta)$ at every
round, including guarded and initialization rounds where the coin is unused.
Unused coins do not affect sampling decisions, so this construction leaves the
algorithm's law unchanged. Let $\cF_u$ be the
sigma-field generated by the sampled arms, observed rewards, and strategy
randomization through round $u$, including unused coins through that round.

Fix an integer $m\ge1$. For integers $u\ge m$, define
\[
  Z_{u+1}^{(m)}
  :=
  (U_{u+1}-\beta)\one\{G_{u+1}=0\}.
\]
Since $G_{u+1}$ is $\cF_{u}$-measurable and $U_{u+1}$ is independent of
$\cF_{u}$ with mean $\beta$, one has
\[
  \mathbb{E}_{\mu}\!\left[Z_{u+1}^{(m)}\mid\cF_{u}\right]=0,
  \qquad
  \lvert Z_{u+1}^{(m)}\rvert\le 1.
\]
Thus
\[
  S_{t}^{(m)}
  :=
  \sum_{u=m}^{t-1}Z_{u+1}^{(m)},
  \qquad t\ge m,
\]
is a martingale with bounded increments under $\mathbb{P}_{\mu}$.

For $n\ge1$, put $t=m+n$. Azuma--Hoeffding gives
\[
  \mathbb{P}_{\mu}
  \left(
    \lvert S_{m+n}^{(m)}\rvert
    \ge
    \sqrt{2n\log\!\left(\frac{12n^2}{\eta_m}\right)}
  \right)
  \le
  \frac{\eta_m}{6n^2}.
\]
A union bound over $n\ge1$ gives, with
probability at least $1-\eta_m$,
\[
  \lvert S_t^{(m)}\rvert
  \le
  C_{\beta}
  \sqrt{
    (t-m+1)\log\!\left(\frac{etm^2}{\eta}\right)
  }
  \qquad
  \forall t\ge m,
\]
after increasing the numerical constant $C_{\beta}$.

Suppose that $B_{u+1}=\istar$ for every integer $u\ge m$. On every unguarded
round $u+1$ in this range, the algorithm pulls $\istar$ if and only if
$U_{u+1}=1$. Initialization rounds have $G_{u+1}=1$ and therefore contribute
only to the second sum below. For every $t\ge m$,
\[
\begin{aligned}
  &(N_{t,\istar}-N_{m,\istar})-\beta(t-m)  \\
  &=
  \sum_{u=m}^{t-1}
  (U_{u+1}-\beta)\one\{G_{u+1}=0\}
  +
  \sum_{u=m}^{t-1}
  \left(\one\{I_{u+1}=\istar\}-\beta\right)
  \one\{G_{u+1}=1\}.
\end{aligned}
\]
The first sum is $S_t^{(m)}$. The second sum has absolute value at most
$G_t^{\mathrm{tot}}(m)$, since each guarded or initialization pull contributes
a term of absolute value at most one. Therefore,
\[
  \left\lvert
    (N_{t,\istar}-N_{m,\istar})-\beta(t-m)
  \right\rvert
  \le
  C_{\beta}
  \sqrt{
    (t-m+1)\log\!\left(\frac{etm^2}{\eta}\right)
  }
  +
  G_t^{\mathrm{tot}}(m).
\]
Intersecting the events over all $m\ge1$ and using
$\sum_{m\ge1}\eta_m\le\eta$ gives $\cA_{\eta}$. The fraction bound follows by
adding the initial discrepancy $N_{m,\istar}-\beta m$ and dividing by $t$.
\end{proof}

\subsection{Entrance into a Compact Neighborhood of \texorpdfstring{$\beta$}{beta}}
\label{app:compactU}

We use the fraction bound to control when $\alpha_t$ remains near $\beta$.
Let $t_{\mathrm{lead}}\ge K$ denote an integer time from which
$B_{u+1}=\istar$ for every integer $u\ge t_{\mathrm{lead}}$ on the continued
sampling process. This time requires only permanent leader correctness;
$t_U$ also requires $\alpha_t\in U$ thereafter.

\begin{lemma}[Entrance into a fixed compact neighborhood of $\beta$]
\label{lem:compactU}
Fix $\eta\in(0,1)$ and a compact interval
$U=[\beta-\rho,\beta+\rho]\Subset(0,1)$ with $\rho>0$.
There exists $C_{\beta}(U)<\infty$ such that the following holds. On the event of
Proposition~\ref{prop:alpha}, if $t_{\mathrm{lead}}\le L$ with $L\ge1$, then
\[
  \alpha_t\in U
  \qquad
  \forall t\ge
  C_{\beta}(U)
  \left(L+K+\log\frac{e(L+1)^2}{\eta}\right).
\]
When the guard is disabled, the additive $K$ term can be dropped.
\end{lemma}

\begin{proof}
We bound the initial discrepancy, coin fluctuations, and guarded-pull fraction
uniformly beyond a time proportional to $L$ and the confidence scale.
If the guard is enabled, Proposition~\ref{prop:alpha} yields for all
$t\ge t_{\mathrm{lead}}$,
\[
  \lvert \alpha_t-\beta\rvert
  \le
  \frac{(1+\beta)L}{t}
  +C_{\beta}\sqrt{\frac1t\log\left(\frac{et(L+1)^2}{\eta}\right)}
  +\frac{G_t^{\mathrm{tot}}(t_{\mathrm{lead}})}{t}.
\]
By Lemma~\ref{lem:guard-count}, the last term is at most
$\sqrt{K/t}+K/t$, and it is absent in the unguarded case. Let
\[
  A:=L+K+\log\frac{e(L+1)^2}{\eta}.
\]
For $t\ge CA$, the terms for the initial discrepancy and guarded pulls are each
at most $\rho/3$ once $C$ is large enough. The coin-fluctuation term is also at
most $\rho/3$: the map
$t\mapsto\log(et(L+1)^2/\eta)/t$ is decreasing above an absolute constant, and
at $t=CA$ its numerator is bounded by a constant multiple of $A$. Increasing
$C$ if necessary proves the claim.
\end{proof}

\subsection{Centered Quantile Integration}
\label{app:quantile-integration}

The next results convert upper-quantile bounds into expectation bounds by
integrating the remainder above a deterministic leading term. In this subsection,
$\mathbb{P}$ and $\mathbb{E}$ denote a generic probability law and its expectation.

\begin{theorem}[Centered quantile integration]
\label{thm:centered-quantile-integration}
Let $X\ge0$ be a random variable. Let $A\ge0$ and let
$\rho_X:(0,1]\to[0,\infty)$ satisfy
\[
  \mathbb{P}(X\le A+\rho_X(\eta))\ge1-\eta
  \qquad\forall\eta\in(0,1].
\]
Define the monotone envelope
\[
  \rho_X^{\downarrow}(\eta):=\sup_{\xi\in[\eta,1]}\rho_X(\xi).
\]
Then
\[
  \mathbb{E}[X]
  \le
  A+\int_0^1\rho_X^{\downarrow}(u)\,du.
\]
Moreover, if
\[
  \sum_{k\ge0}2^{-k}\rho_X^{\downarrow}(2^{-k})<\infty,
\]
then
\[
  \mathbb{E}[X]
  \le
  A+
  \sum_{k\ge0}2^{-k}\rho_X^{\downarrow}(2^{-k}).
\]
\end{theorem}

\begin{proof}
We integrate the upper-tail quantile, then bound the remainder integral on
dyadic intervals.
Let
\[
  q_X(\eta):=
  \inf\{x\ge0:\mathbb{P}(X\le x)\ge1-\eta\},
  \qquad \eta\in(0,1],
\]
be the upper-tail quantile of $X$. The assumed upper-quantile inequality gives
\[
  q_X(\eta)\le A+\rho_X(\eta)\le A+\rho_X^{\downarrow}(\eta)
  \qquad\forall\eta\in(0,1].
\]
The quantile representation of expectation for nonnegative random variables
therefore yields
\[
  \mathbb{E}[X]
  =
  \int_0^1q_X(u)\,du
  \le
  A+
  \int_0^1\rho_X^{\downarrow}(u)\,du.
\]
For the dyadic bound, $\rho_X^{\downarrow}$ is nonincreasing in $\eta$, so for
$u\in[2^{-k-1},2^{-k}]$,
\[
  \rho_X^{\downarrow}(u)\le \rho_X^{\downarrow}(2^{-k-1}).
\]
Consequently,
\[
  \int_0^1\rho_X^{\downarrow}(u)\,du
  \le
  \sum_{k\ge0}2^{-k-1}\rho_X^{\downarrow}(2^{-k-1})
  \le
  \sum_{k\ge0}2^{-k}\rho_X^{\downarrow}(2^{-k}).
\]
Combining this estimate with the first part proves the theorem.
\end{proof}

\begin{lemma}[Monotone summable upper-quantile bound]
\label{lemma:monotone-summable-localization}
Let $T\ge0$ be a random time admitting a summable upper-quantile bound $u_T$, namely
$
  \mathbb{P}(T\le u_T(\eta))\ge 1-\eta, \forall \eta\in(0,1],
$
and
$
  \sum_{k\ge0}2^{-k}u_T(2^{-k})<\infty.
$
Define the upper-tail quantile
\[
  q_T(\eta)
  :=
  \inf\{x\ge0:\mathbb{P}(T\le x)\ge1-\eta\},
  \qquad \eta\in(0,1].
\]
Then $q_T$ is a nonincreasing upper-quantile bound for $T$, and
$
  \sum_{k\ge0}2^{-k}q_T(2^{-k})<\infty.
$
\end{lemma}

\begin{proof}
We verify the quantile's confidence property and monotonicity, then compare
its dyadic series with that of $u_T$.
By definition of $q_T$, one has
\[
  \mathbb{P}(T\le q_T(\eta))\ge1-\eta
  \qquad
  \forall \eta\in(0,1].
\]
Thus $q_T$ is an upper-quantile bound. If $0<\eta_1\le\eta_2\le1$, then
$1-\eta_1\ge1-\eta_2$, so the confidence requirement at level $\eta_1$ is at
least as stringent as the one at level $\eta_2$. Therefore
\[
  q_T(\eta_1)\ge q_T(\eta_2),
\]
so $q_T$ is nonincreasing in $\eta$.

Since $u_T$ is also an upper-quantile bound, the definition of $q_T$ gives
\[
  q_T(\eta)\le u_T(\eta)
  \qquad
  \forall \eta\in(0,1].
\]
Hence
\[
  \sum_{k\ge0}2^{-k}q_T(2^{-k})
  \le
  \sum_{k\ge0}2^{-k}u_T(2^{-k})
  <\infty.
\]
\end{proof}

\begin{lemma}[Finite mean gives summable dyadic quantiles]
\label{lem:finite-mean-quantiles}
Let $X\ge0$ satisfy $\mathbb{E}[X]<\infty$, and let
\[
  q_X(\eta):=
  \inf\{x\ge0:\mathbb{P}(X\le x)\ge1-\eta\},
  \qquad \eta\in(0,1].
\]
Then
\[
  \sum_{k\ge0}2^{-k}q_X(2^{-k})<\infty.
\]
\end{lemma}

\begin{proof}
We use quantile monotonicity to bound each dyadic term by a quantile integral.
The map $\eta\mapsto q_X(\eta)$ is nonincreasing. For each $k\ge0$ and
$u\in[2^{-k-1},2^{-k}]$, one has $q_X(u)\ge q_X(2^{-k})$. Hence
\[
  2^{-k}q_X(2^{-k})
  \le
  2\int_{2^{-k-1}}^{2^{-k}}q_X(u)\,du.
\]
Summing over $k$ and using the quantile representation of expectation gives
\[
  \sum_{k\ge0}2^{-k}q_X(2^{-k})
  \le
  2\int_0^1q_X(u)\,du
  =2\mathbb{E}[X]<\infty.
\]
\end{proof}

\subsection{Threshold Inversion}
\label{app:threshold-inversion}

We give a deterministic threshold-crossing bound that is uniform in the
entrance delay $L$. Here $T$ is a positive scalar parameter, and time arguments
in $c$ and $\widetilde{H}$ are real.

\begin{theorem}[Threshold inversion with a linear entrance delay]
\label{thm:threshold-inversion-linear-delay}
Fix constants $T>0$, $C_0\ge0$, and $a,b,d\ge0$. For $\eta\in(0,1]$ and
$t\ge1$, define
\[
  \widetilde{H}_t(\eta):=\log\frac{K}{\eta}+\log(et).
\]
There exists a finite constant $A_{\mathrm{inv}}$, depending only on
$T,C_0,a,b,d$ and $K$, such that the following holds. For every
$\delta\in(0,1/2]$ and $\eta\in(0,1/4]$, set
\[
  L_{\delta}:=\log\frac1\delta,
  \qquad
  M_{\delta,\eta}:=
  \log\frac{K}{\eta}+\log\log\frac{e}{\delta},
\]
and
\[
  s_{\delta,\eta}
  :=
  \ceil{
    TL_{\delta}
    +A_{\mathrm{inv}}\sqrt{L_{\delta}M_{\delta,\eta}}
    +A_{\mathrm{inv}}M_{\delta,\eta}
    +A_{\mathrm{inv}}
  }.
\]
Then, for every $L\ge0$ and
\[
  t=(1+C_0)L+s_{\delta,\eta},
\]
one has
\[
  \frac{L+s_{\delta,\eta}}{T}
  -a\sqrt{t\widetilde{H}_t(\eta)}
  -b\widetilde{H}_t(\eta)
  -d
  >
  c(t,\delta).
\]
\end{theorem}

\begin{proof}
We choose $A_{\mathrm{inv}}$ so that the threshold residual is positive at
$L=0$ and increasing for $L\ge0$.
Write
\[
  \overline{c}:=1+C_0,
  \qquad
  s:=s_{\delta,\eta},
  \qquad
  M:=M_{\delta,\eta},
  \qquad
  L_{\delta}:=\log\frac1\delta.
\]
Because $\delta\in(0,1/2]$ and $\eta\in(0,1/4]$, there is a positive constant
$m_0=m_0(K)$ such that $M\ge m_0$.

We first record two elementary bounds. Put
\[
  R:=\sqrt{L_{\delta}M}+M+1.
\]
For every fixed $A\ge1$ and every
$1\le s\le TL_{\delta}+AR+2$, there exists a constant $C_1=C_1(T,K)$ such that
\[
  \widetilde{H}_s(\eta)
  \le
  C_1(1+\log A)M.
\]
By AM--GM, $R\le(L_{\delta}+3M)/2+1\le2(L_{\delta}+M+1)$.
Absorbing the factor two into $C_T$ gives
$\log(es)\le C_T+\log A+\log(L_{\delta}+M+2)$. Since
$\log(L_{\delta}+M+2)\le C_K+\log(1+L_{\delta})+\log(1+M)$ and
$M=\log(K/\eta)+\log(1+L_{\delta})$, the displayed bound follows after using
$M\ge m_0$. Consequently, for the value $A=A_{\mathrm{inv}}$ and the associated
$s$,
\begin{equation}
\label{eq:threshold-Hs-bound}
  \widetilde{H}_s(\eta)
  \le
  C_1(1+\log A_{\mathrm{inv}})M.
\end{equation}
Moreover, because $s\ge A_{\mathrm{inv}}M$,
\begin{equation}
\label{eq:threshold-Hs-over-s}
  \frac{\widetilde{H}_s(\eta)}{s}
  \le
  \frac{C_1(1+\log A_{\mathrm{inv}})}{A_{\mathrm{inv}}}.
\end{equation}

Define
\[
  \psi(L)
  :=
  \frac{L+s}{T}
  -a\sqrt{(\overline{c}L+s)\widetilde{H}_{\overline{c}L+s}(\eta)}
  -b\widetilde{H}_{\overline{c}L+s}(\eta)
  -d
  -c(\overline{c}L+s,\delta).
\]
For $t\ge1$,
\[
  \frac{d}{dt}\widetilde{H}_t(\eta)=\frac1t,
\]
and
\[
  \frac{d}{dt}\sqrt{t\widetilde{H}_t(\eta)}
  =
  \frac{\widetilde{H}_t(\eta)+1}{2\sqrt{t\widetilde{H}_t(\eta)}}
  \le
  \sqrt{\frac{\widetilde{H}_t(\eta)}{t}},
\]
because $\widetilde{H}_t(\eta)\ge1$. Also
\[
  \frac{d}{dt}c(t,\delta)=\frac{1}{1+t/2}\le\frac{2}{t}.
\]
Hence
\[
  \psi'(L)
  \ge
  \frac1T
  -a\overline{c}\sqrt{\frac{\widetilde{H}_{\overline{c}L+s}(\eta)}{\overline{c}L+s}}
  -(b+2)\frac{\overline{c}}{\overline{c}L+s}.
\]
The map $t\mapsto\widetilde{H}_t(\eta)/t$ is decreasing for $t\ge1$ because
$\widetilde{H}_t(\eta)\ge1$. Using \eqref{eq:threshold-Hs-over-s}, we can choose
$A_{\mathrm{inv}}$ large enough that
\[
  a\overline{c}\sqrt{\frac{\widetilde{H}_s(\eta)}{s}}
  +(b+2)\frac{\overline{c}}{s}
  \le
  \frac{1}{2T},
\]
uniformly in $\delta$ and $\eta$. For such a choice, $\psi'(L)\ge1/(2T)>0$ for all
$L\ge0$.

It remains to make $\psi(0)>0$. Since
\[
  c(s,\delta)
  \le
  L_{\delta}+2\widetilde{H}_s(\eta)+C_2(K),
\]
one has
\[
  \psi(0)
  \ge
  \frac{s}{T}-L_{\delta}
  -a\sqrt{s\widetilde{H}_s(\eta)}
  -(b+2)\widetilde{H}_s(\eta)
  -d-C_2(K).
\]
The definition of $s$ gives
\[
  \frac{s}{T}-L_{\delta}
  \ge
  \frac{A_{\mathrm{inv}}}{T}R.
\]
Using \eqref{eq:threshold-Hs-bound} and $s\le TL_{\delta}+A_{\mathrm{inv}}R+2$,
we get
\[
  s\widetilde{H}_s(\eta)
  \le
  C_3(1+\log A_{\mathrm{inv}})M
  \left(TL_{\delta}+A_{\mathrm{inv}}R+2\right).
\]
Since $L_{\delta}M\le R^2$, $RM\le R^2$, and $M\le R^2$, this implies
\[
  \sqrt{s\widetilde{H}_s(\eta)}
  \le
  C_4\sqrt{A_{\mathrm{inv}}(1+\log A_{\mathrm{inv}})}\,R.
\]
Also, \eqref{eq:threshold-Hs-bound} gives
$\widetilde{H}_s(\eta)\le C_5(1+\log A_{\mathrm{inv}})R$. Therefore
\[
  \psi(0)
  \ge
  \left[
    \frac{A_{\mathrm{inv}}}{T}
    -aC_4\sqrt{A_{\mathrm{inv}}(1+\log A_{\mathrm{inv}})}
    -(b+2)C_5(1+\log A_{\mathrm{inv}})
    -d-C_2(K)
  \right]R.
\]
Because
\[
  \frac{A}{\sqrt{A(1+\log A)}}\to\infty
  \quad\text{and}\quad
  \frac{A}{1+\log A}\to\infty
  \qquad\text{as }A\to\infty,
\]
we can increase $A_{\mathrm{inv}}$ so that the bracket is strictly positive.
The residual therefore satisfies $\psi(L)\ge\psi(0)>0$ for every $L\ge0$,
which proves the desired inequality.
\end{proof}

\subsection{Proof of Theorem~\ref{thm:conditional}}
\label{app:proof-conditional}

The preceding estimates yield the high-probability and expected stopping-time
bounds in the stabilized regime.

\begin{proof}[Proof of Theorem~\ref{thm:conditional}]
We combine the evidence envelope and fraction bound to obtain a threshold
crossing, then integrate the centered stopping-time quantiles.
Fix $\delta\in(0,1/2]$ and $\eta\in(0,1/4]$, and use
$\widetilde{H}_t(\eta)=\log(K/\eta)+\log(et)$ as above.

Let $\cE_{\mathrm{env}}(\eta)$ be the event from
Proposition~\ref{prop:envelope}, with probability at least $1-2\eta$, and let
$\cA_{\eta}$ be the event from Proposition~\ref{prop:alpha}, with
probability at least $1-\eta$. We work on
\[
  \cE:=\cE_{\mathrm{env}}(\eta)\cap\cA_{\eta},
  \qquad
  \mathbb{P}_{\mu}(\cE)\ge1-3\eta.
\]

For all $t\ge t_U$, entry into the stabilized regime gives $B_{t+1}=\istar$ and
$\alpha_t\in U$. Hence the stopping condition is met at time $t$ as soon as
\[
  \min_{j\neq \istar}W_t(\istar,j)>c(t,\delta).
\]
On $\cE_{\mathrm{env}}(\eta)$, Proposition~\ref{prop:envelope} gives,
for all $t\ge t_U$,
\[
  \min_{j\neq \istar}W_t(\istar,j)
  \ge
  \frac{t}{\Tstar_{\alpha_t}(\mu)}
  -B_{\mu,0}(U)t_U
  -B_{\mu,1}(U)\sqrt{t\widetilde{H}_t(\eta)}
  -B_{\mu,2}(U)\widetilde{H}_t(\eta)
  -B_{\mu,3}(U),
\]
for positive finite constants $B_{\mu,\ell}(U)$, $\ell\in\{0,1,2,3\}$.

By Lemma~\ref{lem:T-inv-lipschitz}, for all $\alpha\in U$,
\[
  \frac1{\Tstar_{\alpha}(\mu)}
  \ge
  \frac1{\Tstar_{\beta}(\mu)}-L_{\mu}(U)\lvert \alpha-\beta\rvert.
\]
On $\cA_{\eta}$, Proposition~\ref{prop:alpha} with start time $m=t_U$
implies, uniformly for $t\ge t_U$,
\[
  t\lvert \alpha_t-\beta\rvert
  \le
  C_{\beta}t_U+C_{\beta}\sqrt{t\widetilde{H}_t(\eta)}+G_t^{\mathrm{tot}}(t_U).
\]
The proof of Proposition~\ref{prop:alpha} permits a numerical $C_{\beta}$
uniform in $\beta$; the initial discrepancy also uses only $1+\beta\le2$.
For the unguarded rule the last term is zero. With the guard enabled,
Lemma~\ref{lem:guard-count} gives
$G_t^{\mathrm{tot}}(t_U)\le\sqrt{Kt}+K\le C_K\sqrt{t\widetilde{H}_t(\eta)}$.
Thus
\[
  \frac{t}{\Tstar_{\alpha_t}(\mu)}
  \ge
  \frac{t}{\Tstar_{\beta}(\mu)}
  -A_{\mu,0}(U)t_U-A_{\mu,1}(U)\sqrt{t\widetilde{H}_t(\eta)}.
\]
Combining the envelope and fraction bounds gives constants
$a_{\mu}(U),b_{\mu}(U),d_{\mu}(U)$ and $C_{\mu,0}(U)\ge0$ such that, for all
$t\ge t_U$,
\begin{equation}
\label{eq:conditional-lower-final}
  \min_{j\neq \istar}W_t(\istar,j)
  \ge
  \frac{t}{\Tstar_{\beta}(\mu)}
  -\frac{C_{\mu,0}(U)}{\Tstar_{\beta}(\mu)}t_U
  -a_{\mu}(U)\sqrt{t\widetilde{H}_t(\eta)}
  -b_{\mu}(U)\widetilde{H}_t(\eta)
  -d_{\mu}(U).
\end{equation}

To choose these coefficients uniformly over $\beta\in U$, set
$T_{\max}:=\max_{\alpha\in U}\Tstar_{\alpha}(\mu)<\infty$ and take
$C_{\mu,0}(U)=T_{\max}(B_{\mu,0}(U)+A_{\mu,0}(U))$.
The remaining coefficients already depend only on $(\mu,U)$.
In the inversion proof, use $T_{\max}$ in the upper bound on $s$ and
$1/T\ge1/T_{\max}$ in the derivative and base-value bounds. A single
$A_{\mathrm{inv}}$ then works for all $T=\Tstar_{\beta}(\mu)$ with $\beta\in U$.

Apply Theorem~\ref{thm:threshold-inversion-linear-delay} with
\[
  T=\Tstar_{\beta}(\mu),
  \qquad
  C_0=C_{\mu,0}(U),
  \qquad
  a=a_{\mu}(U),
  \qquad
  b=b_{\mu}(U),
  \qquad
  d=d_{\mu}(U).
\]
Let $s_{\delta,\eta}$ be the time given by this theorem and set
\[
  t^{\star}:=\ceil{(1+C_{\mu,0}(U))t_U+s_{\delta,\eta}}.
\]
Evaluating \eqref{eq:conditional-lower-final} at $t=t^{\star}$ gives
\[
  \min_{j\neq \istar}W_{t^{\star}}(\istar,j)
  \ge
  \frac{t^{\star}-C_{\mu,0}(U)t_U}{\Tstar_{\beta}(\mu)}
  -a_{\mu}(U)\sqrt{t^{\star}\widetilde{H}_{t^{\star}}(\eta)}
  -b_{\mu}(U)\widetilde{H}_{t^{\star}}(\eta)
  -d_{\mu}(U).
\]
To apply Theorem~\ref{thm:threshold-inversion-linear-delay} at this integer
time, take $L=(t^{\star}-s_{\delta,\eta})/(1+C_{\mu,0}(U))\ge t_U$.
Then $t^{\star}=(1+C_{\mu,0}(U))L+s_{\delta,\eta}$ and
$t^{\star}-C_{\mu,0}(U)t_U\ge L+s_{\delta,\eta}$. The theorem therefore
makes the right-hand side strictly larger than $c(t^{\star},\delta)$. Hence
$\tau_{\delta}\le t^{\star}$ on $\cE$, and the stated high-probability
bound follows by expanding $s_{\delta,\eta}$ and absorbing both ceiling errors,
each less than one, into the bounded remainder.

Now assume that $t_U$ admits a summable upper-quantile bound $u_U$.
The expected-bound coefficients may also depend on $u_U$. Let
\[
  L_{\delta}:=\log\frac1\delta,
  \qquad
  A_{\delta}:=\Tstar_{\beta}(\mu)L_{\delta}.
\]
The high-probability bound gives constants $D_{\mu,k}(U)<\infty$,
$k\in\{0,1,2,3\}$, such that, for every $\zeta\in(0,1]$,
\[
  \mathbb{P}_{\mu}
  \left(
    \tau_{\delta}\le A_{\delta}+R_{\delta}(\zeta)
  \right)
  \ge1-\zeta,
\]
where
\begin{align*}
  R_{\delta}(\zeta)
  &=
  D_{\mu,0}(U)u_U(\zeta/4)
  +D_{\mu,1}(U)
  \sqrt{L_{\delta}
  \left(\log\frac{4K}{\zeta}+\log\log\frac{e}{\delta}\right)} \\
  &\quad
  +D_{\mu,2}(U)
  \left(\log\frac{4K}{\zeta}+\log\log\frac{e}{\delta}\right)
  +D_{\mu,3}(U).
\end{align*}
This follows by applying the high-probability bound with $\eta=\zeta/4$ and
intersecting it with the event $\{t_U\le u_U(\zeta/4)\}$.

Let $R_{\delta}^{\downarrow}$ be the monotone envelope of $R_{\delta}$.
Theorem~\ref{thm:centered-quantile-integration} gives
\[
  \mathbb{E}_{\mu}[\tau_{\delta}]
  \le
  A_{\delta}
  +
  \sum_{k\ge0}2^{-k}R_{\delta}^{\downarrow}(2^{-k}).
\]
By Lemma~\ref{lemma:monotone-summable-localization}, we may take $u_U$ to be
the exact upper-tail quantile of $t_U$. Then $R_{\delta}$ is nonincreasing, so
$R_{\delta}^{\downarrow}=R_{\delta}$. The shifted dyadic series also satisfies
\[
  \sum_{k\ge0}2^{-k}u_U(2^{-k-2})<\infty.
\]
In the remaining estimates, $C$ denotes a finite constant independent of
$\delta$ and may change between inequalities.
The term $C\sqrt{L_{\delta}}$ is dominated by
$C\sqrt{L_{\delta}\log\log(e/\delta)}$ for $\delta\in(0,1/2]$. Similarly,
\[
  \sum_{k\ge0}2^{-k}
  \left(
    \log(4K)+k\log2+
    \log\log\frac{e}{\delta}
  \right)
  \le
  C\log\log\frac{e}{\delta}+C.
\]
Combining these bounds yields
\[
\begin{aligned}
  \mathbb{E}_{\mu}[\tau_{\delta}]
  &\le
  \Tstar_{\beta}(\mu)\log(1/\delta)
  +C'_{\mu,1}(U)\sqrt{\log(1/\delta)\log\log(e/\delta)} \\
  &\quad +C'_{\mu,2}(U)\log\log(e/\delta)
  +C'_{\mu,3}(U).
\end{aligned}
\]
\end{proof}

\section{Guarantees Under a Unique Best Arm}
\label{app:uniquebest}

We bound the time to enter the stabilized regime under a unique best arm, then
derive the high-probability and expected stopping-time guarantees. All algorithmic
time indices refer to the continued sampling process in
Section~\ref{sec:stable-entry}.
Throughout, $t_U$ is the earliest $U$-stabilization time defined in
Section~\ref{sec:unconditional-closures}, with value $\infty$ if none exists.
The countable persistence conditions make this time measurable and independent
of the confidence level. Each horizon below bounds this same $t_U$.
For each compact interval $U\Subset(0,1)$ with $\beta$ in its interior, we
apply Lemma~\ref{lem:compactU} to a symmetric compact subinterval centered at
$\beta$. Stabilization in this subinterval implies stabilization in $U$.

\subsection{Separation Implies Permanent Leader Correctness}
\label{app:permanent-proof}

The separation counts ensure that the empirical best-arm mean remains strictly
above every suboptimal-arm mean.

\begin{proof}[Proof of Proposition~\ref{prop:permanent}]
We compare the armwise confidence radii with one quarter of each best-arm gap.
Fix $\eta\in(0,1)$. The real extension of $r(n,\eta)^2$ is decreasing for
$n\ge1$: apart from the positive factor $c_0/(2n^2)$, its derivative is
$1/\log(en)-\log(8K^2/\eta)-\log\log(en)<0$, since $K\ge2$.
Thus the radius bound at $n_j(\eta)$ holds at every larger count.
Fix an integer $t\ge t_{\mathrm{sep}}(\eta)$ and an arm
$j\in[K]\setminus\{\istar\}$. By definition,
$N_{t,\istar}\ge n_{\star}(\eta)\ge n_j(\eta)$ and
$N_{t,j}\ge n_j(\eta)$. Hence
\[
  r(N_{t,\istar},\eta)
  \le
  \frac{\mu_{\istar}-\mu_j}{4},
  \qquad
  r(N_{t,j},\eta)
  \le
  \frac{\mu_{\istar}-\mu_j}{4}.
\]
On $\cK^{\mathrm{sep}}_{\eta}$, these radius bounds yield
\[
  \widehat{\mu}_{t,\istar}
  \ge
  \mu_{\istar}-\frac{\mu_{\istar}-\mu_j}{4},
  \qquad
  \widehat{\mu}_{t,j}
  \le
  \mu_j+\frac{\mu_{\istar}-\mu_j}{4}.
\]
The bounds imply $\widehat{\mu}_{t,\istar}>\widehat{\mu}_{t,j}$. Since
$t\ge t_{\mathrm{sep}}(\eta)$ and $j\neq\istar$ were arbitrary, the
empirical leader is correct at every time from $t_{\mathrm{sep}}(\eta)$ onward.
\end{proof}

\subsection{High-Probability Stabilization Bound for the Unguarded Algorithm}
\label{app:poly-localization}

We bound the penalized index of any nonleader whose count is below its
separation threshold. This bound ensures progress toward separation; a
martingale argument then bounds the time needed to reach the separation counts.

\begin{lemma}[Polynomial-horizon index bound for separation-deficient nonleaders]
\label{lem:deficient}
Let $p:=\min\{\beta,1-\beta\}$.
There exist positive finite constants $\widetilde{A}_{\mu},\widetilde{B}_{\mu}$
such that the following holds. Fix exponents
$a>\widetilde{A}_{\mu}$ and $b>\widetilde{B}_{\mu}+1$. Then there exists a constant
$C_{\mu,\beta,a,b}<\infty$ with the property that, for every $\eta\in(0,1/4]$, if
\[
  M_{\eta}=\ceil{C_{\mu,\beta,a,b}\eta^{-a}(\log(e/\eta))^b},
  \qquad
  T_{\eta}=K+
  \ceil{\frac{2K}{p}(M_{\eta}-1)+\frac{16}{p^2}\log\frac6\eta},
\]
then on the event $\cG_{\eta/6}\cap\cK_{\eta/6}$ the following implication holds
for every integer $K\le t\le T_{\eta}$. If $j\in[K]\setminus\{B_{t+1}\}$ and
\[
  N_{t,j}<
  \begin{cases}
    n_{\star}(\eta/6), & j=\istar,\\
    n_j(\eta/6), & j\neq \istar,
  \end{cases},
\]
then
\[
  W_t(B_{t+1},j)+\log N_{t,j}\le\log(M_{\eta}-1).
\]
Consequently, whenever such a separation-deficient nonleader exists, the
challenger selected by the unguarded algorithm has count strictly smaller than
$M_{\eta}$.
\end{lemma}

\begin{proof}
We bound the population GLR by the deficient arm's count, then apply the
uniform GLR comparison and logarithmic separation-count bound.
Let $j\neq B_{t+1}$ satisfy the displayed separation-deficiency condition. Then
$N_{t,j}\le n_{\star}(\eta/6)$. Since $B_{t+1}$ is an empirical leader,
$\widehat{\mu}_{t,B_{t+1}}\ge\widehat{\mu}_{t,j}$, and therefore
$W_t(B_{t+1},j)=V_t(B_{t+1},j)$. Write
\[
  D_{\mu,\mathrm{pair}}^{\max}:=\max_{a,b\in[K]:a\neq b}\KL{\mu_a}{\mu_b}<\infty.
\]
Evaluating the population symmetric GLR at $x=\mu_{B_{t+1}}$ gives
\[
  V_t^{\mu}(B_{t+1},j)
  \le
  D_{\mu,\mathrm{pair}}^{\max}n_{\star}(\eta/6).
\]
On $\cG_{\eta/6}$, Theorem~\ref{thm:symmetric-glr-control} and
$N_{\mathrm{eff},t}(B_{t+1},j)\le N_{t,j}\le n_{\star}(\eta/6)$ imply
\[
  W_t(B_{t+1},j)
  \le
  D_{\mu,\mathrm{pair}}^{\max}n_{\star}(\eta/6)
  +C_{\mu}
  \left(
    \sqrt{n_{\star}(\eta/6)H_{T_{\eta}}(\eta/6)}
    +H_{T_{\eta}}(\eta/6)
  \right).
\]
We track the dependence on the cap constant before choosing it.
Put $x:=\log(e/\eta)$. Lemma~\ref{lem:logsep} gives a constant
$S_{\mu}\ge1$, independent of $a,b,p$ and the cap constant, such that
$n_{\star}(\eta/6)\le S_{\mu}x$ for all $\eta\in(0,1/4]$.
Write $c:=\max\{1,C_{\mu}\}$ and $d:=D_{\mu,\mathrm{pair}}^{\max}$, and choose
\[
  \widetilde{A}_{\mu}:=(d+c/2)S_{\mu}+3c/2,
  \qquad
  \widetilde{B}_{\mu}:=3c/2+1.
\]
Fix $a>\widetilde{A}_{\mu}$ and $b>\widetilde{B}_{\mu}+1$.
For a candidate $C:=C_{\mu,\beta,a,b}\ge2S_{\mu}$, let
$F_{\eta}:=\eta^{-a}x^b$. Since $a>0$ and $b>1$, $F_{\eta}\ge x\ge1$.
The ceiling bounds $M_{\eta}-1\le CF_{\eta}$ and
$\ceil{u}\le u+1$, together with
$\log(6/\eta)\le(\log6)x$, give
\[
  T_{\eta}\le D_{K,p}CF_{\eta},
  \qquad
  D_{K,p}:=K+1+\frac{2K}{p}+\frac{16\log6}{p^2}.
\]
Define $\Lambda_C:=1+\log D_{K,p}+\log C+a+b$. Using
$\log(1/\eta)\le x$ and $\log x\le x$, we obtain
\[
  \log(eT_{\eta})\le\Lambda_C x,
  \qquad
  H_{T_{\eta}}(\eta/6)
  \le\log(1/\eta)+\log x+\log(6K)+\log\Lambda_C.
\]
Applying $\sqrt{nH}\le(n+H)/2$ to the preceding GLR bound and using
$\log N_{t,j}\le\log(S_{\mu}x)$ now yields
\[
  W_t(B_{t+1},j)+\log N_{t,j}
  \le
  \widetilde{A}_{\mu}\log(1/\eta)
  +\widetilde{B}_{\mu}\log x
  +R_{\mu}+\frac{3c}{2}\log\Lambda_C,
\]
where $R_{\mu}:=(d+c/2)S_{\mu}+(3c/2)\log(6K)+\log S_{\mu}$.
For the fixed $a,b,p$, $\log\Lambda_C=o(\log C)$ as $C\to\infty$.
We can therefore choose $C\ge2S_{\mu}$ so that
\[
  \log C\ge R_{\mu}+\log2+\frac{3c}{2}\log\Lambda_C.
\]
This choice is independent of $\eta$. Since $CF_{\eta}\ge2$, the lower
ceiling bound gives $M_{\eta}-1\ge CF_{\eta}-1\ge CF_{\eta}/2$, and hence
\[
  \log(M_{\eta}-1)
  \ge\log C-\log2+a\log(1/\eta)+b\log x.
\]
The last two displays and the choices of $a,b$ prove the required index
bound. They also give
$M_{\eta}-1\ge CF_{\eta}/2\ge S_{\mu}x\ge n_{\star}(\eta/6)$ uniformly in
$\eta$.

For the challenger-count conclusion, any arm $k$ with $N_{t,k}\ge M_{\eta}$
has penalized index at least $\log M_{\eta}$, since transportation costs are
nonnegative. A separation-deficient nonleader has index at most
$\log(M_{\eta}-1)$, so a
minimizer of the challenger index cannot have count at least $M_{\eta}$.
\end{proof}

\begin{proof}[Proof of Proposition~\ref{prop:poly}]
We bound the separation time using a lower bound on the conditional probability
of pulls below the count cap and Azuma--Hoeffding, then control the best-arm
fraction.
Fix $\eta\in(0,1/4]$ and write $\overline{\eta}:=\eta/6$. Let
$M_{\eta},T_{\eta}$ be the quantities from Lemma~\ref{lem:deficient}, and set
$p:=\min\{\beta,1-\beta\}$. For integers $t\ge K$, define the progress
indicator and total count deficit by
\[
  X_{t+1}:=\one\{N_{t,I_{t+1}}<M_{\eta}\},
  \qquad
  \mathcal{D}_t:=\sum_{i=1}^{K}(M_{\eta}-N_{t,i})_{+},
  \qquad t\ge K.
\]
Since each arm is pulled once initially, $\mathcal{D}_K=K(M_{\eta}-1)$ and
$\mathcal{D}_{t+1}=\mathcal{D}_t-X_{t+1}$. Let
$\tau_M:=\inf\{t\ge K:\mathcal{D}_t=0\}$. At time $\tau_M$, all arms have at
least $M_{\eta}$ pulls. By Lemma~\ref{lem:deficient},
$M_{\eta}\ge n_{\star}(\overline{\eta})$, so
$t_{\mathrm{sep}}(\overline{\eta})\le\tau_M$.

Define the good event
\[
  \cE_{\overline{\eta}}
  :=
  \cG_{\overline{\eta}}
  \cap
  \cK_{\overline{\eta}}.
\]
We use the filtration $\cF_s$ defined in Appendix~\ref{app:alpha}.
All martingale statements below are under $\mathbb{P}_{\mu}$, without
conditioning on the good event. For $s=K,\dots,T_{\eta}-1$, define the
martingale differences
\[
  Y_{s+1}:=X_{s+1}-\mathbb{E}_{\mu}[X_{s+1}\mid\cF_s].
\]
Then $\mathbb{E}_{\mu}[Y_{s+1}\mid\cF_s]=0$ and $\lvert Y_{s+1}\rvert\le1$.

We first show that the conditional probability of progress is bounded below
on the good event. Fix
an integer $K\le s<T_{\eta}\wedge\tau_M\wedge t_{\mathrm{sep}}(\overline{\eta})$
and suppose that $\cE_{\overline{\eta}}$ holds. Since
$s<t_{\mathrm{sep}}(\overline{\eta})$, at least one arm is separation-deficient.
If some separation-deficient arm is not the current leader $B_{s+1}$, then
Lemma~\ref{lem:deficient} implies that the challenger selected by the
unguarded algorithm has count $N_{s,C_{s+1}}<M_{\eta}$. Thus $X_{s+1}=1$ whenever
the challenger is pulled. If no separation-deficient nonleader exists, the
current leader is the only separation-deficient arm. In that case
$N_{s,B_{s+1}}<n_{\star}(\overline{\eta})\le M_{\eta}$, so $X_{s+1}=1$ whenever
the leader is pulled. In both cases,
\begin{equation}
\label{eq:poly-pathwise-drift}
  \mathbb{E}_{\mu}[X_{s+1}\mid\cF_s]
  \ge p
  \quad\text{on }\cE_{\overline{\eta}}\cap
  \{s<T_{\eta}\wedge\tau_M\wedge t_{\mathrm{sep}}(\overline{\eta})\}.
\end{equation}

On $\cE_{\overline{\eta}}\cap\{t_{\mathrm{sep}}(\overline{\eta})>T_{\eta}\}$,
one also has $\tau_M>T_{\eta}$, because $\tau_M\le T_{\eta}$ would imply
$t_{\mathrm{sep}}(\overline{\eta})\le\tau_M\le T_{\eta}$. Therefore
\eqref{eq:poly-pathwise-drift} applies for every $s=K,\dots,T_{\eta}-1$.
Moreover, $\mathcal{D}_{T_{\eta}}>0$, so
\[
  \sum_{s=K}^{T_{\eta}-1}X_{s+1}
  =
  \mathcal{D}_K-\mathcal{D}_{T_{\eta}}
  <\mathcal{D}_K
  =K(M_{\eta}-1).
\]
Consequently,
\begin{equation}
\label{eq:poly-event-inclusion}
  \cE_{\overline{\eta}}
  \cap\{t_{\mathrm{sep}}(\overline{\eta})>T_{\eta}\}
  \subseteq
  \left\{
    \sum_{s=K}^{T_{\eta}-1}Y_{s+1}\le-a_{\eta}
  \right\},
  \qquad
  a_{\eta}:=p(T_{\eta}-K)-K(M_{\eta}-1).
\end{equation}
By the definition of $T_{\eta}$,
\[
  T_{\eta}-K
  \ge
  \frac{2K}{p}(M_{\eta}-1)+\frac{16}{p^2}\log\frac6\eta,
\]
so $a_{\eta}\ge p(T_{\eta}-K)/2$. Azuma--Hoeffding gives
\[
  \mathbb{P}_{\mu}
  \left(\sum_{s=K}^{T_{\eta}-1}Y_{s+1}\le-a_{\eta}\right)
  \le
  \exp\left(-\frac{a_{\eta}^2}{2(T_{\eta}-K)}\right)
  \le
  \exp\left(-\frac{p^2(T_{\eta}-K)}8\right)
  \le\eta/6.
\]
Since $\mathbb{P}_{\mu}(\cE_{\overline{\eta}}^c)\le2\overline{\eta}=\eta/3$,
\eqref{eq:poly-event-inclusion} yields
\[
  \mathbb{P}_{\mu}
  \left(t_{\mathrm{sep}}(\eta/6)>T_{\eta}\right)
  \le
  \eta/2.
\]
The polynomial form of $T_{\eta}$ in Lemma~\ref{lem:deficient} gives the
claimed separation-time bound.

To bound the stabilization time, set $\overline{\eta}:=\eta/12$ and define the
rescaled horizon $\widehat{T}_{\eta}:=T_{6\overline{\eta}}=T_{\eta/2}$.
Applying the preceding separation-time bound at confidence level $\eta/2$ gives
\[
  \mathbb{P}_{\mu}
  \left(t_{\mathrm{sep}}(\overline{\eta})\le\widehat{T}_{\eta}\right)
  \ge1-\eta/4,
\]
because $(\eta/2)/6=\overline{\eta}$. Intersect this event with the
KL confidence event $\cK_{\overline{\eta}}$, which has probability at
least $1-\eta/12$ and is contained in $\cK^{\mathrm{sep}}_{\overline{\eta}}$
by Appendix~\ref{app:armwise-cs}, and with the leader-fraction event
$\cA_{\eta/4}$ of Proposition~\ref{prop:alpha}, which has probability at least
$1-\eta/4$. On
the intersection, Proposition~\ref{prop:permanent} implies
\[
  B_{t+1}=\istar
  \qquad\forall t\ge\widehat{T}_{\eta}.
\]
Lemma~\ref{lem:compactU}, applied at confidence level $\eta/4$, bounds the
earliest stabilization time by
\[
  t_U
  \le
  C_{\beta}(U)
  \left(
    \widehat{T}_{\eta}+K+
    \log\frac{e(\widehat{T}_{\eta}+1)^2}{\eta}
  \right).
\]
Here $C_{\beta}(U)$ absorbs the factor $4$ inside the logarithm and upward
rounding of the resulting horizon to an integer.
The failure probability of the intersection is at most
\[
  \eta/4+\eta/12+\eta/4<\eta.
\]
Since $\widehat{T}_{\eta}$ has the same polynomial form in $1/\eta$ after
changing constants, the claimed upper-quantile bound follows. Letting
$\eta\downarrow0$ shows that this $t_U$ is almost surely finite.
\end{proof}

\begin{proof}[Proof of Corollary~\ref{cor:unguarded-hp}]
Combine Proposition~\ref{prop:poly} with Theorem~\ref{thm:conditional}. The
upper-quantile bound for $t_U$ is polynomial in $1/\eta$, which yields the
claimed high-probability stopping bound. This upper-quantile bound need not
be summable, so it does not by itself yield an expected stopping-time bound.
\end{proof}

\subsection{Finite-Mean Sufficient Exploration}
\label{app:suff-exp}

Finite-mean sufficient exploration yields a summable upper-quantile bound for
the stabilization time, allowing us to apply the conditional expectation bound.

\begin{proof}[Proof of Theorem~\ref{thm:unguarded-suffexp}]
We construct an upper-quantile bound for stabilization from quantiles of
$T_{\mathrm{se}}$ and verify dyadic summability before applying
Theorem~\ref{thm:conditional}.
Fix the compact interval $U\Subset(0,1)$ from the theorem, with $\beta$ in
its interior.

Let
\[
  q_{\mathrm{se}}(\xi):=
  \inf\{x\ge0:\mathbb{P}_{\mu}(T_{\mathrm{se}}\le x)\ge1-\xi\},
  \qquad \xi\in(0,1],
\]
be the upper-tail quantile of $T_{\mathrm{se}}$. By
Lemma~\ref{lem:finite-mean-quantiles},
\[
  \sum_{k\ge0}2^{-k}q_{\mathrm{se}}(2^{-k})<\infty.
\]
Recall the separation scale from Appendix~\ref{app:challenger-growth}: for
$\xi\in(0,1/2]$,
\[
  R_{\xi}:=\log\frac{K}{\xi}+\log\log\frac{e}{\xi}+1.
\]
Lemma~\ref{lem:logsep} gives $n_{\star}(\xi)\le A_{\mu}R_{\xi}$. For
$\eta\in(0,1/4]$, define
\[
  L_{\mathrm{se}}(\eta)
  :=
  q_{\mathrm{se}}(\eta/3)
  +
  \ceil{\frac{K}{c_{\mathrm{se}}^2}n_{\star}(\eta/3)^2}+1.
\]
The $+1$ slack in this horizon ensures
\[
  \floor{L_{\mathrm{se}}(\eta)}\ge q_{\mathrm{se}}(\eta/3),
  \qquad
  \floor{L_{\mathrm{se}}(\eta)}
  \ge\ceil{\frac{K}{c_{\mathrm{se}}^2}n_{\star}(\eta/3)^2}.
\]
On the event $\{T_{\mathrm{se}}\le q_{\mathrm{se}}(\eta/3)\}$, every
integer $t\ge\floor{L_{\mathrm{se}}(\eta)}$ satisfies $t\ge T_{\mathrm{se}}$ and
\[
  \min_iN_{t,i}
  \ge
  c_{\mathrm{se}}\sqrt{t/K}
  \ge
  n_{\star}(\eta/3).
\]
Hence $t_{\mathrm{sep}}(\eta/3)\le\floor{L_{\mathrm{se}}(\eta)}
\le L_{\mathrm{se}}(\eta)$. Intersecting this
event with the KL confidence event $\cK_{\eta/3}$, which is contained in
$\cK^{\mathrm{sep}}_{\eta/3}$ by Appendix~\ref{app:armwise-cs},
Proposition~\ref{prop:permanent} implies
\[
  B_{t+1}=\istar
  \qquad\forall t\ge\floor{L_{\mathrm{se}}(\eta)}.
\]
Now also intersect with the event $\cA_{\eta/3}$ of
Proposition~\ref{prop:alpha}. Since the rule is unguarded, the guard term in
Lemma~\ref{lem:compactU} is absent. Applying the lemma with the integer
leader time $\floor{L_{\mathrm{se}}(\eta)}$ and its upper bound
$L_{\mathrm{se}}(\eta)$ gives
\[
  t_U
  \le
  u_U(\eta)
  :=
  C_{\beta}(U)
  \left(
    L_{\mathrm{se}}(\eta)+
    \log\frac{e(L_{\mathrm{se}}(\eta)+1)^2}{\eta/3}
  \right).
\]
The intersection of the three events has probability at least $1-\eta$. Thus
$u_U$ is an upper-quantile bound on $(0,1/4]$; extend it to $(0,1]$ by setting
$u_U(\eta)=u_U(1/4)$ for $\eta>1/4$.
As above, upward rounding is absorbed into $C_{\beta}(U)$, and letting
$\eta\downarrow0$ shows that $t_U$ is almost surely finite.

It remains to check dyadic summability. Since
$n_{\star}(\eta)\le A_{\mu}R_{\eta}$,
\[
  L_{\mathrm{se}}(\eta)
  \le
  q_{\mathrm{se}}(\eta/3)+C_{\mu,c_{\mathrm{se}}}R_{\eta/3}^2.
\]
For dyadic levels $\eta_k=2^{-k}\le1/4$, monotonicity of $q_{\mathrm{se}}$ gives
$q_{\mathrm{se}}(2^{-k}/3)\le q_{\mathrm{se}}(2^{-k-2})$, and therefore
\[
  \sum_{k\ge0}2^{-k}q_{\mathrm{se}}(2^{-k}/3)<\infty.
\]
The dyadic sums of $R_{2^{-k}/3}^2$ and $\log(e/2^{-k})$ are finite.
The finitely many larger error levels do not affect summability. Finally,
\[
  \log\frac{e(L_{\mathrm{se}}(\eta)+1)^2}{\eta/3}
  \le
  C\left(L_{\mathrm{se}}(\eta)+\log(e/\eta)+1\right),
\]
for a finite constant $C$ independent of $\eta$, using $\log(1+x)\le x$. Hence
\[
  \sum_{k\ge0}2^{-k}u_U(2^{-k})<\infty.
\]
For the stopping-bound coefficients, repeat this construction with the larger
compact interval $U^+:=[(\inf U)/2,(1+\sup U)/2]$.
For every $\beta\in U$, the symmetric interval centered at $\beta$ with radius
$\frac12\min\{\inf U,1-\sup U\}$ lies in $U^+$.
Thus Lemma~\ref{lem:compactU} permits an entrance coefficient depending only
on $U$, since the coin coefficient is numerical and $1+\beta\le2$.
The expectation part of Theorem~\ref{thm:conditional}, applied on $U^+$,
yields the claimed coefficients $C_{\mu,k}(U)$, with the permitted dependence
on $c_{\mathrm{se}}$ and the quantiles of $T_{\mathrm{se}}$.
\end{proof}

\subsection{Pairwise-Distinct Means}
\label{app:distinct-means}

For pairwise-distinct means, the cited finite-mean sufficient-exploration
result supplies the premise of Theorem~\ref{thm:unguarded-suffexp} for the
rule in Corollary~\ref{cor:distinctmeans}.

\begin{proof}[Proof of Corollary~\ref{cor:distinctmeans}]
We apply the cited sufficient-exploration result and then
Theorem~\ref{thm:unguarded-suffexp}.
In \citet[Lemmas~7, 17, and~21, Appendices~C.3 and~D.1]{jourdan2022top},
Lemma~17 verifies the leader property for the empirical leader, and Lemma~21
verifies the challenger property for the TCI index. Bernoulli rewards belong
to their bounded-reward model, and pairwise-distinct means satisfy their
Assumption~2. With tie-breaking as in Corollary~\ref{cor:distinctmeans},
Lemma~7 therefore gives an almost surely finite random time $T_{\mathrm{se}}$ with
$\mathbb{E}_{\mu}[T_{\mathrm{se}}]<\infty$ such that
\[
  \min_{i\in[K]}N_{t,i}\ge\sqrt{t/K}
  \qquad\forall t\ge T_{\mathrm{se}}.
\]
Thus Theorem~\ref{thm:unguarded-suffexp} applies with $c_{\mathrm{se}}=1$ to
the rule in that corollary.
\end{proof}

\subsection{Guarded Expected Stopping-Time Bound}
\label{app:guarded-closure}

We prove the exploration and stabilization bounds below, then derive the
expected stopping-time bound.

\noindent\textbf{Proposition~\ref{prop:guard} (restated).}
Assume the exploration guard in Algorithm~\ref{alg:ebtci} is enabled. Then for
all $t\ge K$,
$
  \min_{i\in[K]}N_{t,i}\ge \floor{\sqrt{t/K}}-1.
$
Consequently, for every $\eta\in(0,1)$,
$
  t_{\mathrm{sep}}(\eta)\le K\left(n_{\star}(\eta)+1\right)^2.
$
For every compact interval $U\Subset(0,1)$ containing $\beta$ in its interior,
there exists a constant $C_{\mu,\beta}(U)<\infty$ such that the earliest
$U$-stabilization time $t_U$ satisfies, for every $\eta\in(0,1/4]$,
\[
  \mathbb{P}_{\mu}\!\left(t_U\le C_{\mu,\beta}(U)(\log(e/\eta))^2\right)
  \ge1-\eta.
\]

\begin{proof}[Proof of Proposition~\ref{prop:guard}]
We prove the count bound by induction over square-root epochs, then combine
separation and leader-fraction control to bound stabilization.
Let $m\ge1$ be an integer and define epoch boundaries $t_m:=Km^2$.
By initialization,
$N_{t_1,i}=1$ for all arms. Assume inductively that $N_{t_m,i}\ge m$ for all
arms. For $t_m<t<t_{m+1}$, we have $m<\sqrt{t/K}<m+1$. Thus every round
after the epoch's first round is guarded while any arm still has count $m$.
Each guard pull raises a least-sampled arm's count from $m$ to $m+1$, so at
most $K$ such pulls suffice. There are $K(2m+1)-1\ge K$ rounds after the first,
and hence all counts are at least $m+1$ at $t_{m+1}$. This proves the epoch
invariant and the stated lower bound. Evaluating that bound at
$K(n_{\star}(\eta)+1)^2$ gives the claimed bound on $t_{\mathrm{sep}}(\eta)$.

For stabilization, fix $U$ as in the proposition and $\eta\in(0,1/4]$, and set
$L:=K(n_{\star}(\eta/2)+1)^2$. On
$\cK_{\eta/2}\cap\cA_{\eta/2}$, an event of probability at least $1-\eta$,
the inclusion $\cK_{\eta/2}\subseteq\cK^{\mathrm{sep}}_{\eta/2}$ and
Proposition~\ref{prop:permanent} give permanent leader correctness from $L$.
Apply Lemma~\ref{lem:compactU} to a symmetric compact subinterval of $U$
centered at $\beta$, with a constant $C_{\beta}(U)\ge1$. Since the following
integer horizon is at least $L$, it bounds the same earliest stabilization time:
\[
  t_U\le\ceil{C_{\beta}(U)
  \left(L+K+\log\frac{e(L+1)^2}{\eta/2}\right)}.
\]
Write $x:=\log(e/\eta)\ge1$. Lemma~\ref{lem:logsep} and
$\log(x+\log2)\le x$ give $n_{\star}(\eta/2)\le S_{\mu}x$, where
$S_{\mu}:=\max\{1,A_{\mu}(2+\log(2K))\}$.
Thus $L\le D_{\mu}x^2$ with $D_{\mu}:=K(S_{\mu}+1)^2$, and
\[
  \log\frac{e(L+1)^2}{\eta/2}
  \le x+\log2+2\log(D_{\mu}+1)+4\log x
  \le \left(5+\log2+2\log(D_{\mu}+1)\right)x^2.
\]
The ceiling adds at most $1\le x^2$, so the constant
\[
  C_{\mu,\beta}(U):=
  C_{\beta}(U)\left(D_{\mu}+K+5+\log2+2\log(D_{\mu}+1)\right)+1
\]
depends only on $\mu,\beta,U$ (with $K$ included in the instance $\mu$) and satisfies
\[
  \mathbb{P}_{\mu}\!\left(
    t_U\le C_{\mu,\beta}(U)(\log(e/\eta))^2
  \right)\ge1-\eta
  \qquad\forall\eta\in(0,1/4].
\]
Letting $\eta\downarrow0$ shows that $t_U$ is almost surely finite.
An upper-quantile bound on $(0,1]$ is
$u_U(\eta):=C_{\mu,\beta}(U)(\log(e/\min\{\eta,1/4\}))^2$.
Its dyadic series is finite because $u_U(2^{-k})=O_{\mu,\beta,U}((k+1)^2)$
and $\sum_{k\ge0}2^{-k}(k+1)^2<\infty$.
\end{proof}

\begin{proof}[Proof of Theorem~\ref{thm:guarded}]
By Proposition~\ref{prop:guard} and Lemma~\ref{lem:compactU}, the guard yields a
polylogarithmic upper-quantile bound for a $U$-stabilization time $t_U$. In
particular, this bound satisfies both the probability requirement and the
dyadic summability condition in Theorem~\ref{thm:conditional}. Applying the
expectation part of Theorem~\ref{thm:conditional} completes the proof.
\end{proof}

\subsection{Limitation of One-Step Index Comparisons}
\label{app:barrier}

We first determine the limiting population evidence as a provisional leader's
count grows while a competitor's count stays fixed. We then use this limit
to prove the one-step index-comparison limitation.

\begin{lemma}[Asymptotics with a fixed leader identity and growing count]
\label{lem:frozen}
Fix distinct arms $b,u\in[K]$ and a positive integer count $m$. The following
map is nondecreasing for positive integer counts $n$:
\[
  n\mapsto
  \inf_{x\in(0,1)}
  \left\{n \KL{\mu_b}{x}+m \KL{\mu_u}{x}\right\}.
\]
As $n\to\infty$, it converges to $m\KL{\mu_u}{\mu_b}$.
\end{lemma}

\begin{proof}
Monotonicity in $n$ follows because the objective contains the nonnegative
term $n \KL{\mu_b}{x}$. Evaluating at $x=\mu_b$ gives the upper bound
$m \KL{\mu_u}{\mu_b}$. Conversely, let
$x_n:=(n\mu_b+m\mu_u)/(n+m)$ be the exact pooled minimizer, as in
Remark~\ref{rem:optimal-root}.
If $\KL{\mu_b}{x_n}$ failed to converge to zero, then $n\KL{\mu_b}{x_n}$ would
be unbounded along a subsequence, contradicting the finite upper bound at
$x=\mu_b$. Hence $x_n\to\mu_b$, and continuity gives the matching lower bound.
\end{proof}

\begin{proof}[Proof of Proposition~\ref{prop:barrier}]
We compare the exact index of the equal-mean arm with the limiting index of
the higher-mean arm as the provisional leader's count grows.
Because $\mu_j=\mu_b=\lambda$, choosing $x=\lambda$ gives
\[
  \Phi_j^{(b)}(n,M(\eta))=\log M(\eta),
\]
for every $n$. Fix any $a<C\KL{\mu_u}{\lambda}$ and choose $\eps>0$ so that
$a<C\KL{\mu_u}{\lambda}-2\eps$. By Lemma~\ref{lem:frozen}, for every
sufficiently small $\eta$ there exists $n_{\eta}$ such that for all
$n\ge n_{\eta}$,
\[
  \inf_{x\in(0,1)}\left\{n \KL{\lambda}{x}+m_u(\eta)\KL{\mu_u}{x}\right\}
  \ge
  m_u(\eta)\KL{\mu_u}{\lambda}-\eps\log(e/\eta).
\]
Since $m_u(\eta)\ge C\log(e/\eta)$ for small $\eta$,
\[
  \Phi_u^{(b)}(n,m_u(\eta))
  \ge
  \left(C\KL{\mu_u}{\lambda}-\eps\right)\log(e/\eta)+\log m_u(\eta).
\]
By assumption, for all sufficiently small $\eta$ and all $n\ge n_{\eta}$,
\[
  \log M(\eta)
  =
  \Phi_j^{(b)}(n,M(\eta))
  \ge
  \Phi_u^{(b)}(n,m_u(\eta))
  \ge
  \left(C\KL{\mu_u}{\lambda}-\eps\right)\log(e/\eta).
\]
Hence
\[
  M(\eta)\ge\eta^{-(C\KL{\mu_u}{\lambda}-\eps)}.
\]
Because $a<C\KL{\mu_u}{\lambda}-2\eps$, this implies
$M(\eta)\ge\eta^{-a}$ for all sufficiently small $\eta$.
\end{proof}

\section{Experiment Details}
\label{app:exp-details}

This section describes the experiments in Section~\ref{sec:experiments},
including the grids, recorded statistics, and supplementary diagnostics. All
experiments use synthetic Bernoulli rewards and no external datasets.

\subsection{Algorithms and Implementation}
\label{app:exp-algorithms}

All variants use the same GLR stopping rule, the threshold
$c(t,\delta)=\log(1/\delta)+2\log(1+t/2)+2+\log(K-1)$, and smallest-index
tie-breaking. We compare three sampling variants, each initialized by one
pull per arm:
\begin{itemize}
    \item \textbf{Unguarded penalized}: the unguarded algorithm, with challenger
    $C_{t+1}\in\argmin_{j\in[K]\setminus\{B_{t+1}\}}\{W_t(B_{t+1},j)+\log N_{t,j}\}$.
    \item \textbf{Guarded penalized}: the same rule, except that if
    $\min_{i\in[K]} N_{t,i}<\sqrt{t/K}$, the algorithm pulls a least-sampled arm.
    \item \textbf{Unpenalized ablation}: the unguarded algorithm whose challenger
    minimizes $W_t(B_{t+1},j)$ instead of $W_t(B_{t+1},j)+\log N_{t,j}$; the stopping
    rule, threshold, initialization, and tie-breaking are unchanged.
\end{itemize}
Challenger comparisons treat index values within an absolute tolerance of
$10^{-12}$ of the current minimum as ties, so the smaller index is retained.
The implementation evaluates the Bernoulli two-sample GLR at the pooled mean,
uses Numba-accelerated simulation kernels, and records every run in a raw CSV
with the recorded time, final recommendation, cap indicator, guard count, final
arm counts, and optimal population weights.

\paragraph{Code and reproduction.}
The supplementary code archive \path{eb_tci_neurips_submission_code.zip}
contains the directory \path{eb_tci_neurips_submission_code/}; all paths below
are relative to this directory. The setup instructions are in \path{README.md}.
The simulation kernels are in \path{src/eb_tci/sim_numba.py}, and the per-run
data are in \path{results/raw_full.csv} and
\path{results/raw_oracle_tracking_full.csv}. From the package root, the full
experiments and plots are reproduced by:
\begingroup
\small
\begin{verbatim}
python scripts/run_experiments.py --profile full --out results
python scripts/run_oracle_tracking_baseline.py --profile full --out results
python scripts/make_plots.py
\end{verbatim}
\endgroup

\paragraph{Population benchmark and threshold-crossing predictor.}
For each instance and each $\beta$, we compute $\Tstar_{\beta}(\mu)$ from the
Bernoulli equalization system in Proposition~\ref{prop:fixedalpha}. For fixed
$\alpha\in(0,1)$ and suboptimal arm
$j\in[K]\setminus\{\istar\}$, recall
\[
  I_j^{\alpha}(\omega)
  :=
  \inf_{x\in(0,1)}
  \{\alpha\KL{\mu_{\istar}}{x}+\omega\KL{\mu_j}{x}\},
  \qquad \omega>0.
\]
The optimizer is characterized by a scalar $\lambda_{\alpha}$ such that
\[
  I_j^{\alpha}(w_j^{\alpha}(\mu))=\lambda_{\alpha}
  \quad\forall j\in[K]\setminus\{\istar\},
  \qquad
  \sum_{j\neq \istar}w_j^{\alpha}(\mu)=1-\alpha,
\]
and $\Tstar_{\alpha}(\mu)=1/\lambda_{\alpha}$. Numerically, we solve this by
bisection on $\lambda_{\alpha}$, evaluating each inverse
$(I_j^{\alpha})^{-1}$ by bisection of the strictly increasing rate
$I_j^{\alpha}$. The inverse-map domain is
$0<\lambda_{\alpha}<\alpha\min_{j\neq\istar}\KL{\mu_{\istar}}{\mu_j}$,
as established in Proposition~\ref{prop:fixedalpha}.

The implementation in \path{src/eb_tci/oracle.py}, function
\texttt{solve\_t\_star}, uses the outer bracket
$[0,(1-10^{-12})\alpha\min_{j\neq\istar}\KL{\mu_{\istar}}{\mu_j}]$.
Each inverse search starts with the weight bracket $[0,1]$ and doubles its
upper endpoint until the rate reaches the trial value of $\lambda_{\alpha}$;
a safeguard ends expansion if that endpoint exceeds $10^{12}$. The outer
and inner searches allow at most $140$ and $120$ bisection iterations,
respectively. Both stop when the bracket width is at most $10^{-12}$ times
the larger of one and the upper endpoint. The returned suboptimal weights
are rescaled to sum to $1-\alpha$, and the information rate is recomputed as
the minimum of their pairwise rates. No separate residual test is performed.
The threshold-crossing predictor retains the exact finite-confidence threshold
and minimizes over integer times:
\[
  t_{\mathrm{pred}}(\mu,\beta,\delta)
  :=
  \min\{t\ge K:t/\Tstar_{\beta}(\mu)>c(t,\delta)\}.
\]

\paragraph{Experiment grids and replication counts.}
The five grids in Table~\ref{tab:app-exp-grid} comprise $23{,}050$ independent
simulation runs.

\begin{table}[t]
\centering
\small
\begin{tabular}{llr}
\toprule
Experiment & Grid & Runs \\
\midrule
Main evaluation & $3$ instances $\times$ $3$ values of $\beta$ $\times$ $5$ values of $\delta$ $\times$ $2$ variants $\times$ $200$ replications & $18{,}000$ \\
$\beta$-sweep & $2$ instances $\times$ $7$ values of $\beta$ $\times$ $200$ replications & $2{,}800$ \\
Stress survival & $2$ penalized variants $\times$ $600$ replications & $1{,}200$ \\
Penalty ablation & $1$ unpenalized variant $\times$ $500$ replications & $500$ \\
Guard overhead & $K\in\{4,8,12,16\}$ with $200,200,100,50$ replications & $550$ \\
\bottomrule
\end{tabular}
\caption{Experiment grids. The three main evaluation instances are
\textsc{sep5}, \textsc{cluster-close5}, and \textsc{tie-close4}. The evaluation
confidence levels are $10^{-4},10^{-5},10^{-6},10^{-7},10^{-8}$.}
\label{tab:app-exp-grid}
\end{table}

The main evaluation instances are
\[
\textsc{sep5}=[0.70,0.50,0.40,0.30,0.20],\qquad
\textsc{cluster-close5}=[0.70,0.64,0.63,0.62,0.61],
\]
and
\[
\textsc{tie-close4}=[0.60,0.56,0.56,0.56].
\]
The $\beta$-sweep uses \textsc{cluster-close5} and \textsc{tie-close4},
$\beta\in\{0.2,0.3,\dots,0.8\}$, $\delta=10^{-6}$, and the guarded penalized
variant. The guard-overhead family uses
$[0.60,0.56,\dots,0.56]$ with $K\in\{4,8,12,16\}$, $\beta=0.5$, and
$\delta=10^{-6}$. The stress experiments use \textsc{tie-close4}, $\beta=0.5$,
and $\delta=10^{-8}$.

\paragraph{Randomness, caps, and uncertainty summaries.}
The seed-to-setting map is defined by \texttt{build\_specs} in
\path{scripts/run_experiments.py}. The full profile starts at $20260424$ and
enumerates the main evaluation, $\beta$-sweep, guard overhead, stress survival,
and penalty ablation in that order. Within the main evaluation, the nested
loops run over instances (\textsc{sep5}, \textsc{cluster-close5},
\textsc{tie-close4}), $\beta$ ($0.3,0.5,0.7$), $\delta$
($10^{-4},\dots,10^{-8}$), and variant (unguarded, then guarded), with
replication as the innermost index. The $\beta$-sweep visits
\textsc{cluster-close5}, then \textsc{tie-close4}, with $\beta$ increasing
from $0.2$ to $0.8$. Guard overhead visits $K=4,8,12,16$; stress survival
visits the unguarded, then guarded variant. Each setting uses consecutive
seeds and advances the seed counter by its replication count, except that
every guard-overhead setting reserves $200$ seeds, including those with
$100$ or $50$ replications. The separately reported population-reference
runs use $200$ consecutive seeds starting at $20268000$ for
\textsc{cluster-close5} and $20269000$ for \textsc{tie-close4}, as specified
in \path{src/eb_tci/oracle_tracking.py}. Both raw CSV files record the seed
of each run.

All runs use a fail-safe cap of $700{,}000$ rounds. Each run records the smaller
of its stopping time
$\tau_{\delta}$ and the cap. A cap hit therefore contributes $700{,}000$ to
recorded-time summaries, while its stopping time remains censored. Only the
unpenalized ablation hits the cap; its cap-hit rate is reported explicitly.

For a scalar statistic $Z$, let $Z^{(r)}$ be its value in replication
$r\in\{1,\dots,R\}$, where $R\ge2$. The empirical mean is
$\overline{Z}:=R^{-1}\sum_{r=1}^{R}Z^{(r)}$, and its empirical standard
error over independent runs is
\[
  \widehat{\mathrm{se}}(\overline{Z})
  =
  \left(
    \frac{1}{R(R-1)}
    \sum_{r=1}^{R}(Z^{(r)}-\overline{Z})^2
  \right)^{1/2}.
\]
Bootstrap percentile $95\%$ confidence intervals use $1000$ resamples with an
independent bootstrap seed $314159$. The finite-confidence scaling figure uses
bootstrap percentile bands for the empirical mean; the $\beta$-sensitivity
figure uses $\pm1.96$ empirical standard errors.

\paragraph{Metrics.}
The main predictor-agreement metric is the absolute relative error
\[
  \operatorname{ARE}(\mu,\beta,\delta)
  :=
  \left\lvert\frac{\overline{\tau}_{\delta}}{t_{\mathrm{pred}}(\mu,\beta,\delta)}-1\right\rvert,
\]
where $\overline{\tau}_{\delta}$ is the Monte Carlo mean stopping time for
the uncapped penalized evaluation runs. We also report the median, $10$th,
$90$th, and $95$th percentiles of the recorded times, the empirical error rate,
the cap-hit rate, and final allocation diagnostics. For the guarded algorithm,
recall the guard count at stopping:
\[
  G_{\tau_{\delta}}^{\mathrm{tot}}
  =
  \sum_{t=K}^{\tau_{\delta}-1}
  \one\{\min_{i\in[K]} N_{t,i}<\sqrt{t/K}\}.
\]
We report the guarded-pull fraction
$G_{\tau_{\delta}}^{\mathrm{tot}}/\tau_{\delta}$ and the normalization
$G_{\tau_{\delta}}^{\mathrm{tot}}/\sqrt{K\tau_{\delta}}$ used in the theoretical
guard-count bound. No guarded run hits the cap.

\paragraph{Aggregate results.}
Table~\ref{tab:app-exp-calibration} reports predictor agreement over the full
$3\times3\times5$ grid. Table~\ref{tab:app-exp-stress} summarizes the penalized
stress tests and unpenalized ablation. Table~\ref{tab:app-exp-guard-overhead}
gives the equal-close guard-overhead family.

\begin{table}[t]
\centering
\small
\begin{tabular}{lcccc}
\toprule
Variant & Settings & Median ARE & Mean ARE & Maximum ARE \\
\midrule
Guarded penalized & $45$ & $1.39\%$ & $1.53\%$ & $3.86\%$ \\
Unguarded penalized & $45$ & $1.60\%$ & $1.66\%$ & $4.20\%$ \\
\bottomrule
\end{tabular}
\caption{Absolute relative error between empirical mean stopping time and
$t_{\mathrm{pred}}$ over the complete evaluation grid. No penalized evaluation
setting hits the simulation cap or yields a Monte Carlo misidentification.}
\label{tab:app-exp-calibration}
\end{table}

\begin{table}[t]
\centering
\small
\begin{tabular}{lrrrr}
\toprule
Variant & Replications & Mean & Median & $95$th percentile \\
\midrule
Unguarded penalized & $600$ & $105{,}593$ & $104{,}918$ & $134{,}140$ \\
Guarded penalized & $600$ & $104{,}022$ & $103{,}743$ & $130{,}792$ \\
Unpenalized ablation & $500$ & $188{,}036$ & $108{,}036$ & $700{,}000$ \\
\bottomrule
\end{tabular}
\caption{Stress-test recorded-time summaries on \textsc{tie-close4} with
$\beta=0.5$ and $\delta=10^{-8}$. The unpenalized ablation hits the
$700{,}000$-round cap in $12.8\%$ of runs; those runs contribute the cap
value to the mean and quantiles.}
\label{tab:app-exp-stress}
\end{table}

\begin{table}[t]
\centering
\small
\begin{tabular}{lrrrrr}
\toprule
Instance & $K$ & Replications & Mean $\tau_{\delta}$ & Mean $G_{\tau_{\delta}}^{\mathrm{tot}}/\tau_{\delta}$ & $95$th percentile $G_{\tau_{\delta}}^{\mathrm{tot}}/\sqrt{K\tau_{\delta}}$ \\
\midrule
\textsc{equal-close4} & $4$ & $200$ & $92{,}922$ & $0.0058\%$ & $1.80\%$ \\
\textsc{equal-close8} & $8$ & $200$ & $196{,}872$ & $0.0065\%$ & $1.80\%$ \\
\textsc{equal-close12} & $12$ & $100$ & $301{,}114$ & $0.0069\%$ & $1.89\%$ \\
\textsc{equal-close16} & $16$ & $50$ & $417{,}489$ & $0.0062\%$ & $1.41\%$ \\
\bottomrule
\end{tabular}
\caption{Guard overhead on the equal-close family with $\beta=0.5$ and
$\delta=10^{-6}$. Both the mean guarded-pull fraction and the $95$th percentile
of the normalized guard count are shown as percentages.}
\label{tab:app-exp-guard-overhead}
\end{table}

\paragraph{$\beta$-sensitivity numerical table.}
Table~\ref{tab:app-exp-beta} gives the numerical values underlying the
right-hand panel of Figure~\ref{fig:exp-main-diagnostics}. The empirical mean
tracks the nonmonotone dependence of $\Tstar_{\beta}(\mu)$ on the leader
sampling probability. The relative error is signed and equals
$\overline{\tau}_{\delta}/t_{\mathrm{pred}}-1$.

\begin{table}[t]
\centering
\small
\begin{tabular}{lrrrr}
\toprule
Instance & $\beta$ & Mean $\tau_{\delta}$ & $t_{\mathrm{pred}}$ & Signed relative error \\
\midrule
\textsc{cluster-close5} & $0.2$ & $34{,}151$ & $33{,}901$ & $0.74\%$ \\
\textsc{cluster-close5} & $0.3$ & $28{,}757$ & $29{,}267$ & $-1.74\%$ \\
\textsc{cluster-close5} & $0.4$ & $29{,}141$ & $29{,}001$ & $0.48\%$ \\
\textsc{cluster-close5} & $0.5$ & $31{,}390$ & $31{,}334$ & $0.18\%$ \\
\textsc{cluster-close5} & $0.6$ & $36{,}063$ & $36{,}505$ & $-1.21\%$ \\
\textsc{cluster-close5} & $0.7$ & $45{,}304$ & $46{,}458$ & $-2.48\%$ \\
\textsc{cluster-close5} & $0.8$ & $65{,}735$ & $67{,}856$ & $-3.13\%$ \\
\textsc{tie-close4} & $0.2$ & $102{,}704$ & $102{,}907$ & $-0.20\%$ \\
\textsc{tie-close4} & $0.3$ & $88{,}998$ & $88{,}819$ & $0.20\%$ \\
\textsc{tie-close4} & $0.4$ & $85{,}348$ & $87{,}269$ & $-2.20\%$ \\
\textsc{tie-close4} & $0.5$ & $92{,}510$ & $93{,}342$ & $-0.89\%$ \\
\textsc{tie-close4} & $0.6$ & $105{,}470$ & $107{,}680$ & $-2.05\%$ \\
\textsc{tie-close4} & $0.7$ & $134{,}493$ & $135{,}786$ & $-0.95\%$ \\
\textsc{tie-close4} & $0.8$ & $191{,}672$ & $196{,}628$ & $-2.52\%$ \\
\bottomrule
\end{tabular}
\caption{$\beta$-sensitivity with $\delta=10^{-6}$, the guarded penalized rule,
and $200$ replications per point.}
\label{tab:app-exp-beta}
\end{table}

\subsection{Idealized Population-Allocation Reference}
\label{app:exp-tracking-reference}

We use an idealized population-allocation reference to compare stopping-time
scales. From the true means, it computes a full $K$-arm allocation
\[
  w^{\star}(\mu)\in
  \argmax_{\substack{w\in(0,1)^K\\ \sum_{i=1}^{K}w_i=1}}
  \min_{j\in[K]\setminus\{\istar\}} I_j^{w_{\istar}}(w_j).
\]
The implementation in \path{src/eb_tci/oracle_tracking.py} initializes with
one pull per arm and then deterministically C-tracks the fixed allocation:
after checking the stopping rule at time $t$, it chooses
\[
  I_{t+1}\in\argmax_{i\in[K]}
  \{(t+1)w_i^{\star}(\mu)-N_{t,i}\},
\]
with smallest-index tie-breaking. It uses no additional projection or
forced-exploration step. The GLR statistic and threshold $c(t,\delta)$ are
the same as in the $\beta$-EB-TCI runs. This construction is motivated by
the allocation principle underlying Track-and-Stop
\citep{garivier2016track}, but it is not an empirical Track-and-Stop
implementation: the reference is given population information unavailable to a
learning algorithm and never recomputes an empirical optimal allocation.

For exact reproduction of the reported reference runs, the default command
uses the stored weights in \texttt{PRECOMPUTED\_GLOBAL\_ALLOCATIONS} in
\path{src/eb_tci/oracle_tracking.py}. Recomputing these weights can change
the deterministic pull order through differences in their last digits.
The optional \texttt{-{}-recompute-allocation} flag calls
\texttt{solve\_global\_optimal\_allocation}: it evaluates the fixed-fraction
solver at $181$ equally spaced values of $\alpha$ in $[0.05,0.95]$, selects
the first maximizing grid point, and refines between the grid points at most
two positions on either side using $40$ golden-section iterations. Among
the two final trial points, it returns the one with the larger computed
information rate, choosing the smaller $\alpha$ in a tie. Each fixed-fraction
evaluation uses the bisection settings described above.

Table~\ref{tab:app-exp-tracking-reference} compares this reference with the
closest guarded $\beta$-EB-TCI point from the $\beta$-sweep. We use $\beta=0.4$
because it is the closest grid value to the optimal best-arm fraction
$\alpha^{\star}:=w_{\istar}^{\star}(\mu)$. The optimal fraction is
$\alpha^{\star}=0.357$ on \textsc{cluster-close5} and
$\alpha^{\star}=0.367$ on \textsc{tie-close4}. Because the methods use
different best-arm fractions and the reference knows $\mu$, the table is a
diagnostic comparison of finite-sample scales, not evidence for an algorithm
ranking.

\begin{table}[t]
\centering
\small
\setlength{\tabcolsep}{3pt}
\begin{tabular}{llrrrrr}
\toprule
Instance & Method & $\beta$ or $\alpha^{\star}$ & Replications & Mean $\tau_{\delta}$ & SE & $95$th percentile \\
\midrule
\textsc{cluster-close5} & Guarded $\beta$-EB-TCI & $0.400$ & $200$ & $29{,}141$ & $401$ & $38{,}530$ \\
\textsc{cluster-close5} & Idealized allocation tracking & $0.357$ & $200$ & $34{,}808$ & $411$ & $43{,}978$ \\
\textsc{tie-close4} & Guarded $\beta$-EB-TCI & $0.400$ & $200$ & $85{,}348$ & $1{,}107$ & $111{,}139$ \\
\textsc{tie-close4} & Idealized allocation tracking & $0.367$ & $200$ & $102{,}085$ & $1{,}373$ & $132{,}011$ \\
\bottomrule
\end{tabular}
\caption{Idealized population-allocation reference at
$\delta=10^{-6}$. The idealized reference rows use the population-optimal Bernoulli
allocation and the same GLR threshold as the $\beta$-EB-TCI experiments. SE is
the empirical standard error of the mean. No run in this
table hits the cap or yields a Monte Carlo misidentification.}
\label{tab:app-exp-tracking-reference}
\end{table}

\subsection{Supplementary Figures}
\label{app:exp-supp-figures}

Figure~\ref{fig:app-exp-guard-diagnostics} reports the guard diagnostics that are
summarized in the main text. The left panel compares stopping-time survival
curves on \textsc{tie-close4}. The right panel reports the normalized guard
overhead across the equal-close family.

\begin{figure*}[t]
  \centering
  \includegraphics[width=0.88\textwidth]{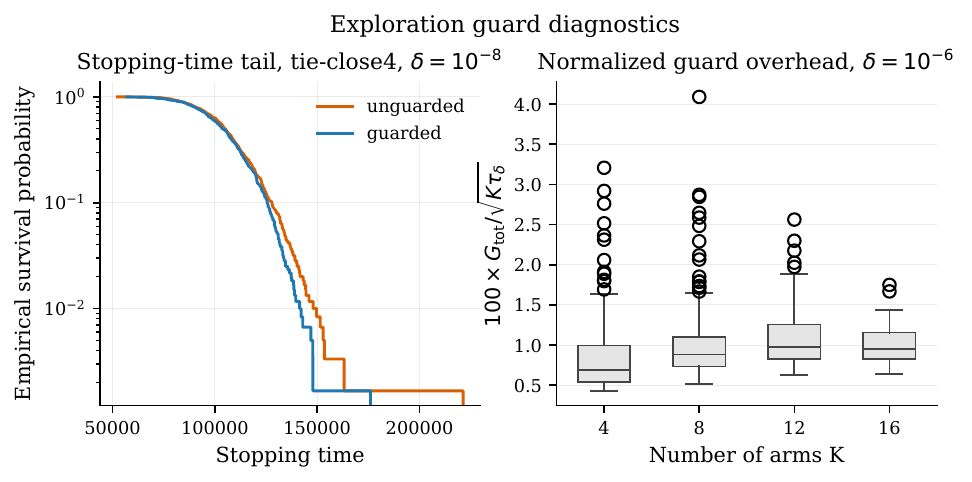}
  \caption{Exploration guard diagnostics. Left: empirical survival curves on
  \textsc{tie-close4} with $\beta=0.5$ and $\delta=10^{-8}$, using $600$
  replications per penalized variant. Right: normalized guard overhead on the
  equal-close family with $K\in\{4,8,12,16\}$ and $\delta=10^{-6}$.}
  \label{fig:app-exp-guard-diagnostics}
\end{figure*}

Figure~\ref{fig:app-exp-penalty-allocation} contains the remaining diagnostics.
The penalty ablation shows the effect of removing $\log N_{t,j}$ from the
challenger index. Its empirical survival curves summarize recorded times:
cap hits contribute the cap value, so the curves do not estimate the
uncensored stopping-time tail at or beyond the cap.

For the allocation panel, define the full count vector
$N_t:=(N_{t,i})_{i\in[K]}$ and the full fixed-$\beta$ target
$v^{\beta}(\mu)$ by
\[
  v_{\istar}^{\beta}(\mu):=\beta,
  \qquad
  v_j^{\beta}(\mu):=w_j^{\beta}(\mu)
  \quad(j\in[K]\setminus\{\istar\}).
\]
The left allocation panel compares mean final sampling fractions with this
target; the right panel reports the per-run distance
$\lVert N_{\tau_{\delta}}/\tau_{\delta}-v^{\beta}(\mu)\rVert_1$.
These are measurements at stopping; no per-run stabilization criterion is
reported.

\begin{figure*}[t]
  \centering
  \begin{minipage}[t]{0.47\textwidth}
    \centering
    \includegraphics[width=\linewidth]{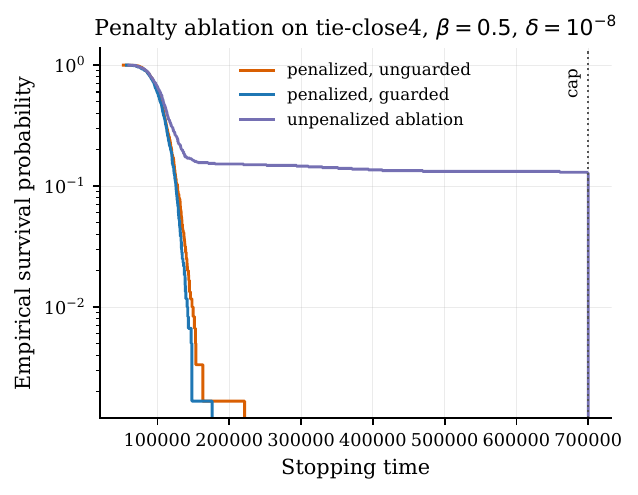}
    \vspace{-0.5em}

    {\small (a) Penalty ablation on the equal-close stress instance.}
  \end{minipage}
  \hfill
  \begin{minipage}[t]{0.47\textwidth}
    \centering
    \includegraphics[width=\linewidth]{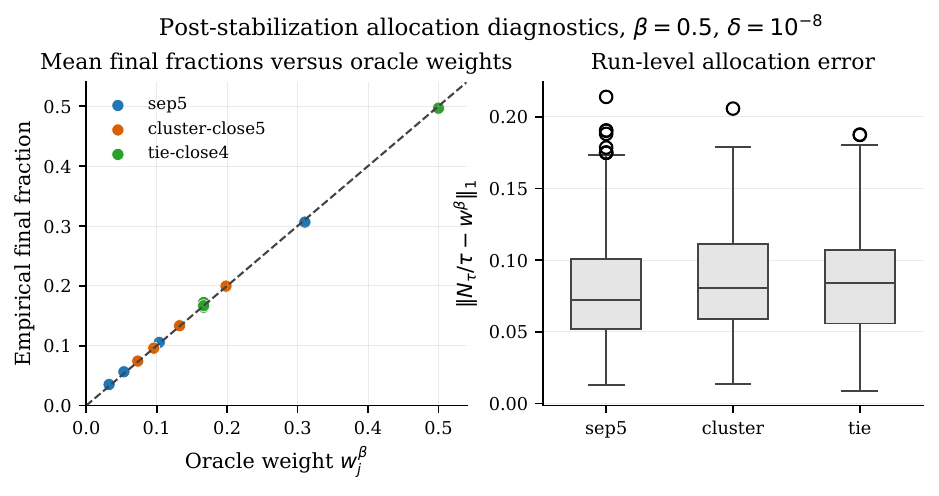}
    \vspace{-0.5em}

    {\small (b) Final allocation diagnostics.}
  \end{minipage}
  \caption{Additional diagnostics. Left: the unpenalized challenger removes
  $\log N_{t,j}$ but leaves the stopping rule and all other details unchanged.
  Right: final sampling fractions for the guarded rule with $\beta=0.5$ and
  $\delta=10^{-8}$, compared with the optimal population weights.}
  \label{fig:app-exp-penalty-allocation}
\end{figure*}

\clearpage
\section*{NeurIPS Paper Checklist}

\begin{enumerate}

\item {\bf Claims}
    \item[] Question: Do the main claims made in the abstract and introduction accurately reflect the paper's contributions and scope?
    \item[] Answer: \answerYes{}.
    \item[] Justification: The abstract and introduction state the exact scope of the paper: a sharp post-stabilization analysis of Bernoulli $\beta$-EB-TCI, unconditional high-probability closures, expectation closures under additional exploration assumptions, and a guarded variant. The claims are matched by Theorems~5.2, 6.4, 6.7, Corollary~6.5, Proposition~6.8, and the experiments in Section~7.

\item {\bf Limitations}
    \item[] Question: Does the paper discuss the limitations of the work performed by the authors?
    \item[] Answer: \answerYes{}.
    \item[] Justification: The paper explicitly identifies the remaining limitation for the original unguarded rule: proving finite-mean stable-entry control under a unique-best instance with suboptimal tie classes. This limitation is discussed in Section~6, especially in the paragraph ``Why the unguarded general-$K$ case remains open,'' and in the conclusion.

\item {\bf Theory assumptions and proofs}
    \item[] Question: For each theoretical result, does the paper provide the full set of assumptions and a complete (and correct) proof?
    \item[] Answer: \answerYes{}.
    \item[] Justification: The statements specify the relevant assumptions, including Bernoulli rewards, unique best arm, compact intervals $U\Subset(0,1)$, stable-regime entrance, sufficient exploration, and the optional guard. Full proofs are provided in Appendices A--F, with the main theorems and propositions cross-referenced from the body of the paper.

    \item {\bf Experimental result reproducibility}
    \item[] Question: Does the paper fully disclose all the information needed to reproduce the main experimental results of the paper to the extent that it affects the main claims and/or conclusions of the paper (regardless of whether the code and data are provided or not)?
    \item[] Answer: \answerYes{}.
    \item[] Justification: The experiments are synthetic Bernoulli simulations and require no external dataset. Section~7 and Appendix~G specify the algorithmic rule, Bernoulli instances, parameter grids, replication counts, random seeds, computation of $T_\beta^\star(\mu)$, stopping threshold, simulation caps, uncertainty summaries, and reported metrics.

\item {\bf Open access to data and code}
    \item[] Question: Does the paper provide open access to the data and code, with sufficient instructions to faithfully reproduce the main experimental results, as described in supplemental material?
    \item[] Answer: \answerYes{}.
    \item[] Justification: We provide an anonymized supplementary code package for the synthetic Bernoulli simulations. The package contains the EB-TCI simulation code, oracle computation for $T_\beta^\star(\mu)$ and $t_{\mathrm{pred}}$, scripts to reproduce the calibration grid, $\beta$-sweep, guard diagnostics, penalty ablation, minimal tracking-reference comparison, and figure/table generation, together with a README, configuration files, and dependency specification. The experiments use no external dataset.

\item {\bf Experimental setting/details}
    \item[] Question: Does the paper specify all the training and test details (e.g., data splits, hyperparameters, how they were chosen, type of optimizer) necessary to understand the results?
    \item[] Answer: \answerYes{}.
    \item[] Justification: There is no training/test split because the experiments are Monte Carlo simulations of Bernoulli bandit instances. Section~7 and Appendix~G specify the instances, $\beta$ and $\delta$ grids, tie-breaking rule, optional guard status, numerical optimizer for $T_\beta^\star(\mu)$, stopping threshold, replication counts, and simulation caps.

\item {\bf Experiment statistical significance}
    \item[] Question: Does the paper report error bars suitably and correctly defined or other appropriate information about the statistical significance of the experiments?
    \item[] Answer: \answerYes{}.
    \item[] Justification: Appendix~G reports empirical standard errors and bootstrap percentile $95\%$ confidence intervals, with the number of bootstrap resamples and independent simulation replications specified. The main figures omit some error bars for readability, but the corresponding numerical uncertainty summaries are provided in the appendix tables.

\item {\bf Experiments compute resources}
    \item[] Question: For each experiment, does the paper provide sufficient information on the computer resources (type of compute workers, memory, time of execution) needed to reproduce the experiments?
    \item[] Answer: \answerNo{}.
    \item[] Justification: The experiments are lightweight CPU-only synthetic Bernoulli simulations and require no GPUs, accelerators, cloud services, external datasets, or specialized hardware. The supplementary code package provides smoke, medium, and full reproduction modes so that the results can be regenerated on standard CPU machines. We do not report exact hardware specifications, memory usage, or wall-clock time for each experimental group.
    
\item {\bf Code of ethics}
    \item[] Question: Does the research conducted in the paper conform, in every respect, with the NeurIPS Code of Ethics \url{https://neurips.cc/public/EthicsGuidelines}?
    \item[] Answer: \answerYes{}.
    \item[] Justification: The work is theoretical and uses only synthetic Bernoulli simulations. It does not involve human subjects, private data, scraped data, deployed systems, or high-risk model releases.

\item {\bf Broader impacts}
    \item[] Question: Does the paper discuss both potential positive societal impacts and negative societal impacts of the work performed?
    \item[] Answer: \answerNA{}.
    \item[] Justification: The paper is foundational theory for fixed-confidence best-arm identification and does not introduce a deployed system, dataset, or model. Any societal impacts would be indirect and depend on downstream applications of bandit algorithms rather than on the specific theoretical analysis presented here.
    
\item {\bf Safeguards}
    \item[] Question: Does the paper describe safeguards that have been put in place for responsible release of data or models that have a high risk for misuse (e.g., pre-trained language models, image generators, or scraped datasets)?
    \item[] Answer: \answerNA{}.
    \item[] Justification: The paper does not release datasets, pretrained models, generative models, scraped data, or other artifacts with high misuse risk. The experiments use only synthetic Bernoulli rewards.

\item {\bf Licenses for existing assets}
    \item[] Question: Are the creators or original owners of assets (e.g., code, data, models), used in the paper, properly credited and are the license and terms of use explicitly mentioned and properly respected?
    \item[] Answer: \answerNA{}.
    \item[] Justification: The paper does not use existing datasets, pretrained models, or third-party code assets as part of the contribution. Prior scientific work is credited through citations, and the simulations are generated synthetically.

\item {\bf New assets}
    \item[] Question: Are new assets introduced in the paper well documented and is the documentation provided alongside the assets?
    \item[] Answer: \answerYes{}.
    \item[] Justification: The only new asset is the anonymized supplementary code package used to reproduce the synthetic experiments. The package does not contain external datasets, pretrained models, scraped data, human-subject data, or assets with high misuse risk.

\item {\bf Crowdsourcing and research with human subjects}
    \item[] Question: For crowdsourcing experiments and research with human subjects, does the paper include the full text of instructions given to participants and screenshots, if applicable, as well as details about compensation (if any)? 
    \item[] Answer: \answerNA{}.
    \item[] Justification: The paper does not involve crowdsourcing, human subjects, user studies, participant data, or human annotation.

\item {\bf Institutional review board (IRB) approvals or equivalent for research with human subjects}
    \item[] Question: Does the paper describe potential risks incurred by study participants, whether such risks were disclosed to the subjects, and whether Institutional Review Board (IRB) approvals (or an equivalent approval/review based on the requirements of your country or institution) were obtained?
    \item[] Answer: \answerNA{}.
    \item[] Justification: The paper does not involve crowdsourcing or research with human subjects, so IRB or equivalent review is not applicable.

\item {\bf Declaration of LLM usage}
    \item[] Question: Does the paper describe the usage of LLMs if it is an important, original, or non-standard component of the core methods in this research? Note that if the LLM is used only for writing, editing, or formatting purposes and does \emph{not} impact the core methodology, scientific rigor, or originality of the research, declaration is not required.
    \item[] Answer: \answerNA{}.
    \item[] Justification: LLM-based assistance was used only for non-core research support, including writing and grammar checking, minor editing, and auxiliary help with organizing experiment scripts, checking reproduction commands, and formatting experimental summaries. The LLM was not used as an algorithmic, experimental, or scientific component of the proposed method. The theoretical results, algorithmic definitions, simulation protocol, numerical outputs, plots, and scientific claims were specified and verified by the authors, and the experiments are reproducible from the released code and the protocol described in Section~7 and Appendix~G.

\end{enumerate}
\end{document}